\documentclass[12pt]{article}
\usepackage{amsmath}
\usepackage{graphicx}
\usepackage{enumerate}
\usepackage{natbib}
\usepackage{url}
\usepackage{bbm}

\usepackage{tikz}
\usetikzlibrary{automata,decorations.markings,positioning,arrows}

\usepackage{amsfonts}       
\usepackage{amssymb}
\usepackage{wrapfig, amsmath, amsthm, mathtools}
\usepackage{mathabx}
\usepackage{centernot}

\usepackage{cancel}
\usepackage{authblk}
\usepackage{bbm}

\usepackage{makecell}

\usepackage{caption, subcaption}

\usepackage[ruled,vlined, linesnumbered]{algorithm2e}

\usepackage[colorlinks,citecolor=blue,urlcolor=blue,linkcolor=blue,linktocpage=true]{hyperref}

\usepackage{nicefrac}

\usetikzlibrary{shapes,decorations,arrows,calc,arrows.meta,fit,positioning}
\tikzset{
    -Latex,auto,node distance =1 cm and 1 cm,semithick,
    state/.style ={ellipse, draw, minimum width = 0.7 cm},
    point/.style = {circle, draw, inner sep=0.04cm,fill,node contents={}},
    bidirected/.style={Latex-Latex,dashed},
    el/.style = {inner sep=2pt, align=left, sloped}
}

\def\0{{\bf 0}}
\def\1{{\bf 1}}

\def\BM{{\mathcal B}}
\def\CM{{\mathcal C}}

\def\FM{{\mathcal F}}
\def\HM{{\mathcal H}}
\def\TM{{\mathcal T}}

\def\NM{{\mathcal N}}

\def\GM{{\mathcal G}}

\def\DM{{\mathcal D}}

\def\RB{{\mathbb R}}
\def\NB{{\mathbb N}}
\def\EB{{\mathbb E}}

\def\PB{{\mathbb P}}

\newcommand{\Amse}{\text{AMSE}}

\newcommand{\fTar}{f_\text{tar}}

\def\@centernot#1#2{%
  \mathrel{%
    \rlap{%
      \settowidth\dimen@{$\m@th#1{#2}$}%
      \kern5\dimen@
      \settowidth\dimen@{$\m@th#1=$}%
      \kern-5\dimen@
      $\m@th#1\not$%
    }%
    {#2}%
  }%
}
\makeatother

\newcommand{\independent}{\perp\mkern-9.5mu\perp}
\newcommand{\notindependent}{\centernot{\independent}}

\newcommand{\CardC}{\left|\DM\right|}
\newcommand{\ChoiceSimplex}{\Delta_{\DM}}
\newcommand{\ctrSim}{\kappa_{\text{ctr}}}

\newcommand{\ConvProb}{\overset{p}{\to}}
\newcommand{\ConvAS}{\overset{a.s.}{\to}}
\newcommand{\ConvDist}{\overset{d}{\to}}

\newcommand{\litOp}{o_{\text{p}}}
\newcommand{\litOas}{o_{\text{a.s.}}}
\newcommand{\bigOp}{O_{\text{p}}}
\newcommand{\bigOas}{O_{\text{a.s.}}}

\newcommand{\hEta}{\widehat{\eta}_{t-1}}
\newcommand{\hVar}{\widehat{V}}

\newtheorem{remark}{Remark}
\newtheorem{theorem}{Theorem}
\newtheorem{lemma}{Lemma}
\newtheorem{definition}{Definition}

\newtheorem{proposition}{Proposition}
\newtheorem{cor}{Corollary}
\newtheorem{assumption}{Assumption}

\newtheorem{property}{Property}
\newtheorem{example}{Example}

\DeclarePairedDelimiter\floor{\lfloor}{\rfloor}

\SetKwInput{KwInput}{Input}
\SetKwInput{KwOutput}{Output}

\begin{document}


\title{Online Data Collection for \\Efficient Semiparametric Inference}

\author[1]{Shantanu Gupta \footnote{Corresponding author: \texttt{\href{mailto:shantang@cs.cmu.edu}{shantang@cs.cmu.edu}}}}
\author[1]{Zachary C. Lipton}
\author[2]{David Childers}
\affil[1]{Carnegie Mellon University}
\affil[2]{Bowdoin College}
\date{}
    
\maketitle

\begin{abstract}
While many works have studied statistical \emph{data fusion},
they typically assume that 
the various datasets are given in advance.
However, in practice, estimation requires
difficult data collection decisions
like determining the available data sources,
their costs, and how many samples to collect from each source.
Moreover, this process is often sequential
because the data collected at a given time 
can improve collection decisions in the future.
In our setup, given access to multiple data sources
and budget constraints,
the agent must sequentially decide
which data source to query
to efficiently estimate a target parameter.
We formalize this task using \emph{Online Moment Selection},
a semiparametric framework that applies to
any parameter identified by
a set of moment conditions.
Interestingly, the optimal budget allocation 
depends on the (unknown) true parameters.
We present two online data collection policies,
Explore-then-Commit and Explore-then-Greedy,
that use the parameter estimates
at a given time to optimally allocate 
the remaining budget in the future steps.
We prove that both policies achieve zero regret
(assessed by asymptotic MSE) relative to an oracle policy.
We empirically validate our methods on
both synthetic and real-world causal effect estimation tasks,
demonstrating that the online 
data collection policies outperform
their fixed counterparts
\footnote{The code and data are available at 
\href{https://github.com/shantanu95/online-moment-selection}{https://github.com/shantanu95/online-moment-selection}.}.
\end{abstract}

\section{Introduction}
\label{sec:intro}
The estimation of a statistical or causal parameter
often requires combining multiple datasets,
potentially containing different subsets of variables
or measurements collected under different conditions.
Many works have studied \emph{data fusion} 
to identify and efficiently estimate a target parameter
(like a treatment effect) under a variety of
different structural assumptions 
\citep{shi2023data, hunermund2019causal, bareinboim2016causal, ridder2007econometrics, d2006statistical}.
Most of these works proceed from the assumption that the 
multiple datasets are given in advance and the focus is typically 
on efficiently combining them,
failing to account for
the difficult data collection decisions
involved in the practice of statistical estimation.
Tasked with an estimation problem and budget constraints, 
a practitioner must reason about the available data sources,
their costs, and allocate their budget across the data sources.

In many applications, data is prohibitively expensive
and it is infeasible to collect everything.
Medical tests can cost thousands of dollars.
In survey design, asking every question can result in
poor data quality and survey fatigue \citep{jeong2023exhaustive}.
Secondly, we often lack complete control over
which variables are measured,
e.g., when relying on third-party data sources,
each capturing different subsets or types of variables.
Thirdly, if the data sources do not suffice,
a practitioner may run their own data collection
studies, requiring decisions about what to measure, 
and how many samples to collect.
Moreover, data collection is an ongoing process,
with the data collected at a particular time allowing 
us to make better decisions in the future steps.

In this work, we bring data collection decisions
within the scope of statistical estimation.
Instead of assuming that the datasets are given in advance,
in our setup, the agent has access to multiple
data sources (with an associated cost structure) 
that they can \emph{sequentially} query (i.e., sample from).
The data sources can be arbitrary probability distributions
returning marginals over different variable subsets,
measurements collected under different conditions, etc.
We present \emph{Online Moment Selection} (OMS),
a framework to formalize
the sequential problem of deciding, at each time step,
which data source to query to efficiently estimate 
a target parameter (Section~\ref{sec:setup}).
In OMS, we apply the generalized method of moments (GMM) \citep{hansen1982large} to
estimate the target parameter,
augmenting the moment conditions by 
a vector that determines which moments
(or data source) gets selected at a given time.
We require that the agent has sufficient structural knowledge
to construct moment conditions that
uniquely identify the target parameter
and that each moment condition can
be estimated using samples returned 
by at least one data source.

Our formulation is semiparametric.
Each moment condition is indexed 
by a finite-dimensional parameter $\theta \in \RB^D$
and a possibly high-dimensional or nonparametric nuisance parameter $\eta$ \citep{tsiatis2006semiparametric, bickel1993efficient}.
We assume that the target parameter $\beta \in \RB$ is a function of $\theta$.
The nuisance parameters
are not of primary interest but must
still be estimated for inferring $\beta$,
e.g., the propensity score in causal inference
\citep{kennedy2016semiparametric}.
By avoiding restrictive parametric assumptions 
on the underlying distribution,
we can model the nuisance parameters
using flexible machine learning or nonparametric estimators.
To construct a $\sqrt{n}$-consistent estimator
for $\beta$ in the presence of nuisance estimators that converge at slower rates,
we assume \emph{Neyman orthogonality},
which states that the moment conditions are locally insensitive to
perturbations of the nuisance parameters \citep{chernozhukov2018double, neyman1959optimal}.

The optimal allocation of the budget across
the data sources for estimating $\beta$
depends on the (unknown) true model parameters
$(\theta^*, \eta^*)$,
motivating our online data collection strategy:
we use the estimate
of the model parameters at time $t$ to allocate
the remaining budget in the future steps.
In particular, we use the estimated model parameters
to estimate the asymptotic variance of $\widehat{\beta}$
and allocate the data collection budget
in the future steps to minimize this estimated variance.

Addressing the setting with 
a uniform cost structure over the data sources (Section~\ref{sec:adaptive-data-collection}),
we show that any fixed data collection 
policy suffers constant regret, 
as assessed by
the asymptotic MSE relative to the (unknown) \emph{oracle policy},
the policy with the lowest asymptotic MSE.
To overcome this limitation, we propose two online data collection policies,
Explore-then-Commit (OMS-ETC) and Explore-then-Greedy (OMS-ETG),
and prove that both policies achieve zero regret.
Under OMS-ETC (Section~\ref{sec:oms-etc}), 
we use some fraction of the budget to
estimate the model parameters
by querying the data sources uniformly (the \emph{explore} phase).
In the subsequent steps (the \emph{commit} phase), 
we allocate the remaining budget across the data sources 
such that the 
(estimated) asymptotic variance 
of the target parameter $\beta$ is minimized.
Under OMS-ETG (Section~\ref{sec:oms-etg}), 
instead of committing 
to the optimal allocation
determined after exploration, 
we continually update the model parameters
and thereby the optimal allocation
of the remaining budget.
We then extend our analysis to account for 
a non-uniform cost structure
over the data sources (Section~\ref{sec:oms-cost-structure}),
proposing variants of OMS-ETC and OMS-ETG for this setting.

Next, we develop asymptotic confidence sequences
(Section~\ref{sec:asymp-cs}),
a time-uniform counterpart of
CLT-style confidence intervals \citep{waudby2021time}.
Unlike confidence intervals,
confidence sequences
are valid at all time steps simultaneously.
This gives the practitioner more flexibility as
the experiment can be continuously monitored and
the data collection can be adaptively stopped or continued.
Finally, we validate our methods experimentally
on causal effect estimation tasks (Section~\ref{sec:experiments}),
comparing our online strategies against fixed
data collection policies
on synthetic causal models (Section~\ref{sec:expr-synthetic})
and two real-world datasets (Section~\ref{sec:expr-real-world}).
We observe that the online strategies 
have lower regret and MSE and 
better coverage than the fixed policies.

\section{Related Work}
\label{sec:related}
There is a rich literature on semiparametric estimation
under data fusion for causal inference \citep{shi2023data, li2022robust}
and econometrics \citep{buchinsky2022estimation, ridder2007econometrics}.
Many works have studied the \emph{two-sample} instrumental
variable (IV) setting where the
IV, treatment, and outcome are not jointly observed
\citep{shuai2023identifying, shinoda2022estimation, sun2022semiparametric, zhao2019two, graham2016efficient, angrist1990effect}.
Others have studied data fusion for 
combining randomized control trial and observational datasets \citep{lin2024data, li2024efficient, colnet2024causal, carneiro2020optimal},
estimating long-term treatment effects \citep{chen2023semiparametric, ghassami2022combining, imbens2022long, li2021efficient, athey2020combining},
improving efficiency using external auxiliary datasets \citep{li2024identification, hu2022paradoxes, chen2022improving, li2021improving, evans2021doubly, yang2019combining, sturmer2005adjusting},
combining multiple IVs \citep{wu2023learning, burgess2016combining}, and
leveraging additional datasets with selection bias \citep{guo2022multi}.
\citet{graham2024unified} and \citet{li2021improving}
present a semiparametrically efficient data fusion strategy
under assumptions on the alignment of the multiple data
distributions. 
Our work is complementary and addresses the challenge of
sequentially allocating a given data collection budget 
across the given data sources.

Several works have studied the relative efficiencies of different adjustment sets
for causal inference
\citep{henckel2019graphical, rotnitzky2020efficient, witte2020efficient}.
Others have compared the relative efficiencies of the backdoor
and frontdoor estimators in linear \citep{gupta2021estimating, ramsahai2012supplementary}
and semiparametric causal models \citep{gorbach2023contrasting}.
These works show that the relative efficiencies
cannot always be known a priori and depend on the underlying model
parameters, motivating our work on online data collection.

Another related line of work studies inference
from adaptively collected data.
A commonly studied setting is adaptive experimental design,
where the probability of treatment is sequentially updated 
to identify the best treatment or reduce the variance of a causal effect estimator
\citep{li2023double, zhao2023adaptive, cook2023semiparametric, hadad2021confidence, kato2020adaptive, hahn2011adaptive}.
Others have extended this to the 
indirect experimentation setting \citep{zhao2024adaptive, morrison2024constrained, ailer2023sequential, chandak2023adaptive}.
\citet{lin2023semi} study semiparametric inference
for generalized linear regression from adaptively collected data.
While many of the theoretical tools we use are similar 
(like martingale asymptotics), 
our setting is different from these works.

Some works have studied the identification of causal effects
from data sources 
collected under heterogeneous conditions
\citep{bareinboim2016causal, hunermund2019causal}.
Others have studied identification 
with observational and interventional distributions involving
different sets of variables \citep{lee2024a, kivva2022revisiting, lee2021causal, lee2020identification, lee2020causal, tikka2019causal}.
In this work, we take identification for granted and
focus on efficient estimation.

Our work also shares motivation with \emph{active learning},
where the goal is to sequentially decide which data points to label
to learn a predictor or parameter efficiently \citep{zrnic2024active, zhao2012sequential, settles2009active, cohn1996active}.
Another related area is \emph{active feature acquisition},
where the goal is to incrementally acquire the 
most informative or cost-efficient feature subset
for training a predictive model \citep{li2021active, shim2018joint, hu2018survey, attenberg2011selective, saar2009active}.

The authors also studied OMS in \citet{gupta2021efficient}.
We build on this work in three key ways:
(1) we consider the semiparametric setting
allowing for flexible nuisance estimation,
(2) we make weaker assumptions, and 
(3) our results enable time-uniform inference.

\section{Online Moment Selection}
\label{sec:setup}
\subsection{Setup}

We represent the available data sources by
$\DM := ( \PB^{(i)} )_{i=1}^{\CardC}$,
a collection of probability distributions.
Querying a data source $\PB^{(i)} \in \DM$ is equivalent to
drawing an independent and identically distributed (i.i.d.)
sample from $\PB^{(i)}$.
The collection $\DM$ can include 
marginals over different subsets of variables, 
observational or interventional distributions, 
measurements under heterogeneous conditions, etc.
To simplify exposition, 
we make the following assumption 
(which we relax in Section~\ref{sec:oms-cost-structure} to allow for non-uniform costs):
\begin{assumption}[Equal cost]\label{assum:uniform-cost-structure}
Each sample from every data source has an equal cost.
\end{assumption}
We denote by $T$, the horizon or the known total number of queries the agent can make.
The \emph{selection vector}, 
denoted by $s_t \in \left\{ 0, 1 \right\}^{\CardC}$ for $t \in [T]$, 
is a random binary one-hot vector (i.e., $\sum_{d} s_{t, d} = 1$) 
indicating the data source queried by the agent at time $t$.
That is, $s_{t, j} = 1$ indicates that the agent queried $\PB^{(j)}$
at time $t$.
For convenience, we define $\PB(s_t) := \PB^{(j)}$.
Let $H_t = \{ Z_1, \hdots, Z_t \} \in \HM_t$ denote the \emph{history}, or the data collected until time $t$
(with $H_0 = \emptyset$),
where $Z_t$ is the sample from the data source queried at time $t$
and $\HM_t$ is the set of possible histories.
A \emph{data collection policy}, denoted by $\pi$, is a sequence of
functions $\pi_t : \HM_{t-1} \mapsto \{ 0, 1 \}^{\CardC}$ with
$s_t = \pi_t(H_{t-1})$.
Thus, $s_t$ can depend on
the past data $H_{t-1}$.

\paragraph{Notation.}
We use $o$ and $O$ to denote the classical order notation;
$\litOp$ ($\litOas$) to denote convergence in probability (almost surely);
and $\bigOp$ ($\bigOas$) to denote stochastic boundedness
in probability (almost surely).
The set $\ChoiceSimplex$ denotes the $(\CardC - 1)$ probability simplex
and $\kappa_{\text{ctr}}$ denotes the center of $\ChoiceSimplex$:
$\kappa_{\text{ctr}} := \left[ 1/{\CardC}, \hdots, 1/{\CardC} \right]^\top$.
We use $\|.\|$ to denote the spectral and $\text{\emph{l}}_2$
norms for matrices and vectors, respectively.
For two functions $\widehat{f}, f^{*} : \mathcal{X} \mapsto \RB$, we 
denote their $\text{\emph{l}}_2$ distance as
$\| \widehat{f} - f^{*}  \|^2 := \EB_{\PB(x)} [ (\widehat{f}(x) - f^{*}(x))^2 ]$.
If $\widehat{f}$ is an estimated function, 
then $\| \widehat{f} - f^{*}  \|$ will be a random variable
(with randomness over the estimation of $\widehat{f}$).
For a vector of functions 
$\widehat{\eta} = \left( \widehat{f}_1, \hdots, \widehat{f}_k \right)$
and $\eta^{*} = \left( f^{*}_1, \hdots, f^{*}_k \right)$,
we define $\| \widehat{\eta} - \eta^{*}  \| := \sum_{i=1}^{k} \| \widehat{f}_i - f^{*}_i  \|$.
We use $\BM_\epsilon(x)$ to denote the $\epsilon$-ball
around $x$: $\BM_\epsilon(x) = \{ x' : \| x - x'\| < \epsilon \}$.
For a vector $v$, we use $v_j$ to denote its $j^{\text{th}}$ coordinate.
We use $\NM(\mu, \sigma^2)$ to denote the normal distribution
with mean $\mu$ and variance $\sigma^2$;
and $\NM_{\sigma}(\mu, \sigma^2)$ to denote a 
mixture of normals (the mixture is taken over $\sigma$), 
with characteristic function $\EB_{\sigma}[\exp\{-i \mu t - \sigma^2 t / 2\}]$.
We use $\EB_{t-1}[.] := \EB[. | H_{t-1}]$ to denote the
conditional expectation given the past data $H_{t-1}$.

\paragraph{Constructing the moment conditions.}
Our framework is applicable 
if the target parameter can be identified 
by a set of moment conditions
such that each moment relies on samples
returned by at least one data source.
We assume that the moment conditions can be written as
\begin{align}
    g_t(\theta, \eta) = m(s_t) \odot \underbrace{\begin{bmatrix}
    \psi^{(1)}(Z^{(1)}_t; \theta, \eta^{(1)}) \\
    \vdots \\
    \psi^{(M)}(Z^{(M)}_t; \theta, \eta^{(M)})
    \end{bmatrix}}_{:= \Tilde{g}_t(\theta, \eta)} \in \RB^{M}, \label{eq:augmented-moment-conditions}
\end{align}
where $\odot$ is the element-wise product, 
$\theta \in \Theta \subset \RB^{D}$ is a
finite-dimensional parameter, and
$\eta := \left( \eta^{(1)}, \hdots, \eta^{(M)} \right) \in \TM$
are the nuisance parameters.
For brevity, we use $\psi^{(i)}_t(\theta, \eta) := \psi^{(i)}(Z^{(i)}_t; \theta, \eta)$.
At the true parameter values $(\theta^*, \eta^*)$, the moment conditions satisfy
\begin{align*}
    \forall i \in [M], \,\, \EB[ \psi^{(i)}(Z^{(i)}_t; \theta^*, \eta^*) ] = 0.
\end{align*}
We augment the original moment conditions
$\Tilde{g}_t(\theta, \eta)$ with the vector $m(s_t)$.
The function $m : \{0, 1\}^{\CardC} \mapsto \{0, 1\}^M$ is a fixed known
function that determines which moments get selected based
on the selection vector $s_t$.
That is, $m_i(s_t) = 1$ indicates that 
the moment condition $\psi^{(i)}$ can
be estimated from the data source selected in $s_t$
(i.e., $Z^{(i)}_t \overset{iid}{\sim} \PB(s_t)$).
The target parameter is $\beta^* := \fTar(\theta^*)$
for some known function $\fTar : \Theta \mapsto \RB$.

\paragraph{Estimating the nuisance parameters.}
Since the moment conditions depend on 
the nuisance parameters $\eta$, 
they need to be estimated. 
We make the following assumption:
\begin{assumption}\label{assumption:nuisance-common}
(a) The nuisance estimator can be constructed 
without knowledge of $\theta$ and 
(b) at time $t$, the nuisances are estimated 
using the data $H_{t-1}$,
that we denote by $\hEta$.
\end{assumption}
Assumption~\ref{assumption:nuisance-common}a
states that it is possible to construct an estimator
for $\eta$ independently of $\theta$
(see \citet{kallus2024localized} for more discussion and
a strategy for relaxing this condition).
Assumption~\ref{assumption:nuisance-common}b states that,
at time $t$, the nuisance estimator $\hEta$ only depends 
on data collected until time $t-1$.
This ensures that the nuisance estimators are trained and evaluated
on independent samples,
and plays a similar role as the \emph{sample splitting}
technique used to avoid Donsker assumptions on the nuisance function class
\citep[Sec.~4.2]{kennedy2022semiparametric}.

\paragraph{Estimating the target parameter.}
We estimate $\theta$ by plugging $\hEta$ into the moment conditions 
$g_t$ and minimizing the GMM objective $\widehat{Q}_T$:
\begin{align}
    & \widehat{\theta}_T = \arg\min_{\theta \in \Theta} \widehat{Q}_T(\theta, (\hEta)_{t=1}^{T}), \label{eq:gmm-estimator-argmin} \\
    \text{where} \,\,\, & 
    \widehat{Q}_T(\theta, (\hEta)_{t=1}^{T}) = \left[ \frac{1}{T} \sum_{t=1}^T 
    g_t(\theta, \hEta) \right]^\top \widehat{W}_T \left[ \frac{1}{T} \sum_{t=1}^T g_t(\theta, \hEta) \right], \nonumber
\end{align}
$\widehat{W}_T \in \RB^{M \times M}$ is a (possibly data dependent)
positive semidefinite matrix.
We then plug in $\widehat{\theta}_T$ to
estimate the target parameter: $\widehat{\beta}_T = \fTar(\widehat{\theta}_T)$.
We use the \emph{two-step} GMM estimator,
which is computed as follows:
we first compute the \emph{one-step} estimator 
$\widehat{\theta}^{(\text{os})}_T$ with $\widehat{W}_T = I$ (identity),
and then compute the two-step estimator
with $\widehat{W}_T = \widehat{\Omega}_T(\widehat{\theta}^{(\text{os})}_T)^{-1}$, where $ \widehat{\Omega}_T(\theta) = \left[ \sum_{t=1}^T g_t(\theta, \hEta) g_t(\theta, \hEta)^\top / T \right]$.
Informally, $\widehat{W}_T$ determines the importance given to each 
moment condition in the minimization problem and the choice
in the two-step estimator is asymptotically
the most efficient \citep[Sec.~5]{newey1994large}.

\subsection{Examples}

We give three examples to instantiate our
framework (see Appendix~\ref{sec:apdx-data-fusion-examples} for additional examples).
We begin with the parametric case (where $\eta$ is empty):
\begin{example}[Two-sample IV]\label{example:iv-graph}
Consider a linear IV causal model 
(Fig.~\ref{fig:disjoint-iv-graph}) 
with instrument $Z$, treatment $X$, and outcome $Y$
(and empty $W$);
with the following data-generating process:
\begin{align*}
    X &:= \alpha Z + \epsilon_X, \\
    Y &:= \beta X + \epsilon_Y, \\
    \epsilon_X &\notindependent \epsilon_Y,\, \epsilon_Y \independent Z,\, \epsilon_X \independent Z. 
\end{align*}
In the two-sample IV setting, we have 
two data sources that return an i.i.d. sample 
of $(Z, X)$ and $(Z, Y)$.
Thus, $\DM = \left\{ \PB(Z, X), \, \PB(Z, Y) \right\}$.
The target parameter $\beta^*$ is the 
average treatment effect (ATE) of $X$ on $Y$ 
that can be estimated using:
\begin{align*}
    g_t(\theta) = \underbrace{\begin{bmatrix} 
        s_{t, 1} \\
        s_{t, 2}
    \end{bmatrix}}_{=m(s_t)} \odot \underbrace{\begin{bmatrix}
        Z_t (X_t - \alpha Z_t) \\
        Z_t (Y_t - \alpha \beta Z_t)
    \end{bmatrix}}_{=\Tilde{g}_t(\theta)} = \begin{bmatrix}
        s_{t, 1} \\
        1 - s_{t, 1}
    \end{bmatrix} \odot \begin{bmatrix}
        Z_t (X_t - \alpha Z_t) \\
        Z_t (Y_t - \alpha \beta Z_t)
    \end{bmatrix},
\end{align*}
where $\theta = [\beta, \alpha]^\top$, and $\fTar(\theta) = \beta$.
\end{example}

\begin{figure}
\centering
\begin{subfigure}[b]{0.3\textwidth}
\centering
\begin{tikzpicture}
    \node[state] (1) {$X$};
    \node[state] (2) [left =of 1] {$Z$};
    \node[state] (3) [right =of 1] {$Y$};
    \node[state] (4) [left =of 1, xshift=1cm, yshift=1.5cm] {$W$};
    
    \path (2) edge node[]{} (1);
    \path (1) edge node[]{} (3);
    \path (4) edge node[]{} (1);
    \path (4) edge node[]{} (2);
    \path (4) edge node[]{} (3);
    \path[bidirected] (1) edge[bend left=65] node[el,above]{} (3);
\end{tikzpicture}
\caption{IV graph.}
\label{fig:disjoint-iv-graph}
\end{subfigure}
\hfill
\begin{subfigure}[b]{0.33\textwidth}
\centering
\begin{tikzpicture}
    \node[state] (1) {$W$};
    \node[state] (2) [left =of 1,yshift=-1.5cm] {$X$};
    \node[state] (3) [right =of 2] {$M$};
    \node[state] (4) [right =of 3] {$Y$};
    
    \path (1) edge node[above]{} (2);
    \path (1) edge node[above]{} (4);
    \path (2) edge node[above]{} (3);
    \path (3) edge node[above]{} (4);
\end{tikzpicture}
\caption{Confounder-mediator graph.}
\label{fig:confounder-mediator-graph}
\end{subfigure}
\hfill
\begin{subfigure}[b]{0.3\textwidth}
\centering
\begin{tikzpicture}
    \node[state] (1) {$W$};
    \node[state] (2) [right =of 1] {$U$};
    \node[state] (3) [left =of 1, xshift=1.5cm, yshift=-1.5cm] {$X$};
    \node[state] (4) [right =of 3, xshift=0.8cm] {$Y$};
    
    \path (1) edge node[above]{} (3);
    \path (1) edge node[above]{} (4);
    \path (2) edge node[above]{} (3);
    \path (2) edge node[above]{} (4);
    \path (3) edge node[above]{} (4);
    \path (1) edge[-,dashed] node[above]{} (2);
\end{tikzpicture}
\caption{Two confounders graph.}
\label{fig:two-covariates-graph}
\end{subfigure}
\caption{Examples of causal models---with treatment $X$ and outcome $Y$---where the ATE can be identified by different data sources returning different subsets of variables.}
\end{figure}
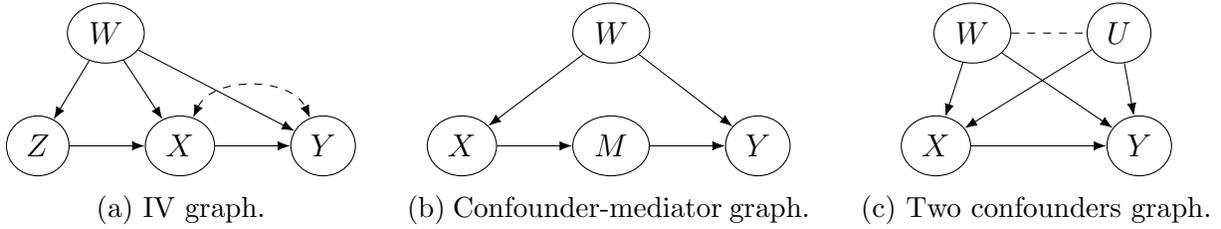

The following examples demonstrate the semiparametric case where 
the moment conditions are also indexed by nuisance parameters.

\begin{example}[Two-sample LATE]\label{example:two-sample-late}
Consider the IV causal model (Fig.~\ref{fig:disjoint-iv-graph})
with $\DM = \{ \PB(W, Z, X), \, \PB(W, Z, Y) \}$.
The target parameter $\beta^*$ is the unconditional local average treatment effect (LATE), 
which is the ATE on the compliers \citep{frolich2007nonparametric}.
The LATE can be estimated
using the following moment conditions \citep[Sec.~5.2]{chernozhukov2018double}:
\begin{align*}
    g_t(\theta, \eta) = \begin{bmatrix}
        s_{t, 1} \\
        1-s_{t, 1}
    \end{bmatrix} \odot \begin{bmatrix}
    \psi_{\text{AIPW}}(W_t, Z_t, Y_t; \eta^{(1)}) - \beta \alpha \\
    \psi_{\text{AIPW}}(W_t, Z_t, X_t; \eta^{(2)}) - \alpha
    \end{bmatrix},
\end{align*}
where $\psi_{\text{AIPW}}$ is the augmented inverse propensity score
influence function \citep[Sec.~3.4]{kennedy2016semiparametric},
$\theta = [\beta, \alpha]^\top$, and $\fTar(\theta) = \beta$. 
For $O = (W, Z, R) \in (\mathcal{W} \times \{0, 1\} \times \RB)$, $\psi_{\text{AIPW}}$ is
\begin{align}\label{eq:aipw-influence-func}
    \psi_{\text{AIPW}}(O) := \left( \frac{Z}{\pi^{*}(W)} - \frac{1 - Z}{1 - \pi^{*}(W)} \right) \left( R - \mu^{*}(Z, W) \right) + \left( \mu^{*}(1, W) - \mu^{*}(0, W) \right),
\end{align}
where $\pi^{*}(W) = \PB(Z=1|W)$ and
$\mu^{*}(Z, W) = \EB[R|Z, W]$
are the nuisance parameters.
\end{example}

\begin{example}\label{example:confounder-mediator}
Consider the confounder-mediator causal graph (Fig.~\ref{fig:confounder-mediator-graph})
with a binary treatment $X$, mediator $M$, outcome $Y$, confounder $W$.
The target parameter is the causal effect of $X$ on $Y$,
i.e., $\beta^* = \EB[Y|\text{do}(X=1)] - \EB[Y|\text{do}(X=0)]$ \citep{pearl2009causality}.
With $\DM = \{ \PB(W, X, Y), \PB(X, M, Y) \}$,
$\beta^*$ can be estimated with
the backdoor 
or the frontdoor criterion \citep[Sec.~3.3]{pearl2009causality}. 
The moment conditions can be written as
\begin{align*}
    g_t(\theta) = \begin{bmatrix}
        s_{t, 1} \\
        (1-s_{t, 1})
    \end{bmatrix} \odot \begin{bmatrix}
    \psi_{\text{AIPW}}((W_t, X_t, Y_t); \eta^{\text{(AIPW)}}) - \beta \\
    \psi_{\text{fd}}((X_t, M_t, Y_t); \eta^{\text{(fd)}}) - \beta
    \end{bmatrix},
\end{align*}
where $\psi_{\text{AIPW}}$ is defined in Eq.~\ref{eq:aipw-influence-func}, $\psi_{\text{fd}}$ is the efficient
influence function for the frontdoor criterion \citep{fulcher2020robust} (see Eq.~\ref{eq:apdx-frontdoor-eif} in Appendix~\ref{sec:apdx-experiments}),
$\theta = [\beta]$, and $\fTar(\theta) = \beta$.
\end{example}

\subsection{Consistency, Asymptotic Inference, and Regret}\label{sec:setup-cons-norm}

We present sufficient conditions for consistency (Prop.~\ref{prop:consistency-gmm})
and asymptotic inference (Props.~\ref{prop:asymp-normality}, \ref{prop:asymp-inference}) in the OMS setting.
We then use these results to define the regret of
a data collection policy (Definition~\ref{defn:regret}).
The proofs are deferred to Appendices~\ref{sec:apdx-proof-consistency} and \ref{sec:apdx-proof-asymp-norm}.

\begin{assumption}\label{assump:standard-gmm}
Let $\mathcal{I} = \{ i \in [M] : \liminf_{T \to \infty} \sum_{t=1}^{T} m_i(s_t) / T > 0 \}$.
(a) (Identification)
$\forall \theta \neq \theta^*, \, \EB\left[ \left\|  \Tilde{g}_{t, \mathcal{I}}\left( \theta, \eta^* \right) \right\| \right] \neq 0$,
where $\Tilde{g}_{t, \mathcal{I}} = [\Tilde{g}_{t, i}]_{i \in \mathcal{I}}$ 
is the subset of moments determined by $\mathcal{I}$;
(b) $\Theta \subset \RB^D$ is compact; and
(c) $\fTar$ is continuously differentiable at $\theta^*$.
\end{assumption}
The index set $\mathcal{I}$ denotes the moments
that get selected an asymptotically non-negligible fraction 
of times by the policy.
Assumption~\ref{assump:standard-gmm}a states that 
the moment conditions selected in $\mathcal{I}$
are sufficient to uniquely identify $\theta^*$
\citep[Sec.~2.2.3]{newey1990semiparametric}.
If there are as many moments as parameters (i.e., $M = D$),
this requires that $|\mathcal{I}| = M$.

\begin{property}[ULLN]\label{property:ulln}
For a function $a_t(\theta, \eta) := a(X_t; \theta, \eta) \in \RB$ 
with $X_t$ sampled i.i.d.,
we say that $a_t(\theta, \eta)$ satisfies the ULLN property if 
(i) $\exists \delta > 0, \, \forall \theta, \EB[ |a_t(\theta, \eta^{*})|^{2 + \delta} ] < \infty$;
(ii) (Lipschitzness) For some constant $L$, $\forall \theta, \theta' \in \Theta, \forall
\eta \in \TM, \, |a_t(\theta', \eta) - a_t(\theta, \eta) | \leq L(X_t) \| \theta' - \theta \|$ with $\EB[L^2(X_t)] \leq L$
and 
$\forall \theta \in \Theta, \forall
\eta, \eta' \in \TM, \, \EB[(a_t(\theta, \eta) - a_t(\theta, \eta'))^2] \leq L \| \eta - \eta' \|^2$;
(iii) (Nuisance consistency) $\| \hEta - \eta^* \| \ConvProb 0$; and
(iv) $\exists \delta > 0, \sup_{\eta \in \TM} \EB[\| \eta - \eta^* \|^{2 + \delta}] < \infty$.
\end{property}

\begin{proposition}[Consistency]\label{prop:consistency-gmm}
Suppose that (i) Assumption~\ref{assump:standard-gmm} holds;
(ii) $\forall i \in [M], \psi^{(i)}$ satisfies Property~\ref{property:ulln};
(iii) $\forall (i, j) \in [M]^2, \psi^{(i)} \psi^{(j)}$ satisfies Property~\ref{property:ulln}.
Then $\widehat{\beta}_T \ConvProb \beta^*$.
\end{proposition}

For $\psi_{\text{AIPW}}$ (Eq.~\ref{eq:aipw-influence-func}),
Property~\ref{property:ulln}(ii) holds under 
boundedness of the data and nuisance parameters \citep[Example~2 (Sec.~4.2)]{kennedy2022semiparametric}.
Condition~(iii) of Prop.~\ref{prop:consistency-gmm}
will hold if $\forall i \in [M], \psi^{(i)}$ are uniformly bounded
(see Prop.~\ref{prop:apdx-ulln-product-fn} in Appendix~\ref{sec:apdx-proof-consistency}).

\begin{assumption}[Nuisance estimation]\label{assum:nuisance-rates}
For all $i \in [M]$,
(a) (Neyman orthogonality) $\forall \eta, \,\, \partial_{r} \EB[\psi^{(i)}_t(\theta^*, \eta^* + r ( \eta - \eta^*))]|_{r=0} = 0$; and
(b) (Second-order remainder) For $R_t := \int_0^1 \partial_{r^2} \EB_{t-1}[ \psi^{(i)}_t(\theta^*, \eta^* + r(\hEta - \eta^*)) ] \, dr$, $R_t = \litOp(t^{-1/2})$ and $\exists \delta > 0, \sup_t \EB[R^{1 + \delta}_t] < \infty$.
\end{assumption}
Assumption~\ref{assum:nuisance-rates}a states that 
every moment condition
is robust to local perturbations of nuisance parameters up to the first
order and \ref{assum:nuisance-rates}b states that
the second-order
remainder converges at a faster-than-CLT rate.
Assumption~\ref{assum:nuisance-rates} 
ensures that the impact of 
the nuisance estimators $\widehat{\eta}$ 
on the estimate of $\theta$ is higher order
\citep{chernozhukov2018double, belloni2017program, neyman1959optimal}.
For $\psi_{\text{AIPW}}$ (Eq.~\ref{eq:aipw-influence-func}),
$R_t$ scales as the product of the errors of the two nuisance estimators:
$R_t = O(\| \widehat{\pi}_{t-1} - \pi^* \| \| \widehat{\mu}_{t-1} - \mu^* \|)$,
allowing each nuisance estimator to converge at a 
slower rate of $o_p( t^{-1/4} )$
\citep[Sec.~4.3]{kennedy2022semiparametric}.
The phenomenon of $R_t$ taking a product form,
also known as \emph{double robustness}, is more general
and holds for many influence functions \citep{chernozhukov2022locally, bhattacharya2022semiparametric, rotnitzky2021characterization}.
Convergence rates exist for many estimation problems \citep{gao2022towards, farrell2021deep, kohler2021rate, wainwright2019high, gyorfi2002distribution}.

\begin{definition}[Selection simplex]\label{defn:selection-simplex}
The selection simplex, denoted by $\kappa_T \in \ChoiceSimplex$, represents the
fraction of times each data source has been queried until time $T$:
\begin{align*}
    \kappa_T := \frac{1}{T} \sum_{t=1}^{T} s_t.
\end{align*}
\end{definition}

\begin{assumption}\label{assump:kappa-T-converges}
The policy $\pi$ is such that
$\kappa_T \ConvProb \kappa_{\infty}$,
for some random variable $\kappa_{\infty}$.
\end{assumption}

\begin{proposition}[Asymptotic normality]\label{prop:asymp-normality}
Suppose that (i) Conditions of Prop.~\ref{prop:consistency-gmm} hold;
(ii) Assumptions~\ref{assum:nuisance-rates} and \ref{assump:kappa-T-converges} hold; and
(iii) $\forall i \in [M], j \in [D], \partial_{\theta_j} 
\left[\psi^{(i)}_t(\theta, \eta) \right]$
satisfies Property~\ref{property:ulln}.
Then $\widehat{\beta}_T$ converges to a mixture of normals,
where the mixture is over $\kappa_{\infty}$:
\begin{align*}
    \sqrt{T} (\widehat{\beta}_T - \beta^*) &\ConvDist \NM_{\kappa_\infty}\left( 0, V_{*}(\kappa_\infty) \right),
\end{align*}
where $V_{*}(\kappa)$ is a constant
that depends on $(\theta^*, \eta^*, \kappa)$ 
(see Eq.~\ref{eq:apdx-var-star-expression} in Appendix~\ref{sec:apdx-proof-asymp-norm}).
If $\kappa_{\infty}$ is almost surely constant, then
$\widehat{\beta}_T$ is asymptotically normal.
\end{proposition}

\begin{proposition}[Asymptotic inference]\label{prop:asymp-inference}
Suppose that the Conditions of Prop.~\ref{prop:asymp-normality} hold.
For
\begin{align}
    \widehat{G}_T(\widehat{\theta}_T) &:= \frac{1}{T} \sum_{t=1}^{T} \nabla_{\theta} g_t(\widehat{\theta}_T, \hEta), \nonumber \\
    \widehat{\Omega}_T(\widehat{\theta}_T) &:= \frac{1}{T} \sum_{t=1}^{T}  g_t(\widehat{\theta}_T, \hEta) g^{\top}_t(\widehat{\theta}_T, \hEta), \nonumber \\
    \widehat{\Sigma}_T &:= \left[ \widehat{G}^{\top}_T(\widehat{\theta}_T) \widehat{\Omega}^{-1}_T(\widehat{\theta}_T) \widehat{G}_T(\widehat{\theta}_T) \right]^{-1}, \nonumber \\
    \hVar_{T} &:= \nabla_\theta \fTar(\widehat{\theta}_T)^\top \widehat{\Sigma}_T \nabla_\theta \fTar(\widehat{\theta}_T), \label{eq:variance-estimator}
\end{align}
we have 
\begin{align*}
    \hVar^{-1/2}_{T} \sqrt{T} (\widehat{\beta}_T - \beta^*) \ConvDist \NM\left(0, 1 \right).
\end{align*}
\end{proposition}

A \emph{fixed policy}, denoted by $\pi_{\kappa}$,
queries each data source
a fixed pre-specified fraction of times
with $\kappa_T = \kappa$ for some constant $\kappa \in \ChoiceSimplex$. 
The \emph{oracle policy}, denoted by $\pi^{*}$, is 
the (unknown) fixed policy
with the lowest asymptotic variance.
For $\pi^{*}$, we have
$\kappa_T = \kappa^*$, where $\kappa^* = \arg\min_{\kappa} V_{*}(\kappa)$.
We call $\kappa^*$ the \emph{oracle simplex}.
The following assumption,
which we make throughout this work,
states that no two combinations of the
data sources minimize the asymptotic variance,
ensuring the uniqueness of $\kappa^*$.

\begin{assumption}\label{assump:kappa-star-identify}
$\kappa^*$ uniquely minimizes $V_{*}(\kappa)$: $\forall \kappa$ such that $\kappa \neq \kappa^*, \,\, V_{*}(\kappa) > V_{*}(\kappa^*)$.
\end{assumption}

\begin{remark}\label{remark:data-collection-order}
Since the asymptotic distribution of $\widehat{\beta}_T$ 
only depends on the limit of $\kappa_T$ (Prop.~\ref{prop:asymp-normality}),
the order in which the data sources are queried
does not matter for our results.
We can also use randomized fixed policies with $s_t \overset{\text{iid}}{\sim} \text{Multinoulli}(\kappa)$.
\end{remark}

\begin{definition}[Asymptotic regret]\label{defn:regret}
We define the asymptotic regret of a policy $\pi$ as
\begin{align}\label{eq:asymptotic-regret}
    R_\infty(\pi) = \Amse\left( \sqrt{T} \left(\widehat{\beta}^{(\pi)}_T - \beta^* \right) \right) - V_{*}(\kappa^*),
\end{align}
where $\Amse$ is the asymptotic MSE, i.e., the MSE of the limiting distribution.
\end{definition}
The regret measures how close the
(asymptotic) MSE of $\pi$ is to the MSE of the oracle policy.
Under Assumption~\ref{assump:kappa-T-converges}, the
selection simplex $\kappa_T$ converges, and therefore,
$\widehat{\beta}_T$ will have a limiting distribution (Prop.~\ref{prop:asymp-normality}),
making the regret well-defined.
Comparing estimators based on their limiting distributions
is common in asymptotic statistics,
where the focus is usually 
on regular and asymptotically linear estimators \citep{van2000asymptotic}.

The next proposition shows that the asymptotic regret of 
any data collection policy satisfying the assumptions of Prop.~\ref{prop:asymp-normality}
is non-negative.
The GMM estimator is efficient 
in the statistical model implied by the moment conditions.
Thus, if the chosen moment conditions are semiparametrically efficient, 
the oracle policy will achieve the semiparametric efficiency bound
\citep{ackerberg2014asymptotic, newey1990semiparametric, chamberlain1987asymptotic}.

\begin{proposition}\label{prop:regret-is-non-negative}
Suppose that (i) $\widehat{\theta}_T$ is estimated using Eq.~\ref{eq:gmm-estimator-argmin}, 
(ii) Assumption~\ref{assump:kappa-star-identify} holds,
and
(iii) the conditions of Prop.~\ref{prop:asymp-normality} hold.
Then, for any data collection policy $\pi$, $R_{\infty}(\pi) \geq 0$.
\end{proposition}

For any fixed policy $\pi_\kappa$ with 
$\kappa_T = \kappa$ for some constant $\kappa \neq \kappa^*$
suffers constant regret because
by Assumption~\ref{assump:kappa-star-identify},
we have $R_\infty(\pi_\kappa) > 0$.
The following lemma shows that policies with 
$\kappa_T \ConvProb \kappa^*$ achieve zero regret,
motivating online data collection.

\begin{lemma}\label{lemma:zero-regret-policy}
Suppose that the conditions of Prop.~\ref{prop:asymp-normality} hold.
Any data collection policy $\pi$ such that $\kappa_T \ConvProb \kappa^*$
has zero regret: $R_\infty(\pi) = 0$.
\end{lemma}
\begin{proof}
By Prop.~\ref{prop:asymp-normality},
we have
$\sqrt{T} \left( \beta_T - \beta^* \right) \ConvDist \NM(0, V_{*}(\kappa^*))$ and therefore, $R_\infty(\pi) = 0$.
\end{proof}

\section{Online Data Collection}
\label{sec:adaptive-data-collection}
In this section, we propose two online data collection policies:
Explore-then-Commit (ETC) and Explore-then-Greedy (ETG),
that have zero regret.
In both policies, we allocate the data collection
budget based on an estimate of the oracle simplex.

\subsection{OMS via Explore-then-Commit (OMS-ETC)}
\label{sec:oms-etc}
\begin{figure}[t]
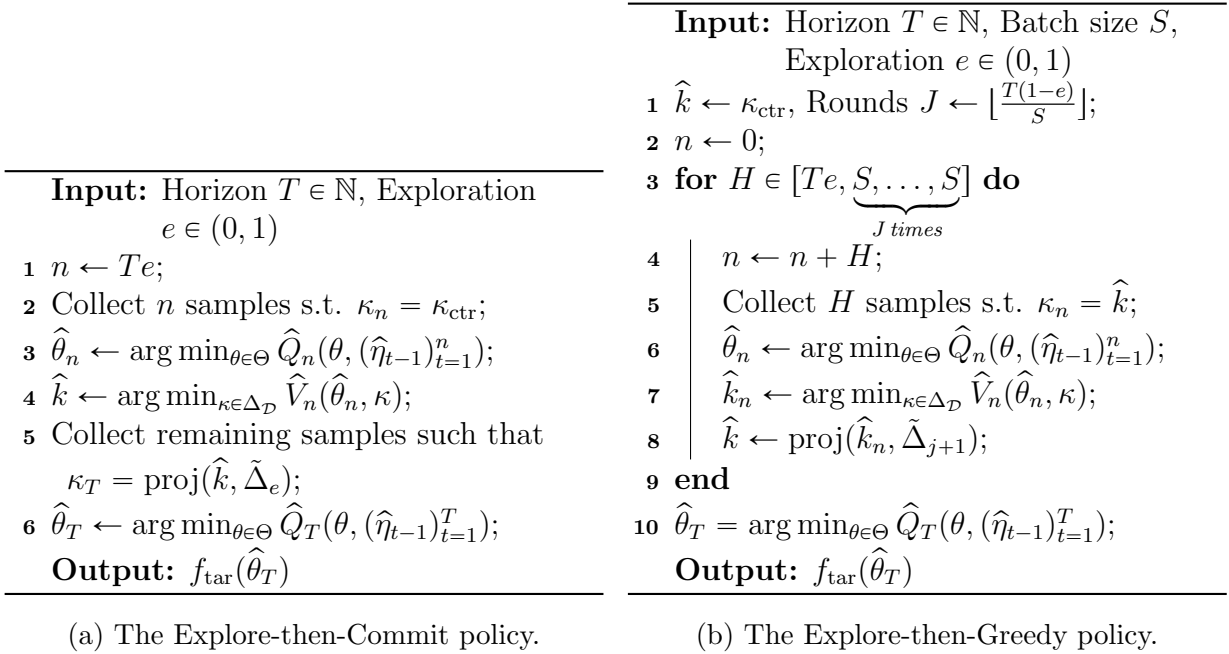

\centering
\begin{subfigure}[b]{0.49\textwidth}
{
    \setlength{\interspacetitleruled}{0pt}%
    \setlength{\algotitleheightrule}{0pt}%
    \begin{algorithm}[H]
    \SetAlgoLined
    \KwInput{Horizon $T \in \mathbb{N}$, Exploration $e \in (0, 1)$}
    $n \gets Te$\;
    Collect $n$ samples s.t. $\kappa_{n} = \ctrSim$\;
    $\widehat{\theta}_{n} \gets \arg\min_{\theta \in \Theta} \widehat{Q}_n(\theta, (\hEta)_{t=1}^{n})$\;
    $\widehat{k} \gets \arg\min_{\kappa \in \ChoiceSimplex} \hVar_n(\widehat{\theta}_{n}, \kappa)$\;
    Collect remaining samples such that $\kappa_T = \text{proj}(\widehat{k}, \Tilde{\Delta}_e)$\;
    $\widehat{\theta}_T \gets \arg\min_{\theta \in \Theta} \widehat{Q}_T(\theta, (\hEta)_{t=1}^{T})$\;
    \KwOutput{$f_{\text{tar}}(\widehat{\theta}_T)$}
    \end{algorithm}
}
\caption{The Explore-then-Commit policy.}
\label{fig:policy-algorithm-etc}
\end{subfigure}
\hfill
\begin{subfigure}[b]{0.49\textwidth}
{
    \setlength{\interspacetitleruled}{0pt}%
    \setlength{\algotitleheightrule}{0pt}%
    \begin{algorithm}[H]
    \SetAlgoLined
    \KwInput{Horizon $T \in \mathbb{N}$, Batch size $S$, Exploration $e \in (0, 1)$}
    $\widehat{k} \gets \ctrSim$, Rounds $J \gets 
\floor{\frac{T(1-e)}{S}}$\;
    $n \gets 0$\;
    \For{$H \in [Te, \underbrace{S, \hdots, S}_{J \, \text{times}}]$}{
        $n \gets n + H$\;
        Collect $H$ samples s.t. $\kappa_{n} = \widehat{k}$\;
        $\widehat{\theta}_{n} \gets \arg\min_{\theta \in \Theta} \widehat{Q}_{n}(\theta, (\hEta)_{t=1}^{n})$\;
        $\widehat{k}_{n} \gets \arg\min_{\kappa \in \ChoiceSimplex} \hVar_n(\widehat{\theta}_{n}, \kappa)$\;
        $\widehat{k} \gets \text{proj}(\widehat{k}_{n}, \Tilde{\Delta}_{j+1})$\;
    }
    $\widehat{\theta}_T = \arg\min_{\theta \in \Theta} \widehat{Q}_T(\theta, (\hEta)_{t=1}^{T})$\;
    \KwOutput{$f_{\text{tar}}(\widehat{\theta}_T)$}
    \end{algorithm}
}
\caption{The Explore-then-Greedy policy.}
\label{fig:policy-algorithm-etg}
\end{subfigure}
\caption{Algorithms for OMS-ETC and OMS-ETG.}
\end{figure}

The Explore-then-Commit (ETC) policy is similar in spirit
to the ETC strategy in multi-armed bandits \citep[Ch.~6]{lattimore2020bandit}.
We denote this policy by $\pi_{\text{ETC}}$ 
(Fig.~\ref{fig:policy-algorithm-etc}).
The ETC policy is characterized by a horizon $T$ and an exploration
fraction $e \in (0, 1)$.
In the \emph{explore} phase, we collect $Te$ samples
by querying the data sources uniformly,
i.e., $\kappa_{Te} = \kappa_{\text{ctr}}$.
Using the exploration samples, we estimate the model parameters
$\widehat{\theta}$,
and use them to estimate the asymptotic variance of some fixed
policy $\pi_{\kappa}$ as a function of $\kappa$.
Let the estimated variance be $\hVar_{Te}(\widehat{\theta}, \kappa)$ 
(Definition~\ref{defn:offpolicy-var-estimate}).
Next, we estimate the oracle simplex $\kappa^*$ as follows:
$\widehat{k}_{Te} = \arg\min_{\kappa} \hVar_{Te}(\widehat{\theta}, \kappa)$.
In the \emph{commit} phase, we collect the remaining $T(1-e)$ samples
such that the final selection simplex $\kappa_T$ 
is as close to $\widehat{k}_{Te}$ as possible:
$\kappa_T = \text{proj}( \widehat{k}_{Te}, \Tilde{{\Delta}}_e)$,
where $\Tilde{{\Delta}}_e := \left\{ e \kappa_{\text{ctr}} + (1 - e) \kappa : \kappa \in \ChoiceSimplex  \right\}$ 
is the set of values $\kappa_T$ can achieve with
the remaining budget.
As stated in Remark~\ref{remark:data-collection-order},
the order in which the data sources are queried in the two phases
does not affect our results.

\begin{definition}[Variance estimator]\label{defn:offpolicy-var-estimate}
For any $\kappa \in \ChoiceSimplex$, we define
\begin{align}
    & \widehat{G}_T(\theta, \kappa) := w_G(\kappa) \odot \widehat{G}_T(\theta), \nonumber \\
    & \widehat{\Omega}_T(\theta, \kappa) := w_{\Omega}(\kappa) \odot \widehat{\Omega}_T(\theta), \nonumber \\
    & \widehat{\Sigma}_T(\theta, \kappa) := \left[ \widehat{G}^{\top}_T(\theta, \kappa) \widehat{\Omega}^{-1}_T(\theta, \kappa) \widehat{G}_T(\theta, \kappa) \right]^{-1}, \nonumber \\
    & \hVar_{T}(\theta, \kappa) := \nabla_\theta \fTar(\theta)^\top \widehat{\Sigma}_T(\theta, \kappa) \nabla_\theta \fTar(\theta) \label{eq:off-policy-variance-estimator},
\end{align}
where $\widehat{G}_T(\theta)$ and $\widehat{\Omega}_T(\theta)$
are defined in Prop.~\ref{prop:asymp-inference}.
The matrices $w_G(\kappa)$ and $w_{\Omega}(\kappa)$
(see Eq.~\ref{eq:apdx-var-estimator-defn} 
in Appendix~\ref{sec:apdx-oms-etc} for the expression)
appropriately reweight the variance estimator 
$\hVar_T$ in Eq.~\ref{eq:variance-estimator}
such that $\hVar_T(\widehat{\theta}_T, \kappa)$ estimates
the variance that results from
the fixed policy $\pi_{\kappa}$.
\end{definition}
We can use $\hVar_T(\widehat{\theta}_T, \kappa)$ to consistently estimate $\kappa^*$ (see Appendix~\ref{sec:apdx-oms-etc} for the proof):
\begin{lemma}\label{lemma:oracle-simplex-estimation}
Let $\widehat{k}_T := \arg\min_{\kappa \in \ChoiceSimplex} \hVar_{T}(\widehat{\theta}_T, \kappa)$
be the estimated oracle simplex.
Suppose that 
(i) the conditions of Prop.~\ref{prop:consistency-gmm} hold,
(ii) Assumption~\ref{assump:kappa-star-identify} holds,
and 
(iii) $\forall i \in [M], j \in [D], \partial_{\theta_j} 
\left[\psi^{(i)}_t(\theta, \eta) \right]$
satisfies Property~\ref{property:ulln}.
Then $\widehat{k}_T \ConvProb \kappa^*$.
\end{lemma}

\begin{assumption}[Exploration]\label{assumption:exploration}
The exploration $e$ depends on $T$ such that
(i) (Asymptotically negligible) $e = o(1)$; and
(ii) $Te \rightarrow \infty$ as $T \to \infty$ (e.g., $e = 1/\sqrt{T}$).
\end{assumption}

\begin{theorem}[ETC Regret]\label{thm:etc-regret}
Suppose that the conditions of Prop.~\ref{prop:asymp-normality} and
Lemma~\ref{lemma:oracle-simplex-estimation} hold,
and Assumption~\ref{assumption:exploration} holds.
Then, $R_\infty(\pi_{\text{ETC}}) = 0$.
\end{theorem}
\begin{proof}
We show that 
$\kappa_T \ConvProb \kappa^*$ and then apply 
Lemma~\ref{lemma:zero-regret-policy}.
Since $Te \to \infty$,
by Lemma~\ref{lemma:oracle-simplex-estimation}, $\widehat{k}_{Te} \ConvProb \kappa^*$.
Since $e \in o(1)$, the set $\Tilde{\Delta}_e$
asymptotically covers the entire simplex $\ChoiceSimplex$.
Thus,
\begin{align*}
    \sup_{\kappa \in \ChoiceSimplex} \| \text{proj}(\kappa, \Tilde{\Delta}_e ) - \kappa \| \ConvProb 0.
\end{align*}
Therefore, for $\kappa_T = \text{proj}(\widehat{k}_{Te}, \Tilde{\Delta}_e)$, 
we have $\| \kappa_T - \widehat{k}_{Te} \| \ConvProb 0$.
Since $\widehat{k}_{Te} \ConvProb \kappa^*$, we also
have $\kappa_T \ConvProb \kappa^*$.
Applying Lemma~\ref{lemma:zero-regret-policy} completes the proof.
\end{proof}

\subsection{OMS via Explore-then-Greedy (OMS-ETG)}
\label{sec:oms-etg}
The Explore-then-Greedy (ETG) policy, 
denoted by $\pi_{\text{ETG}}$,
extends the ETC policy by
repeatedly updating the estimate of the oracle simplex instead
of committing to a fixed value after exploration (Fig.~\ref{fig:policy-algorithm-etg}).
The ETG policy is characterized by a horizon $T$, 
an exploration fraction $e \in (0, 1)$,
and a batch size $S$. 
Like ETC, we first \emph{explore} and collect $Te$ samples
by querying the data sources uniformly: $\kappa_{Te} = \kappa_{\text{ctr}}$.
We collect the remaining $T(1 - e)$ samples in 
batches of
$S$ samples
at a time, i.e., there are $J := (T(1-e)) / S$ rounds
after exploration.
After each round $j \in \{0, \hdots, J\}$ 
($j = 0$ denotes exploration), 
we update our estimate of the oracle
simplex: for $t_j := (Te + jS)$, we compute 
$\widehat{k}_{t_j} = \arg\min_{\kappa} \hVar_{t_j}(\widehat{\theta}_{t_j}, \kappa)$.
The set of values that $\kappa_{t_{j+1}}$ can achieve
is $\Tilde{\Delta}_{j+1} = \left\{ \frac{t_j \kappa_{t_j} + S \kappa}{t_{j+1}} : \kappa \in \ChoiceSimplex \right\}$.
In the subsequent round, we (greedily) query the data sources 
to get as close as possible to $\widehat{k}_{t_j}$:
we collect the next $S$ samples such that 
$\kappa_{t_{j+1}} = \text{proj}( \widehat{k}_{t_j}, \Tilde{\Delta}_{j+1} )$.
The proofs are deferred to Appendix~\ref{sec:apdx-oms-etg}.

In Theorem~\ref{thm:etg-finite-rounds}, we show
zero regret when the number of rounds $J$ is finite.
\begin{theorem}\label{thm:etg-finite-rounds}
Suppose that
(i) the conditions of Theorem~\ref{thm:etc-regret} hold
and
(ii) (Finite rounds) $\lim_{T \to \infty} S / T = r$
for some constant $r \in (0, 1)$.
Then, $R_\infty(\pi_{\text{ETG}}) = 0$.
\end{theorem}

Next, under the 
assumption of strongly consistent nuisance estimation, 
we prove zero regret for any batch size $S$ (Theorem~\ref{thm:etg-infinite-rounds}).
Thus, $S$ can depend on $T$
such that $S / T = o(1)$.
For example, we can set $S = s$ for some constant $s$ or $S \asymp \sqrt{T}$.

\begin{lemma}\label{lemma:oracle-simplex-estimation-almost-sure}
Suppose that (i) the conditions of Lemma~\ref{lemma:oracle-simplex-estimation} hold and
(ii) (Nuisance strong consistency) $\| \hEta - \eta^* \| \ConvAS 0$.
Then $\widehat{\theta}_T \ConvAS \theta^*$
and $\widehat{k}_T \ConvAS \kappa^*$.
\end{lemma}

\begin{theorem}\label{thm:etg-infinite-rounds}
Suppose that (i) the conditions of Theorem~\ref{thm:etc-regret} hold and
(ii) the conditions of Lemma~\ref{lemma:oracle-simplex-estimation-almost-sure} hold.
Then, for any batch size $S$, $R_\infty(\pi_{\text{ETG}}) = 0$.
\end{theorem}

\begin{remark}\label{remark:strong-consistency-nuisance}
Lemma~\ref{lemma:oracle-simplex-estimation-almost-sure} requires strongly consistent nuisance estimators.
Almost sure convergence guarantees exist for some 
nonparametric problems \citep{wu2020consistency, walk2010strong, blondin2007rates, francisco2003uniform, liero1989strong, cheng1984strong}.
We also illustrate 
how non-asymptotic tail bounds from statistical learning
theory along with the Borel–Cantelli lemma
can be used to show strong consistency (see Appendix~\ref{sec:apdx-nuisance-convergence-rates}).
\end{remark}

The ETG policy can be further generalized to an
$\epsilon$-greedy strategy (Fig.~\ref{fig:policy-algorithm-eps-greedy} in Appendix~\ref{sec:apdx-oms-etg}).
This policy is characterized by an exploration
policy $(\epsilon_t)_{t=1}^{\infty}$ where $\epsilon_t \in [0, 1]$.
At each time $t$, with probability $\epsilon_t$,
we select a data source uniformly at random (i.e., explore)
and act greedily with probability $(1 - \epsilon_t)$.

\begin{theorem}\label{thm:eps-greedy}
Let $(\epsilon_t)_{t=1}^{\infty}$ be a non-increasing sequence and
$E_T = \sum_{t=1}^{T} \epsilon_t$.
Suppose that (i) the conditions of Prop.~\ref{prop:asymp-normality} hold; 
(ii) the conditions of Lemma~\ref{lemma:oracle-simplex-estimation-almost-sure} hold;
and
(iii) $E_T = o(T)$ and 
$E_T \to \infty$ (e.g., $\epsilon_t \asymp 1 / t$).
Then, the $\epsilon$-greedy policy suffers zero regret: 
$R_\infty(\pi_{\epsilon\text{-greedy}}) = 0$.
\end{theorem}
The $\epsilon$-greedy policy does not 
require the horizon $T$ to be specified in advance
(unlike ETC and ETG, where the exploration depends on $T$).
This is useful for performing time-uniform
inference, enabling the agent to adaptively
stop or continue their data collection
(we discuss this in more detail in Section~\ref{sec:asymp-cs}).

\section{OMS with a Cost Structure}
\label{sec:oms-cost-structure}
In many real-world settings, the agent must pay a 
different cost to query each data source.  
We adapt OMS-ETC and OMS-ETG to this setting where a cost structure is associated with the data sources in $\DM$
and show that these policies still have zero asymptotic regret.
We denote the (known) budget by $B \in \NB$ and 
the cost vector by $c \in \RB^{\CardC}_{>0}$,
where $c_i$ is the cost of querying the data source $\PB^{(i)} \in \DM$.
Due to the cost structure, the horizon $T$ is a random variable dependent 
on $\pi$ with $T = \max \{ t \in \NB : t (\kappa^\top_t c) \leq B  \}$.
The setting in Section~\ref{sec:setup} is a special case 
with $\forall i, \, c_i = 1$ and $T = B$.
The proofs are deferred to Appendix~\ref{sec:apdx-cost-proofs}.

\begin{example}[Combining observational datasets]\label{example:combine-observational-datasets}
Consider the task of estimating the ATE of
a binary treatment $X$ on an outcome $Y$
where an unconfounded dataset is combined with a cheaper 
confounded dataset \citep{yang2019combining} 
(see the causal graph in Fig.~\ref{fig:two-covariates-graph}).
For $\DM = \{ \PB(U, W, X, Y), \, \PB(W, X, Y) \}$, the moment conditions are:
\begin{align*}
    g_t(\theta) = \begin{bmatrix}
        s_{t, 1} \\
        s_{t, 1} \\
        1-s_{t, 1}
    \end{bmatrix} \odot \begin{bmatrix}
        \psi_{\text{AIPW}}\left( \{W_t, U_t\}, X_t, Y_t; \eta^{(1)} \right) - \beta \\
        \psi_{\text{AIPW}}\left( W_t, X_t, Y_t; \eta^{(2)} \right) - \alpha \\
        \psi_{\text{AIPW}}\left( W_t, X_t, Y_t; \eta^{(2)} \right) - \alpha
    \end{bmatrix},
\end{align*}
where $\psi_{\text{AIPW}}$ is defined in Eq.~\ref{eq:aipw-influence-func}, 
$\theta = [\beta, \alpha]^\top$, and $\fTar(\theta) = \beta$.
The cost structure is $c = [c_1, c_2]^\top$ with
$c_1 > c_2$.
\end{example}

\begin{proposition}\label{prop:asymp-norm-cost}
Suppose that the conditions of Prop.~\ref{prop:asymp-normality} hold. Then,
\begin{align*}
    \sqrt{B} \left(\widehat{\beta}_T - \beta^* \right) \ConvDist \NM_{\kappa_{\infty}}\left(0, V_{*}(\kappa_{\infty}) \cdot \left( \kappa_{\infty}^\top c \right) \right).
\end{align*}
\end{proposition}
We scale by $\sqrt{B}$ instead of $\sqrt{T}$ to make comparisons across policies meaningful. 
The \textit{oracle simplex} is now defined as $\kappa^* := \arg\min_{\kappa \in \ChoiceSimplex} \left[V_{*}(\kappa)  \cdot \left( \kappa^\top c \right)\right]$ and the \textit{asymptotic regret} of a data collection policy $\pi$ is now
\begin{align*}
    R_\infty(\pi) = \Amse\left( \sqrt{B} \left(\widehat{\beta}^{(\pi)}_T - \beta^* \right) \right) - V_{*}(\kappa^*)  \cdot \left((\kappa^*)^\top c \right).
\end{align*}
OMS-ETC-CS (\textit{OMS-ETC with cost structure}), denoted by $\pi_{\text{ETC-CS}}$ 
extends OMS-ETC to this setting 
(Fig.~\ref{fig:policy-algorithm-etc-cs} in Appendix~\ref{sec:apdx-cost-structure}).
We use $Be$ budget to explore and estimate the oracle simplex $\kappa^*$ 
by $\widehat{k} = \arg\min_{\kappa \in \ChoiceSimplex} V(\widehat{\theta}_{T_e}, \kappa) \left( \kappa^\top c \right)$, where $T_e = \floor*{\frac{Be}{\kappa^\top_{T_e} c }}$ and $\kappa_{T_e} = \ctrSim$.
With the remaining budget, we collect samples such that 
$\kappa_T$ gets as close as possible to $\widehat{k}$.

\begin{proposition}\label{prop:regret-etc-cost-structure}
Suppose that the conditions of Prop.~\ref{prop:asymp-normality} and
Lemma~\ref{lemma:oracle-simplex-estimation} hold.
If $e = o(1)$ and $Be \rightarrow \infty$ as $B \to \infty$, then
$R_{\infty}(\pi_{\text{ETC-CS}}) = 0$.
\end{proposition}

Next, we propose OMS-ETG-CS (OMS-ETG with cost structure) to 
extend OMS-ETG to this setting 
(Fig.~\ref{fig:policy-algorithm-etg-cs} in Appendix~\ref{sec:apdx-cost-structure}).
We first explore using budget $Be$.
The remaining budget is used in batches of size $S$.
After each round, we update the estimate $\widehat{k}$ of the oracle simplex,
and collect samples to get as close to $\widehat{k}$ as possible.

\begin{proposition}\label{prop:regret-etg-cs}
Suppose that 
the conditions of Prop.~\ref{prop:asymp-normality} and
Lemma~\ref{lemma:oracle-simplex-estimation-almost-sure} hold.
If $e = o(1)$ and $Be \to \infty$ as $B \to \infty$,
then $R_\infty(\pi_{\text{ETG-CS}}) = 0$.
\end{proposition}

\section{Asymptotic Confidence Sequences}
\label{sec:asymp-cs}
In Section~\ref{sec:setup-cons-norm}, we showed how to 
construct asymptotically valid confidence intervals (CIs)
for the target parameter $\beta^*$ (Prop.~\ref{prop:asymp-inference}).
One limitation of CIs is that they are only valid at a pre-specified horizon $T$.
In this section, we describe how to construct 
asymptotic confidence sequences (Theorem~\ref{thm:asympotic-conf-sequence}) that are valid at all time steps \citep{waudby2021time}.
In recent work, \citet{dalal2024anytime} have also developed
AsympCS for the double/debiased machine learning framework 
(with i.i.d. data).
The proofs are deferred to Appendix~\ref{sec:apdx-time-uniform}.

\begin{definition}[Confidence sequence]
We say that $(\widehat{\mu}_t \pm \Bar{\CM}^{*}_t)_{t=1}^{\infty}$ is a two-sided
(non-asymptotic) $(1-\alpha)$-confidence sequence for parameter $\mu$ if
\begin{align*}
    \PB\left( \forall t \in \NB : \mu \in ( \widehat{\mu}_t \pm \Bar{\CM}^{*}_t) \right) \geq 1 - \alpha. 
\end{align*}
\end{definition}
Confidence sequences (CSs) are a time-uniform counterpart of confidence intervals:
they allow for inference at all time steps simultaneously
(and stopping times).
Thus, the horizon $T$ need not be decided in advance
and the practitioner can adaptively stop or continue 
their experiment
(see \citet{ramdas2023game, howard2021time} for further discussion).
\citet{waudby2021time} define asymptotic confidence sequences (AsympCS),
a time-uniform counterpart of CLT-style asymptotic confidence intervals.

\begin{definition}[AsympCS {\citep[Definition~2.1]{waudby2021time}}]
Let $\mathbb{T}$ be a totally ordered infinite set (denoting time)
that has a minimum value $t_0 \in \mathbb{T}$.
We say that the intervals $(\widehat{\theta}_t - L_t, \widehat{\theta}_t + U_t)_{t \in \mathbb{T}}$
with non-zero bounds $L_t, U_t > 0$
form a $(1 - \alpha)$-asymptotic confidence sequence (AsympCS)
if there exists a (typically unknown) nonasymptotic CS
$(\widehat{\theta}_t - L^*_t, \widehat{\theta}_t + U^*_t)_{t \in \mathbf{T}}$ satisfying
\begin{align*}
     \PB\left( \forall t \in \mathbb{T} : \theta_t \in (\widehat{\theta}_t - L^*_t, \widehat{\theta}_t + U^*_t) \right) \geq 1 - \alpha,
\end{align*}
such that $L_t, U_t$ 
become arbitrarily precise almost-sure approximations to
$L^*_t, U^*_t$:
\begin{align*}
    L^*_t / L_t \ConvAS 1 \,\, \text{and} \,\, U^*_t / U_t \ConvAS 1.
\end{align*}
\end{definition}

\begin{assumption}[Nuisance estimation]\label{assum:nuisance-rates-asymp-cs}
For all $i \in [M]$,
(a) (Neyman orthogonality) $\forall \eta, \,\, \partial_{r} \EB[\psi^{(i)}_t(\theta^*, \eta^* + r ( \eta - \eta^*))]|_{r=0} = 0$;
(b) For $R_t := \int_0^1 \partial_{r^2} \EB_{t-1}[ \psi^{(i)}_t(\theta^*, \eta^* + r(\hEta - \eta^*)) ] \, dr$, $R_t = \litOas(\sqrt{\log t / t})$;
(c) $\sup_t \EB[R^{1 + \delta}_t] < \infty$ for some $\delta > 0$; and
(d) For some $\gamma > 0$, $\| \hEta - \eta^* \| = \litOas\left( t^{-\gamma}  \right)$.
\end{assumption}

\begin{theorem}\label{thm:asympotic-conf-sequence}
Suppose that (i) the Conditions of Lemma~\ref{lemma:oracle-simplex-estimation-almost-sure} hold;
(ii) Assumption~\ref{assum:nuisance-rates-asymp-cs} holds;
and
(iii) $\kappa_T \ConvAS \Tilde{\kappa}$ for some constant $\Tilde{\kappa} \in \ChoiceSimplex$.
Then, for any $\rho \in \RB_{> 0}$ and $\alpha \in (0, 1)$,
\begin{align*}
    \PB\left(\forall t \in \NB : \left| \widehat{\beta}_t - \beta^* \right| \leq
        \underbrace{\sqrt{ \frac{t \widehat{V}_t \rho^2 + 1}{t^2 \rho^2} \log\left( \frac{t \widehat{V}_t \rho^2 + 1}{\alpha^2} \right) }}_{:= \Bar{\CM}_t} + \, \litOas\left( \sqrt{ \frac{\log{t}}{t} } \right) \right) \geq 1 - \alpha,
\end{align*}
where $\widehat{V}_t$ is the estimated variance defined in Eq.~\ref{eq:variance-estimator},
and $\rho$ determines the time step at which $\Bar{\CM}_t$ is tightest \citep[Sec.~B.2]{waudby2021time}.
\end{theorem}
Assumptions~\ref{assum:nuisance-rates-asymp-cs}(b, d) require strong consistency rates
for the nuisance estimators (see Remark~\ref{remark:strong-consistency-nuisance} for more discussion).
Condition~(iii) requires that the selection simplex 
converges almost surely to a constant.
This is trivially satisfied for a fixed policy and 
under the conditions of Theorem~\ref{thm:eps-greedy},
this holds for the $\epsilon$-greedy policy 
with $\kappa_T \ConvAS \kappa^*$.

\section{Experiments}
\label{sec:experiments}
\subsection{Synthetic data} \label{sec:expr-synthetic}

In this section, we evaluate our data collection strategies
on the two-sample IV LATE estimation task 
(Example~\ref{example:two-sample-late})
with synthetic data generated from a nonlinear causal model
(Fig.~\ref{fig:rff-iv-results}).
The nonlinearities are generated using random Fourier features \citep{rahimi2007random}
that approximate a function drawn from a Gaussian Process
with a squared exponential kernel
(see Appendix~\ref{sec:apdx-expt-nonlinear-iv} for more details).
For nuisance estimation, we use a multilayer perception (MLP) \citep[Sec.~6]{goodfellow2016deep} with
two hidden layers and $64$ neurons each,
trained with early stopping \citep[Sec.~7.8]{goodfellow2016deep}.
We compare the MSE of the various policies using relative regret:
\begin{align*}
    \text{Relative regret} = \text{RR}(\pi) := \frac{\text{MSE}^{(\pi)} - \text{MSE}^{(\text{oracle}; \, \eta^*)}}{\text{MSE}^{(\text{oracle}; \, \eta^*)}} \times 100\%,
\end{align*}
where $\text{MSE}^{(\text{oracle}; \, \eta^*)}$ is 
the MSE of the oracle policy that uses the
true nuisances $\eta^*$.

\begin{figure}
\centering
\begin{subfigure}[b]{1\textwidth}
\centering
\includegraphics[scale=0.35]{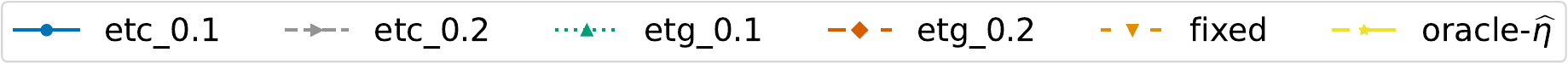}
\vspace{10px}
\end{subfigure}
\begin{subfigure}[b]{0.31\textwidth}
\centering
\includegraphics[scale=0.31]{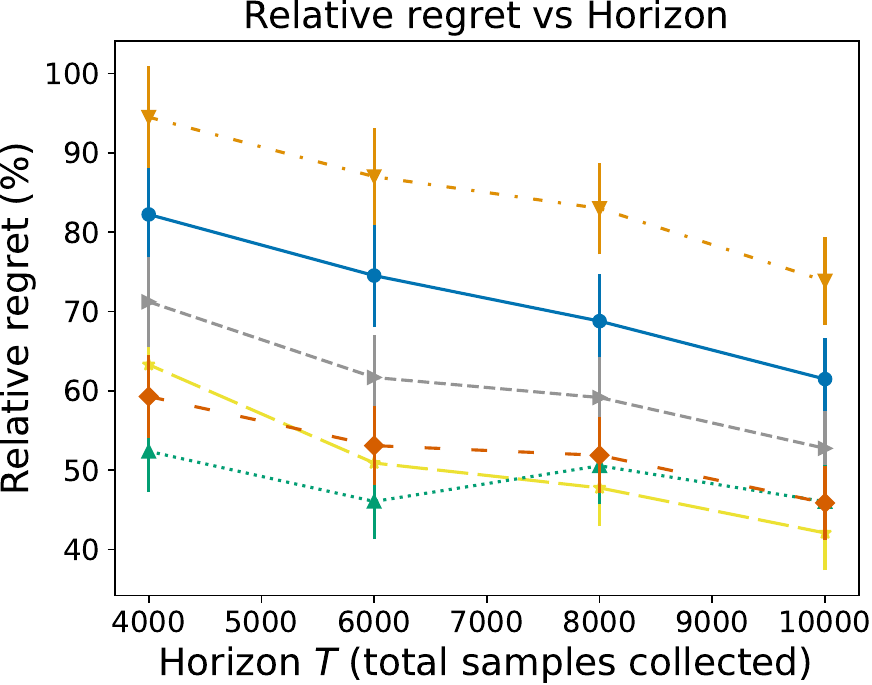}
\caption{Relative regret}
\label{fig:rff-iv-regret}
\end{subfigure}
\begin{subfigure}[b]{0.31\textwidth}
\centering
\includegraphics[scale=0.31]{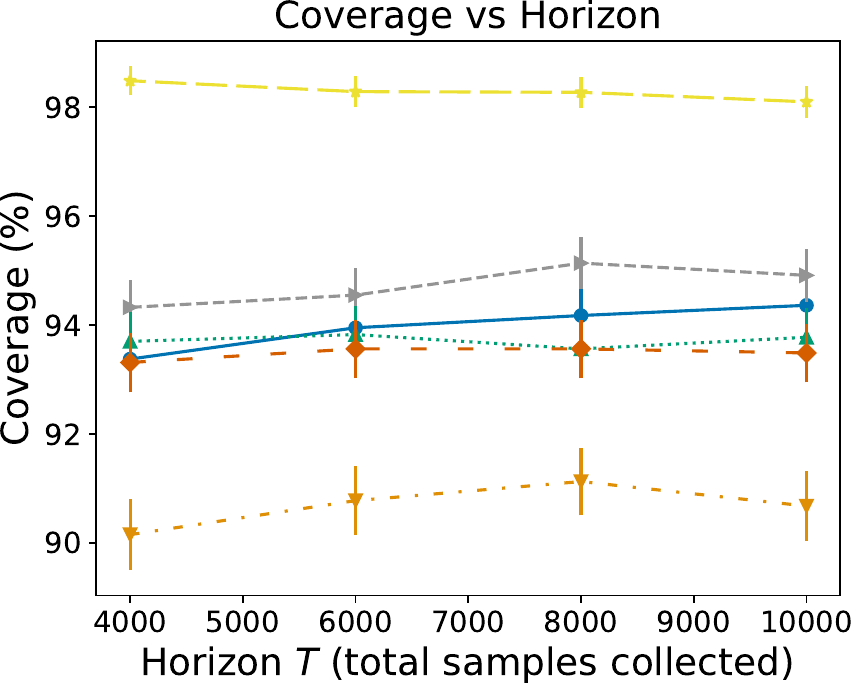}
\caption{Coverage}
\label{fig:rff-iv-coverage}
\end{subfigure}
\begin{subfigure}[b]{0.31\textwidth}
\centering
\includegraphics[scale=0.31]{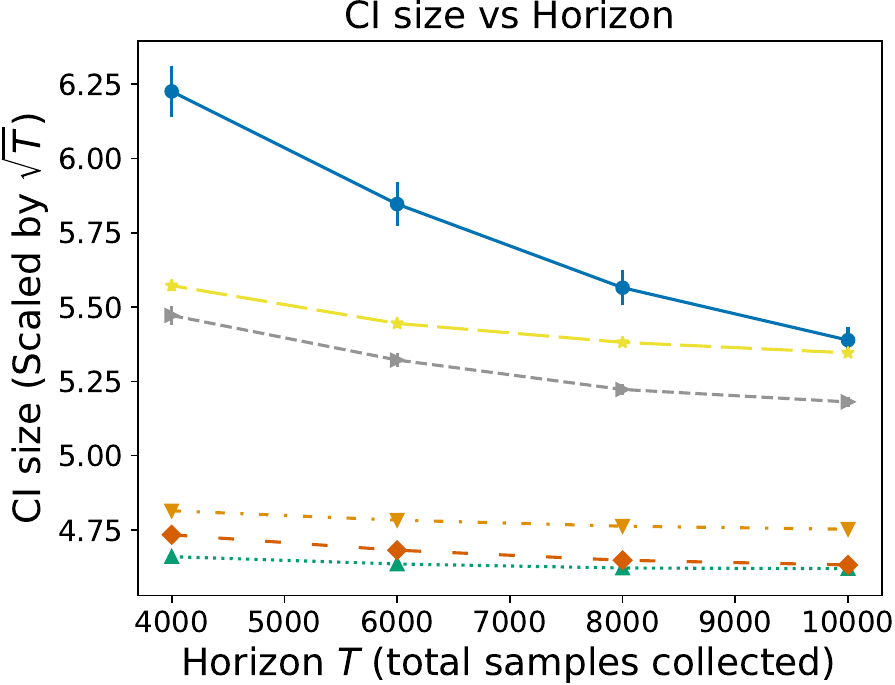}
\caption{CI size}
\label{fig:rff-iv-ci-size}
\end{subfigure}
\caption{Results for the two-sample IV LATE estimation task (Example~\ref{example:two-sample-late})
for a nonlinear causal model where MLPs are used for nuisance estimation
(error bars denote $95\%$ CIs).
In this case, significant bias is incurred due to nuisance estimation even
at large horizons.
The online data collection policies
outperform the fixed policy in terms of regret and coverage.
}
\label{fig:rff-iv-results}
\end{figure}

The labels $\text{etc}\_{x}$ and $\text{etg}\_{x}$
in the plots (Fig.~\ref{fig:rff-iv-results}) 
refer to the ETC and ETG policies with
exploration $e = x$.
For the ETG policies, 
we set the batch size $S$ such that
$S/(T(1 - e)) = 0.1$. 
Thus, the estimate of the oracle simplex $\widehat{k}_t$
is updated every $10\%$ of the horizon $T$ after exploration.
The label \emph{fixed} refers to a fixed policy that queries each
data source equally with $\kappa_T = [0.5, 0.5]^\top$.
We use \emph{oracle-}$\widehat{\eta}$ to denote
the oracle policy that uses estimated nuisance functions
(instead of the true nuisances $\eta^*$).
For computational reasons, we retrain the 
nuisance estimators in batches (rather than after every time step).

In terms of relative regret (Fig.~\ref{fig:rff-iv-regret}), we observe that even \emph{oracle-}$\widehat{\eta}$ has $\approx 50\%$ relative regret,
showing substantial bias due to nuisance estimation.
The online policies (ETC and ETG)
outperform the fixed policy,
with ETG having similar relative regret as \emph{oracle-}$\widehat{\eta}$.
Moreover, \emph{etg\_}$0.1$ outperforms \emph{etc\_}$0.1$,
demonstrating the advantage of repeatedly updating
the estimate of the oracle simplex.
In terms of coverage (Fig.~\ref{fig:rff-iv-coverage}), 
the online policies outperform the fixed policy,
and get close to the nominal coverage of $95\%$ at large horizons.
Finally, we observe that the size of the confidence intervals (Fig.~\ref{fig:rff-iv-ci-size})
for the \emph{etg} policies is significantly smaller than the 
confidence intervals (CIs) of \emph{etc} and
\emph{oracle-}$\widehat{\eta}$.

We also perform experiments for Examples~\ref{example:two-sample-late}, \ref{example:confounder-mediator}, and \ref{example:combine-observational-datasets}
with synthetic data generated from linear causal models 
(see Appendix~\ref{sec:apdx-expt-linear-models}),
observing that the ETC and ETG policies have lower relative regret
than the fixed policy in all cases.
Since we use correctly specified parametric nuisance estimators,
the bias due to nuisance estimation is 
substantially lower in this case, with the
ETG policy achieving nearly zero relative regret at large horizons.

\subsection{Real-world data} \label{sec:expr-real-world}

\paragraph{JTPA dataset.}
The National Job Training Partnership Act (JTPA) study examines the effect
of a job training program on future earnings
and has been analyzed in several prior works \citep{glynn2018front, donald2014testing, abadie2002instrumental, bloom1997benefits}.
The data follows the IV causal model (see Example~\ref{example:two-sample-late} and Fig.~\ref{fig:disjoint-iv-graph})
with a binary IV $Z$, binary treatment $X$, a vector of covariates $W$, 
and a real-valued outcome $Y$.
We use the dataset from \citet{glynnData2019},
which contains $N = 4384$ samples.
In this dataset, $Z$ indicates whether the individual was randomly offered training,
$X$ denotes program participation, $Y$ denotes future earnings,
and $W$ contains six covariates which
include variables like race and age.
We simulate the two-sample LATE estimation setting in Example~\ref{example:two-sample-late}
with $\DM = \left( \widehat{\PB}_{N}(W, Z, Y), \widehat{\PB}_{N}(W, Z, X) \right)$,
where $\widehat{\PB}_{N}$ denotes the empirical distribution over the $N$ samples.
The true LATE and oracle simplex
are computed on the full dataset 
with MLP-based nuisance estimators and
cross-fitting with $k = 2$ folds  \citep{chernozhukov2018double}, 
averaged over $2000$ runs.
We compare the scaled MSE ($T \times \text{MSE}$) of our proposed policies with a fixed policy that
queries both data sources uniformly with $\kappa_T = [0.5, 0.5]^\top$ 
(Fig.~\ref{fig:jtpa-data-mse}).
We observe that the MSE of the fixed policy is nearly twice as
large as the oracle policy.
By contrast, we see that the ETC and ETG policies match the
oracle policy at all horizons.

\begin{figure}
\centering
\begin{subfigure}[b]{1\textwidth}
\centering
\includegraphics[scale=0.35]{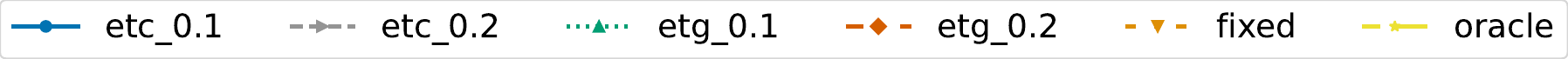}
\vspace{10px}
\end{subfigure}
\begin{subfigure}[b]{0.40\textwidth}
\centering
\includegraphics[scale=0.35]{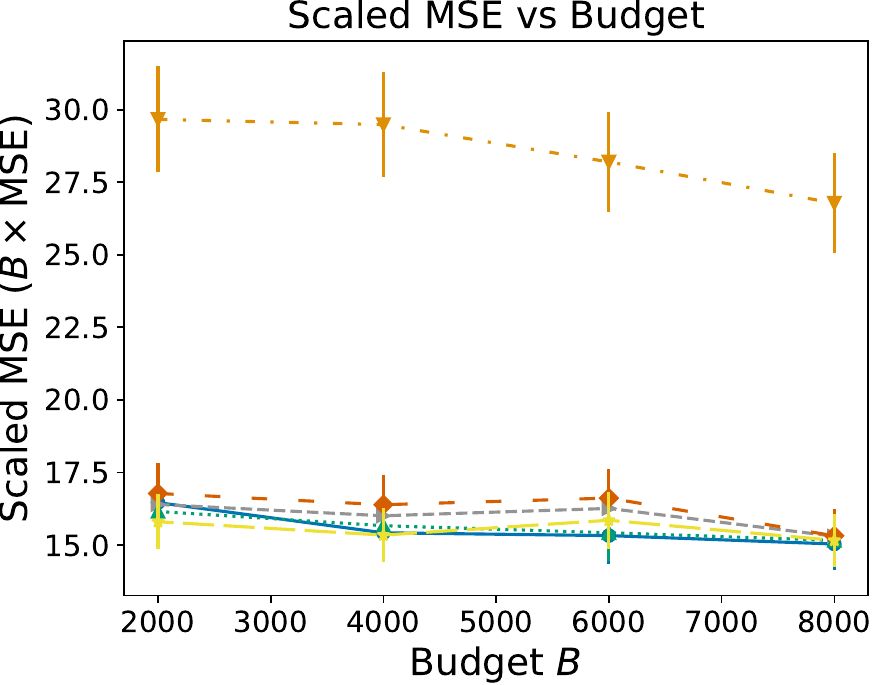}
\caption{JTPA data}
\label{fig:jtpa-data-mse}
\end{subfigure}
\hspace{10px}
\begin{subfigure}[b]{0.40\textwidth}
\centering
\includegraphics[scale=0.35]{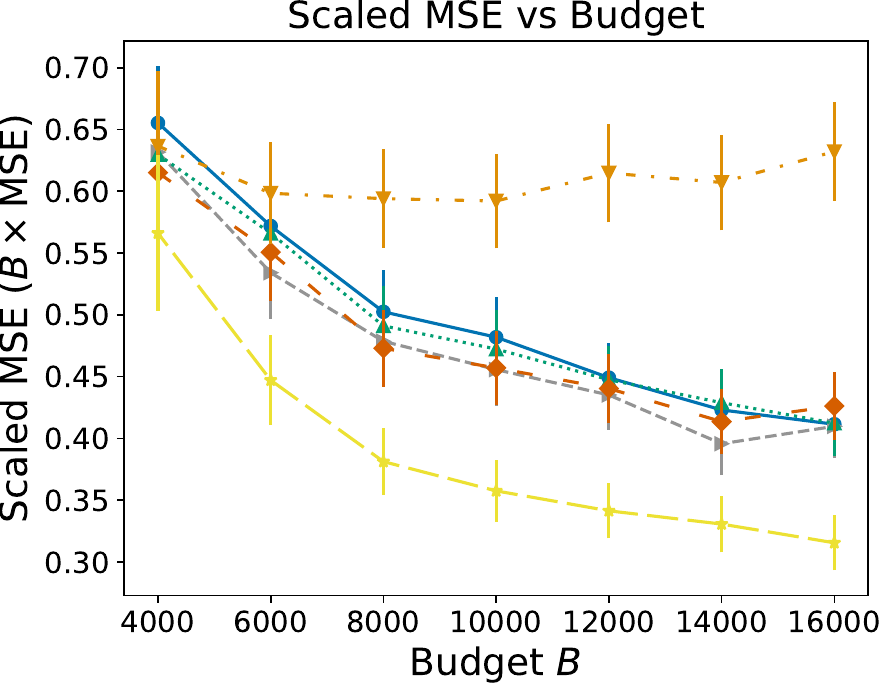}
\caption{COPD data}
\label{fig:copd-data-mse}
\end{subfigure}
\caption{Results on two real-world causal effect estimation tasks 
(error bars denote $95\%$ CIs).
In both cases, we observe that the online data collection policies
outperform the fixed policy (as the budget increases).}
\label{fig:real-dataset-results}
\end{figure}

\paragraph{COPD dataset.}
We test our strategies on datasets used to evaluate the effect
of chronic obstructive pulmonary disease (COPD) on the development
of herpes zoster (HZ) \citep{lin2014adjustment}.
We use the dataset released by \citet{yang2019combining},
which follows the setting of Example~\ref{example:combine-observational-datasets} 
(where an unconfounded dataset is combined with a confounded dataset).
The treatment $X \in \{0, 1\}$ and outcome $Y \in \{0, 1\}$ denote the presence of COPD and HZ, respectively.
We have two datasets:
(i) The \emph{validation} dataset with $N_{\text{val}} = 1148$ samples 
of $(U, W, X, Y)$ and
(ii) The \emph{main} dataset with $N_{\text{main}} = 42430$ samples 
of $(W, X, Y)$.
The covariates $W$ include variables like 
age, presence of liver disease, etc.
and $U$ includes variables like 
cigarette and 
alcohol consumption which are known to be important 
confounders 
(but are missing in the \emph{main} dataset).

To test our methods, we use $\DM = ( \widehat{\PB}_{\text{val}}(U, W, X, Y),\widehat{\PB}_{\text{main}}(W, X, Y))$,
where $\widehat{\PB}_{\text{val}}$ and $\widehat{\PB}_{\text{main}}$
denote the empirical distributions 
over the \emph{validation} and \emph{main} datasets.
We apply a cost structure 
of $c = [4, 1]$, i.e., a query to the \emph{validation} dataset 
costs four times that of the \emph{main} dataset.
The dataset contains the
propensity scores (not the covariates). 
Since propensity scores suffice for removing
confounding bias \citep[Sec.~15.2]{hernan2010causal},
we use them as covariates
with linear nuisance estimators.
The true ATE and oracle simplex 
are computed on the full dataset using cross-fitting with $k = 2$ folds.
We compare the regret of our proposed adaptive policies
to that of a fixed policy that only queries the first data
source with $\kappa_T = [1, 0]^\top$ (Fig.~\ref{fig:copd-data-mse}).
We observe that both the fixed and online policies have
higher MSE than the oracle at all budgets.
At smaller budget sizes, we do not observe a significant
difference with the fixed policy.
However, as the budget increases,
we observe that the online policies substantially outperform
the fixed policy and are much closer to the oracle.

\section{Conclusion}
\label{sec:conc}
Our work presents a semiparametric framework 
for statistical inference
that accounts for sequential data collection decisions.
Aiming to efficiently estimate a given target parameter,
we present two online data collection policies, OMS-ETC and OMS-ETG,
and prove that both achieve zero asymptotic 
regret relative to the oracle policy.
One limitation of our theoretical results is that 
they do not distinguish between OMS-ETC and OMS-ETG.
Overcoming this limitation would potentially
require higher-order or finite-sample analysis.
Another limitation is that our results hold
\emph{pointwise} rather than uniformly in
the probability space,
which could lead to different optimal rates
(e.g., in multi-armed bandits, 
better minimax optimal rates are achieved
when considering pointwise
or instance-dependent bounds \citep[Ch.~16]{lattimore2020bandit}).
In future work, we aim to extend our work to
the \emph{contextual} selection setting
where the data source can be selected based on
variables revealed in the same time step.
We also aim to extend our analysis
beyond scalar parameters to
targets that are multivariate or nonparametric
as well as parameters identified by
conditional moment restrictions. 




\section*{Acknowledgments}
We gratefully acknowledge the NSF (FAI 2040929 and IIS2211955), UPMC, Highmark Health, Abridge, Ford Research, Mozilla, the PwC Center, Amazon AI, JP Morgan Chase, the Block Center, the Center for Machine Learning and Health, and the CMU Software Engineering Institute (SEI) via Department of Defense contract FA8702-15-D-0002, for their generous support of ACMI Lab’s research.
We thank Ian Waudby-Smith and seminar participants
at Indiana University for helpful comments.

\bibliographystyle{abbrvnat}

\bibliography{refs}

\clearpage

\appendix
\section{Proof of consistency}
\label{sec:apdx-proof-consistency}
\begin{proposition}[MDS SLLN {\citep[Theorem~1]{csorgHo1968strong}}]\label{prop:apdx-mds-slln-general}
Let $\{ X_n, n \in \mathbb{N} \}$ be a martingale difference sequence (MDS) 
and $b_1 < b_2 < \hdots \rightarrow \infty$ be a non-decreasing sequence 
such that $\sum_{i=1}^{\infty} b^{-2}_i \EB[X^2_i] < \infty$.
Then $b^{-1}_n \sum_{i=1}^n X_i \ConvAS 0$.
\end{proposition}

\begin{cor}\label{cor:apdx-mds-slln}
Let $\{ X_n, n \in \mathbb{N} \}$ be a MDS with
$\sup_{n} \EB[X^2_n] < \infty$. 
Then $n^{-1} \sum_{i=1}^n X_i \ConvAS 0$.
\end{cor}
\begin{proof}
Apply Prop.~\ref{prop:apdx-mds-slln-general} with $b_i= i$.
\end{proof}

\begin{proposition}[Convergence of moments {\citep[Page~338]{billingsley1995probability}}]\label{prop:apdx-conv-of-moments}
Let $r$ be a positive integer. If $X_n \ConvDist X$ and
$\sup_n \EB[|X_n|^{r + \epsilon}] < \infty$, where $\epsilon > 0$,
then $\EB[|X|^r] < \infty$, and $\EB[X_n^r] \to \EB[X^r]$.
\end{proposition}

\begin{lemma}\label{lemma:avg-eta-convprob}
Let $(X_t)_{t=1}^{\infty}$ be a sequence of non-negative random variables 
with (i) $X_t \ConvProb 0$ and 
(ii) $\exists \delta > 0$  such that $\sup_t \EB[ X_t^{1+\delta}] < \infty$. 
Then $T^{-1} \sum_{t=1}^T X_t \ConvProb 0$.
\end{lemma}
\begin{proof}
By Markov's inequality, for any $\epsilon > 0$,
\begin{align}
    \PB\left( \frac{1}{T} \sum_{t=1}^T X_t > \epsilon \right) &\leq \frac{\sum_{t=1}^T \EB[X_t]}{T \epsilon} \nonumber \\
        &\leq \frac{ o(T) }{T\epsilon} = o(1), \label{eq:lemma-convprob-o1}
\end{align}
where (\ref{eq:lemma-convprob-o1}) follows because $\EB[\| \hEta - \eta^* \|] = o(1)$ 
(by Conditions~(i, ii) and Prop.~\ref{prop:apdx-conv-of-moments}).
\end{proof}

\begin{lemma}[ULLN]\label{lemma:apdx-uniform-conv-prob-semiparametric}
Suppose that $a_t(\theta, \eta) := a(X_t; \theta, \eta)$ satisfies Property~\ref{property:ulln}.
Let $a_{*}(\theta, \eta) = \EB[a(X_t; \theta, \eta)]$ 
and $s_t \in \{ 0, 1 \}$ be a $H_{t-1}$-measurable 
binary random variable. 
Then the following uniform law of large numbers (ULLN) holds:
\begin{align*}
    \sup_{\theta \in \Theta} \left| \frac{1}{T} \sum_{t=1}^{T} s_t
    \left\{ a_t(\theta, \hEta)
     - a_*(\theta, \eta^*)
    \right\} \right| \ConvProb 0.
\end{align*}
\end{lemma}
\begin{proof}
We prove the ULLN by a 
covering number argument
\footnote{We use a covering number argument for convenience
but it may be possible to weaken the conditions of Property~\ref{property:ulln} by using the more general martingale ULLN in \citet{rakhlin2015sequential}. }
\citep[Lemma~1]{tauchen1985diagnostic}.
Let $\left\{ \theta_k \right\}_{k=1}^{K}$ be a minimal $\delta$-cover of $\Theta$ and
$\BM_{\delta}(\theta_k)$ denotes the $\delta$-ball around $\theta_k$. 
By compactness of $\Theta$ (Assumption~\ref{assump:standard-gmm}b), $K$ is finite.
For the rest of the proof,
for any $\theta \in \Theta$, 
let $\theta_k$ denote the element of the $\delta$-cover
such that $\theta \in \BM_{\delta}(\theta_k)$.

By the triangle inequality,
\begin{align}
    \sup_{\theta \in \Theta} & \left| \frac{1}{T} \sum_{t=1}^{T}
    s_t \left\{ a_t(\theta, \hEta)
     - a_*(\theta, \eta^*) \right\} \right| \nonumber \\
        &\leq \begin{aligned}[t]
            &  \sup_{\theta} \left| \frac{1}{T} \sum_{t=1}^{T} s_t \left\{ a_t(\theta, \hEta) - a_t(\theta_k, \hEta) \right\} \right| +
            \max_{k \in [K]} \left| \frac{1}{T} \sum_{t=1}^{T} s_t \left\{ a_t(\theta_k, \hEta) - a_{*}(\theta_k, \eta^*) \right\}  \right| + \\
            & \max_{k \in [K]} \frac{1}{T} \sum_{t=1}^T s_t \left| a_{*}(\theta_k, \eta^*) - a_{*}(\theta, \eta^*) \right|
        \end{aligned} \nonumber \\
        &\leq \begin{aligned}[t]
            & \underbrace{\frac{1}{T} \sum_{t=1}^{T} \sup_{\theta} \left| a_t(\theta, \hEta) - a_t(\theta_k, \hEta) \right|}_{\text{T1}} +
            \underbrace{\max_{k \in [K]} \left| \frac{1}{T} \sum_{t=1}^{T} s_t \left\{ a_t(\theta_k, \hEta) - a_{*}(\theta_k, \eta^*) \right\}  \right|}_{\text{T2}} + \\
            & \underbrace{\max_{k \in [K]} \left| a_{*}(\theta_k, \eta^*) - a_{*}(\theta, \eta^*) \right|}_{\text{T3}},
        \end{aligned} \label{eq:apdx-sup-decomp-ineq}
\end{align}
where (\ref{eq:apdx-sup-decomp-ineq}) follows because $s_t \in \{0, 1\}$.
Next, we show that all three terms are $\litOp(1)$. 

\paragraph{Term (T1).}
\begin{align}
    \frac{1}{T} \sum_{t=1}^{T} \sup_{\theta} \left| a_t(\theta, \hEta) - a_t(\theta_k, \hEta) \right|
        &=  \frac{1}{T} \sum_{t=1}^{T} \delta \left| L(X_t) \right| \label{eq:apdx-term-t1-a} \\
        &= \delta \EB[L(X_t)] + \litOas(1) \label{eq:apdx-term-t1-c} \\
        &\leq \sqrt{L} \delta + \litOas(1), \label{eq:apdx-term-t1-b}
\end{align}
where (\ref{eq:apdx-term-t1-a}) and (\ref{eq:apdx-term-t1-b})
follow by Property~\ref{property:ulln}(ii) (Lipschitzness) 
and (\ref{eq:apdx-term-t1-c}) follows by Corollary~\ref{cor:apdx-mds-slln}.
Therefore, for any $\epsilon > 0$ and a small enough
$\delta$, the term \emph{T1} will be less than $\epsilon$ eventually.

\paragraph{Term (T2).} 
For any $k \in [K]$, by the triangle inequality,
\begin{align*}
    & \left| \frac{1}{T} \sum_t s_t \left[ a_t(\theta_k, \hEta) - a_{*}(\theta_k, \eta^*) \right] \right| \\
    &\leq \begin{aligned}[t]
        & \underbrace{\left| \frac{1}{T} \sum_t s_t \left[ a_t(\theta_k, \eta^*) - a_{*}(\theta_k, \eta^*) \right] \right|}_{:= a} + \\
        & \underbrace{\left| \frac{1}{T} \sum_t s_t \left[ a_t(\theta_k, \hEta) - a_t(\theta_k, \eta^*) - \EB_{t-1}[a_t(\theta_k, \hEta) - a_t(\theta_k, \eta^*) ]  \right] \right|}_{:= b} + \\
        & \underbrace{\left| \frac{1}{T} \sum_t s_t \EB_{t-1}[a_t(\theta_k, \hEta) - a_t(\theta_k, \eta^*)] \right|}_{:= c}.
    \end{aligned}
\end{align*}
The term \emph{a} above is a martingale difference sequence (MDS). 
By Property~\ref{property:ulln}(i) and Corollary~\ref{cor:apdx-mds-slln}, this term is $\litOas(1)$.
The term \emph{b} above is also a MDS. 
For $D_t = a_t(\theta_k, \hEta) - a_t(\theta_k, \eta^*)$,
\begin{align}
    \EB[s_t ( D_t - \EB_{t-1}[D_t] )^2] &\leq \EB[( D_t - \EB_{t-1}[D_t] )^2] \label{eq:apdx-ep-cons-a} \\
    &\leq \EB[D_t^2] \nonumber \\
        &\leq \EB[ \EB_{t-1}[D_t^2] ] \nonumber \\
        &\leq L \EB[ \| \hEta - \eta^* \|^2] \label{eq:apdx-ep-cons-b} \\
        &< \infty, \label{eq:apdx-ep-cons-c}
\end{align}
where (\ref{eq:apdx-ep-cons-a}) follows because $s_t \in \{0, 1\}$,
(\ref{eq:apdx-ep-cons-b}) by Property~\ref{property:ulln}(ii),
and (\ref{eq:apdx-ep-cons-c}) by Property~\ref{property:ulln}(iv).
Thus, by Corollary~\ref{cor:apdx-mds-slln}, this term is also $\litOas(1)$.
For term \emph{c}, we have
\begin{align}
    \left| \frac{1}{T} \sum_t s_t \EB_{t-1}[a_t(\theta_k, \hEta) - a_t(\theta_k, \eta^*)] \right|
        &\leq \frac{1}{T} \sum_t \EB_{t-1} [| a_t(\theta_k, \hEta) - a_t(\theta_k, \eta^*) |] \label{eq:apdx-term-t2-a0} \\
        &\leq \frac{1}{T} \sum_t \sqrt{\EB_{t-1} [| a_t(\theta_k, \hEta) - a_t(\theta_k, \eta^*) |^2]} \nonumber \\
        &\leq \frac{\sqrt{L}}{T} \sum_t \| \hEta - \eta^* \| \label{eq:apdx-term-t2-3-a} \\
        &= o_p(1) \label{eq:apdx-term-t2-3-b},
\end{align}
where (\ref{eq:apdx-term-t2-a0}) follows because $s_t \in \{0, 1\}$,
(\ref{eq:apdx-term-t2-3-a}) by Property~\ref{property:ulln}(ii),
and (\ref{eq:apdx-term-t2-3-b}) by Lemma~\ref{lemma:avg-eta-convprob}.
This allows us to show that term \emph{T2} is $\litOp(1)$:
\begin{align*}
     \max_{k \in [K]} \left| \frac{1}{T} \sum_{t=1}^{T} s_t \left\{ a_t(\theta_k, \hEta) - a_{*}(\theta_k, \eta^*) \right\}  \right|
     &\leq \sum_{k \in [K]} \left| \frac{1}{T} \sum_{t=1}^{T} s_t \left\{ a_t(\theta_k, \hEta) - a_{*}(\theta_k, \eta^*) \right\}  \right| \\
     &= \litOp(1),
\end{align*}
where the last line follows because $K$ is finite.

\paragraph{Term (T3).} By Property~\ref{property:ulln}(ii),
\begin{align*}
    \max_{k \in [K]} \left| a_{*}(\theta_k, \eta^*) - a_{*}(\theta, \eta^*) \right| \leq \EB[L(X_t)] \delta
    \leq \sqrt{L} \delta.
\end{align*}
Thus, for any $\epsilon > 0$ and a small enough
$\delta$, the term \emph{T3} will be less than $\epsilon$.
\end{proof}

\begin{proposition}\label{prop:apdx-ulln-product-fn}
Suppose that 
(i) $f_t(\theta, \eta) := f(X_t; \theta, \eta)$ and $g_t(\theta, \eta)  := g(X_t; \theta, \eta)$ satisfy Property~\ref{property:ulln}; and
(ii) (Boundedness) $\forall (\theta, \eta), |f_t(\theta, \eta)|_{\infty} \leq H$ and 
$|g_t(\theta, \eta)|_{\infty} \leq H$
for some constant $H < \infty$.
Then, $f_t(\theta, \eta) g_t(\theta, \eta)$ satisfies Property~\ref{property:ulln}.
\end{proposition}
\begin{proof}
We show that the conditions of Property~\ref{property:ulln} hold
for $a_t(\theta, \eta) = f_t(\theta, \eta) g_t(\theta, \eta)$.

\paragraph{Property~\ref{property:ulln}(i).}
For any $\theta, \eta$,
\begin{align*}
    \EB\left[|a_t(\theta, \eta)|^{2 + \delta}\right] &= \EB\left[ |f_t(\theta, \eta) g_t(\theta, \eta)|^{2+\delta} \right] \\
        &\leq H^{4 + 2\delta} \\
        &< \infty.
\end{align*}

\paragraph{Property~\ref{property:ulln}(ii).}
We have that $\forall \theta, \theta' \in \Theta, \forall
\eta \in \TM$,
\begin{align}
    |a_t(\theta', \eta) - a_t(\theta, \eta) | &= |f_t(\theta', \eta) g_t(\theta', \eta) - f_t(\theta, \eta) g_t(\theta, \eta) | \nonumber \\
        &= |f_t(\theta', \eta) (g_t(\theta', \eta) - g_t(\theta, \eta) + g_t(\theta, \eta)) - f_t(\theta, \eta) g_t(\theta, \eta) | \nonumber \\
        &= |f_t(\theta', \eta) (g_t(\theta', \eta) - g_t(\theta, \eta)) + g_t(\theta, \eta) (f_t(\theta', \eta) - f_t(\theta, \eta)) | \nonumber \\
        &\leq L(X_t)( |f_t(\theta', \eta)| +| g_t(\theta, \eta)|) \| \theta - \theta' \|, \label{eq:apdx-product-ulln-a}
\end{align}
where (\ref{eq:apdx-product-ulln-a}) follows by Condition~(i). 
Furthermore, by Condition~(ii), we have 
\begin{align*}
    \EB[(L(X_t)( |f_t(\theta', \eta)| + |g_t(\theta, \eta)|))^2] &\leq 4 H^2 \EB[L(X_t)^2] \\
    &< \infty.
\end{align*}
Next, by Conditions~(i, ii), we have that $\forall \theta \in \Theta, \forall
\eta, \eta' \in \TM$,
\begin{align*}
    \EB[(a_t(\theta, \eta) & - a_t(\theta, \eta'))^2] = \EB[(f_t(\theta, \eta) g_t(\theta, \eta) - f_t(\theta, \eta') g_t(\theta, \eta'))^2] \\
        &= \EB[ (f_t(\theta, \eta) (g_t(\theta, \eta) - g_t(\theta, \eta')) + g_t(\theta, \eta') (f_t(\theta, \eta) - f_t(\theta, \eta')))^2 ] \\
        &= \begin{aligned}[t]
            & \EB[f^2_t(\theta, \eta) (g_t(\theta, \eta) - g_t(\theta, \eta'))^2 ] + \EB[g^2_t(\theta, \eta') (f_t(\theta, \eta) - f_t(\theta, \eta'))^2 ] \\
            &+ 2 \EB[f_t(\theta, \eta) g_t(\theta, \eta') (g_t(\theta, \eta) - g_t(\theta, \eta')) (f_t(\theta, \eta) - f_t(\theta, \eta')]
        \end{aligned} \\
        &\leq \begin{aligned}[t]
            2 H^2 L \| \eta - \eta'\|^2
            + 2 H^2 \EB[|(g_t(\theta, \eta) - g_t(\theta, \eta')) (f_t(\theta, \eta) - f_t(\theta, \eta')|]
        \end{aligned} \\
        &\leq \begin{aligned}[t]
            2 H^2 L \| \eta - \eta'\|^2
            + 2 H^2 \sqrt{\EB[(g_t(\theta, \eta) - g_t(\theta, \eta'))^2] \EB[(f_t(\theta, \eta) - f_t(\theta, \eta'))^2]}
        \end{aligned} \\
        &\leq \begin{aligned}[t]
            4 H^2 L \| \eta - \eta'\|^2
        \end{aligned},
\end{align*}
completing the proof.
\end{proof}

Before presenting the proof of Prop.~\ref{prop:consistency-gmm},
we define some additional notation.
Recall that $\widehat{\theta}_T = \arg\min_{\theta \in \Theta} \widehat{Q}_T(\theta, (\hEta)_{t=1}^{T})$ 
(see Eq.~\ref{eq:gmm-estimator-argmin}) and
the moment conditions have the following form (see Eq.~\ref{eq:augmented-moment-conditions}):
\begin{align*}
    g_t(\theta, \eta) = m(s_t) \odot \Tilde{g}_t(\theta, \eta),
\end{align*}
where $m : \{0, 1\}^{\CardC} \mapsto \{0, 1\}^M$ is a fixed known
function that determines which moments get selected based
on the selection vector $s_t$.
For some $\kappa \in \ChoiceSimplex$ and
$s_t \overset{\text{i.i.d.}}{\sim} \text{Multinouilli}(\kappa)$,
we define the following matrices:
\begin{align}
    m_G(\kappa) &:= \EB[\underbrace{[m(s_t), \hdots, m(s_t) ]}_{D \, \text{times}}] \in [0, 1]^{M \times D}, \label{eq:apdx-m-g-grad} \\
    m_\Omega(\kappa) &:= \EB \left[ m(s_t) m(s_t)^\top \right] \in [0, 1]^{M \times M}. \label{eq:apdx-m-omega}
\end{align}

\newtheorem*{prop:consistency-gmm}{Proposition~\ref{prop:consistency-gmm}}
\begin{prop:consistency-gmm}[Consistency]
Suppose that (i) Assumption~\ref{assump:standard-gmm} holds;
(ii) $\forall i \in [M], \psi^{(i)}$ satisfies Property~\ref{property:ulln};
(iii) $\forall (i, j) \in [M]^2, \psi^{(i)} \psi^{(j)}$ satisfies Property~\ref{property:ulln}.
Then $\widehat{\beta}_T \ConvProb \beta^*$.
\end{prop:consistency-gmm}
\begin{proof}
We define the following matrices:
\begin{align*}
    \widehat{\Omega}_T(\theta) &:= \frac{1}{T} \sum_t g_t(\theta, \hEta) g^\top_t(\theta, \hEta) \\
    \Omega_T(\theta) &:= \frac{1}{T} \sum_t \EB_{t-1} \left[ g_t(\theta, \eta^*) g^\top_t(\theta, \eta^*) \right] \\
        &= \left[ \frac{1}{T} \sum_{t} m(s_t) m(s_t)^\top \right] \odot \EB [ \Tilde{g}_t(\theta, \eta^*) \Tilde{g}_t^\top(\theta, \eta^*) ] \\
        &= m_{\Omega}(\kappa_T) \odot \EB [ \Tilde{g}_t(\theta, \eta^*) \Tilde{g}_t^\top(\theta, \eta^*) ].
\end{align*}
We define the empirical and population versions of the GMM objective:
\begin{align*}
    \widehat{Q}_T(\theta, (\hEta)_{t=1}^{T}) :&= \left[ \frac{1}{T} \sum_{t=1}^T 
    g_t(\theta, \hEta) \right]^\top \widehat{W}_T \left[ \frac{1}{T} \sum_{t=1}^T g_t(\theta, \hEta) \right], \\
    \Bar{Q}_T(\theta) :&= \left[ \frac{1}{T} \sum_{t=1}^T \EB_{t-1} \left[
    g_t(\theta, \eta^*) \right]^\top \right] W_T \left[ \frac{1}{T} \sum_{t=1}^T \EB_{t-1} \left[
    g_t(\theta, \eta^*) \right] \right] \\
        &= \left[ \left( \frac{1}{T} \sum_{t=1}^T m(s_t) \right) \odot g_*(\theta) \right]^\top W_T \left[ \left( \frac{1}{T} \sum_{t=1}^T m(s_t) \right) \odot g_*(\theta) \right]
\end{align*}
where $g_*(\theta) := \EB\left[ \Tilde{g}_t(\theta, \eta^*) \right]$.
For the one-step GMM estimator, denoted by $\widehat{\theta}_T^{(\text{os})}$, 
we use $\widehat{W}_T = W_T = I \, (\text{identity})$.
For the two-step GMM estimator, we use $\widehat{W}_T = \widehat{\Omega}_T(\widehat{\theta}_T^{(\text{os})})^{-1}$
and $W_T = \Omega_T(\theta^*)^{-1}$.

By Assumption~\ref{assump:standard-gmm}(a), the true parameter $\theta^*$ uniquely minimizes
the population objective: $\theta^* = \arg\min_{\theta \in \Theta} \Bar{Q}_T(\theta)$.
By the same argument as \citet[Thm.~2.1]{newey1994large}, 
if $\sup_{\theta \in \Theta} |\widehat{Q}_T(\theta, (\hEta)_{t=1}^{T}) - \Bar{Q}_T(\theta)| \ConvProb 0$,
then $\widehat{\theta}_T \ConvProb \theta^*$.
Thus, to show consistency, we prove the uniform convergence of $\widehat{Q}_T(\theta, (\hEta)_{t=1}^{T})$.
By the triangle inequality,  
\begin{align*}
    \left| \widehat{Q}_T(\theta, (\hEta)_{t=1}^{T}) - \Bar{Q}_T(\theta) \right| \leq
    F_T(\theta)^2 \| \widehat{W}_T \|^2 + 2 \left\|g_{*}(\theta) \right\| F_T(\theta) \| \widehat{W}_T \| +
    \left\| g_{*}(\theta) \right\|^2 \| \widehat{W}_T - W_T \|,
\end{align*}
where $F_T(\theta) := \| T^{-1} \sum_{t=1}^{T} \left[ g_t(\theta, \hEta) \\ - m(s_t) \odot g_{*}(\theta) \right] \|$.
We prove the uniform convergence of $F_T(\theta)$ by applying Lemma~\ref{lemma:apdx-uniform-conv-prob-semiparametric} to every element
of $g_t(\theta, \hEta)$ along with the union bound. For any $\epsilon > 0$,
\begin{align}
    \PB\left( \sup_{\theta \in \Theta} F_T(\theta) < \epsilon  \right) &\geq \PB\left( \sup_{\theta \in \Theta} \sum_{i=1}^{M} \left| \frac{1}{T} \sum_{t=1}^{T} \left[ g_{t, i}(\theta, \hEta) - m_i(s_t) \odot g_{*}(\theta)_{i}\right] \right| < \epsilon  \right) \nonumber \\
    &\geq 1 - \sum_{i=1}^{M} \PB\left( \sup_{\theta \in \Theta} \left| \frac{1}{T} \sum_{t=1}^{T} \left[ g_{t, i}(\theta, \hEta) - m_i(s_t) \odot g_{*}(\theta)_{i}\right] \right| > \frac{\epsilon}{M}  \right) \label{eq:apdx-consistency-proof-a} \\
    &\geq 1 - \litOp(1), \label{eq:apdx-consistency-proof-b}
\end{align}
where (\ref{eq:apdx-consistency-proof-a}) follows by the union bound,
and (\ref{eq:apdx-consistency-proof-b}) by Lemma~\ref{lemma:apdx-uniform-conv-prob-semiparametric}.
For the one-step GMM estimator, we have $\widehat{W}_T = W_T = I$ (identity) and so
$\|\widehat{W}_T - W_T \| \ConvProb 0$ holds trivially.
Therefore,
$\widehat{\theta}^{\text{(os)}}_T \ConvProb \theta^*$.
For the two-step estimator, we need to show that
$\|\widehat{W}_T - W_T \| \ConvProb 0$. For all $(i, j) \in [M]^2$ and any $\epsilon > 0$,
\begin{align}
    \PB\left( \sup_{\theta \in \Theta} \left| \widehat{\Omega}_T(\theta)_{i, j} - \Omega_T(\theta)_{i, j} \right| > \epsilon \right) &\to 0, \label{eq:apdx-consistency-proof-c} \\
    \therefore \, \PB\left( \left| \widehat{\Omega}_T(\widehat{\theta}^{\text{(os)}}_T)_{i, j} - \Omega_T(\widehat{\theta}^{\text{(os)}}_T)_{i, j} \right| > \epsilon \right) &\to 0, \nonumber \\
    \therefore \, \PB\left( \left| \widehat{\Omega}_T(\widehat{\theta}^{\text{(os)}}_T)_{i, j} - \Omega_T(\theta^*)_{i, j} \right| > \epsilon \right) &\to 0, \label{eq:apdx-consistency-proof-d} \\
    \therefore \, \PB\left( \left| (\widehat{W}_T)_{i, j} - (W_T)_{i, j} \right| > \epsilon \right) &\to 0, \label{eq:apdx-consistency-proof-e} \\
    \therefore \, \PB\left( \left\| \widehat{W}_T - W_T \right\| > \epsilon \right) &\to 0, \nonumber
\end{align}
where (\ref{eq:apdx-consistency-proof-c}) follows by Lemma~\ref{lemma:apdx-uniform-conv-prob-semiparametric} and Condition~(iii);
(\ref{eq:apdx-consistency-proof-d}) follows because 
$\widehat{\theta}^{\text{(os)}}_T \ConvProb \theta^*$;
and (\ref{eq:apdx-consistency-proof-e}) by continuity of 
the matrix inverse.
Therefore, $\widehat{\theta}_T \ConvProb \theta^*$.
Since $\widehat{\beta}_T = \fTar(\widehat{\theta}_T)$, 
by continuity of $\fTar$ (Assumption~\ref{assump:standard-gmm}c), 
we have $\widehat{\beta}_T \ConvProb \beta^*$.
\end{proof}

\section{Proof of asymptotic normality}
\label{sec:apdx-proof-asymp-norm}
\begin{proposition}[Martingale CLT {\citep[Thm.~6.23]{hausler2015stable}}]\label{prop:apdx-martingale-clt}
Let $\{ X_i, \FM_i, 1 \leq i \leq n \}$ be a martingale difference sequence
and let $\FM_{\infty} := \sigma\left( \bigcup_{i=1}^{\infty} \FM_i \right)$.
Let $\gamma^2$ be an a.s. finite random variable.
Suppose that (i) (Square integrable) $\forall i, \EB[X^2_i] < \infty$;
(ii) (Conditional Lindeberg) $\forall \epsilon > 0$, $n^{-1} \sum_{i=1}^n \EB\left[ X^2_i \mathbbm{1} \left( |X_i| > \epsilon \sqrt{n} \right) | \FM_{i-1} \right] \ConvProb 0$, and 
(ii) (Convergence of conditional variance) $n^{-1} \sum_{i=1}^n \EB\left[ X^2_i | \FM_{i-1} \right] \ConvProb \gamma^2$ for some random variable $\gamma$.
Let $Z \sim N(0, 1)$
be a standard Gaussian independent of $\FM_{\infty}$.
Then 
\begin{align*}
    \frac{1}{\sqrt{n}} \sum_{i=1}^n X_i \ConvDist \gamma Z \,\,\, \FM_{\infty}\text{-stably},
\end{align*}
where the random variable $\gamma Z$ is a mixture of centered normal distributions with the
characteristic function $\varphi(t) = \EB_{\gamma}[- \gamma^2 t^2 / 2]$.
If $\gamma$ is a.s. constant, then 
\begin{align*}
    \frac{1}{\sqrt{n}} \sum_{i=1}^n X_i \ConvDist \NM(0, \gamma^2).
\end{align*}
\end{proposition}

\begin{proposition}[Properties of stable convergence {\citep[Thm.~3.18]{hausler2015stable}}]\label{prop:apdx-prop-stable-convergence}
Assume that $X_n \ConvDist X \,\, \GM\text{-stably}$.
Let $Y_n$ and $Y$ be random variables.
(a) If $Y_n \ConvProb Y$ and $Y$ is $\GM$-measurable, 
then $(Y_n, X_n) \to (X, Y) \,\, \GM\text{-stably}$.
(b) If $g$ is a measurable and continuous function, 
then $g(X_n) \to g(X) \,\, \GM\text{-stably}$.
\end{proposition}

\begin{cor}[Cramer-Slutsky]\label{cor:apdx-cramer-slutsky}
Assume that $X_n \ConvDist X \,\, \GM\text{-stably}$.
Let $Y_n$ and $Y$ be random variables.
If $Y_n \ConvProb Y$ and $Y$ is $\GM$-measurable,
then $Y_n X_n \ConvDist Y X \,\, \GM\text{-stably}$.
\end{cor}
\begin{proof}
Apply Prop.~\ref{prop:apdx-prop-stable-convergence} with 
the function $g(X, Y) = X Y$.
\end{proof}

\begin{proposition}[Cramer-Wold device {\citep[Cor.~3.19]{hausler2015stable}}]\label{prop:cramer-wold}
If $u^\top Y_n \ConvDist u^\top Y \,\, \GM\text{-stably}$ for every $u \in \RB^d$, 
then $Y_n \ConvDist Y \,\, \GM\text{-stably}$.
\end{proposition}

\begin{lemma}\label{lemma:apdx-op-half-converge}
Let $(X_t)_{t=1}^{\infty}$ be a sequence of non-negative random variables
such that (i) $\sup_t \EB[X^{1+\delta}_t] < \infty$ and
(ii) $X_t = \litOp(t^{-1/2})$. Then
\begin{align*}
    \frac{1}{\sqrt{T}} \sum_{i=1}^{T} X_t = \litOp(1).
\end{align*}
\end{lemma}
\begin{proof}
By Markov's inequality, for any $\epsilon > 0$,
\begin{align}
    \PB \left( \frac{1}{\sqrt{T}} \sum_{i=1}^{T} X_t > \epsilon \right) &\leq \frac{\sum_t \EB[X_t]}{\sqrt{T} \epsilon} \nonumber \\
    &= \frac{o(\sqrt{T})}{\sqrt{T} \epsilon} \label{eq:apdx-op-half-markov-a} \\
    &= o(1), \nonumber
\end{align}
where (\ref{eq:apdx-op-half-markov-a}) follows because $\EB[X_t] = o(t^{-1/2})$
(by Conditions~(i, ii) and Prop.~\ref{prop:apdx-conv-of-moments}).
\end{proof}

\newtheorem*{prop:asymp-normality}{Proposition~\ref{prop:asymp-normality}}
\begin{prop:asymp-normality}[Asymptotic normality]
Suppose that (i) Conditions of Prop.~\ref{prop:consistency-gmm} hold;
(ii) Assumptions~\ref{assum:nuisance-rates} and \ref{assump:kappa-T-converges} hold; and
(iii) $\forall i \in [M], j \in [D], \partial_{\theta_j} 
\left[\psi^{(i)}_t(\theta, \eta) \right]$
satisfies Property~\ref{property:ulln}.
Then $\widehat{\beta}_T$ converges to a mixture of normals,
where the mixture is over $\kappa_{\infty}$:
\begin{align*}
    \sqrt{T} (\widehat{\beta}_T - \beta^*) &\ConvDist \NM_{\kappa_\infty}\left( 0, V_{*}(\kappa_\infty) \right),
\end{align*}
where $V_{*}(\kappa)$ is a constant
that depends on $(\theta^*, \eta^*, \kappa)$.
If $\kappa_{\infty}$ is almost surely constant, then
$\widehat{\beta}_T$ is asymptotically normal.
\end{prop:asymp-normality}
\begin{proof}
Recall that $\widehat{\theta}_T = \arg\min_{\theta \in \Theta} \widehat{Q}_T(\theta, (\hEta)_{t=1}^{T})$ 
(see Eq.~\ref{eq:gmm-estimator-argmin}).
For the two-step GMM estimator, by first-order optimality, we have
\begin{align}
    \sqrt{T} (\widehat{\theta}_T - \theta^*) = - \left[ \widehat{G}_{T}^\top(\widehat{\theta}_T) \widehat{\Omega}(\widehat{\theta}^{(\text{os})}_T)^{-1} \widehat{G}_{T}(\Tilde{\theta})  \right]^{-1} \widehat{G}_{T}^\top(\widehat{\theta}_T) \widehat{\Omega}(\widehat{\theta}^{(\text{os})}_T)^{-1} \frac{1}{\sqrt{T}} \sum_{t=1}^{T} g_t(\theta^*, \hEta), \label{eq:apdx-first-order-gmm-a}
\end{align}
where $\widehat{\theta}^{(\text{os})}_T$ is the one-step GMM estimator, $\Tilde{\theta}$ is a point on the line-segment joining $\widehat{\theta}_T$ and $\theta^*$,
\begin{align*}
    \widehat{G}_T(\theta) &= \frac{1}{T} \sum_{t=1}^T \frac{\partial g_t(\theta, \hEta)}{\partial \theta} \\
        &= \frac{1}{T} \sum_{t=1}^T \frac{\partial (m(s_t) \odot \Tilde{g}_t(\theta, \hEta))}{\partial \theta} \\
        &= \frac{1}{T} \sum_{t=1}^T \left( \underbrace{\left[ m(s_t), m(s_t), \hdots, m(s_t) \right]}_{D \, \text{times}}  \odot \left[\frac{\partial \Tilde{g}_t(\theta, \hEta)}{\partial \theta} \right] \right), \\
        &= \frac{1}{T} \sum_{t=1}^T \left( m_G(s_t)  \odot \left[\frac{\partial \Tilde{g}_t(\theta, \hEta)}{\partial \theta} \right] \right), \, \text{and} \\
    \widehat{\Omega}_{T}(\theta) &= \frac{1}{T} \sum_{t=1}^T \left[ g_t(\theta, \hEta) g_t(\theta, \hEta)^\top \right] \\
        &= \frac{1}{T} \sum_{t=1}^T \left( \left[ m(s_t) m(s_t)^\top \right] \odot \left[ \Tilde{g}_t(\theta, \hEta) \Tilde{g}_t(\theta, \hEta)^\top \right] \right), \\
        &= \frac{1}{T} \sum_{t=1}^T \left( m_{\Omega}(s_t) \odot \left[ \Tilde{g}_t(\theta, \hEta) \Tilde{g}_t(\theta, \hEta)^\top \right] \right),
\end{align*}
where $m_G$ and $m_{\Omega}$ are defined in Eqs.~\ref{eq:apdx-m-g-grad},\ref{eq:apdx-m-omega}.

\paragraph{Convergence of $\widehat{G}_T(\widehat{\theta}_T)$.}
Let $G(\theta, \eta) = \EB\left[ \frac{\partial \Tilde{g}_t(\theta, \eta)}{\partial \theta} \right]$.
Using Condition~(iii) and applying Lemma~\ref{lemma:apdx-uniform-conv-prob-semiparametric} to each element of $\widehat{G}$
with a union bound (similar to Eq.~\ref{eq:apdx-consistency-proof-c}), we get
\begin{align}
    & \sup_{\theta \in \Theta} \left\| \widehat{G}_T(\theta) - \left( \frac{1}{T} \sum_{t=1}^T m_G(s_t) \right) \odot G(\theta, \eta^*) \right\| \ConvProb 0, \nonumber \\
    & \sup_{\theta \in \Theta} \left\| \widehat{G}_T(\theta) - m_G(\kappa_T) \odot G(\theta, \eta^*) \right\| \ConvProb 0, \nonumber \\
    & \therefore \left\| \widehat{G}_T(\widehat{\theta}_T) - m_G(\kappa_T) \odot G(\widehat{\theta}_T, \eta^*) \right\| \ConvProb 0. \label{eq:apdx-g-hat-conv}
\end{align}
Since $\kappa_T \ConvProb \kappa_{\infty}$, by continuous mapping,
$m_G(\kappa_T) \ConvProb m_G(\kappa_{\infty})$.
Since $\widehat{\theta}_T \ConvProb \theta^*$ (by Condition~(i)), 
$G(\widehat{\theta}_T, \eta^*) \ConvProb G(\theta^*, \eta^*)$.
Therefore, for $G_{*}(\kappa_{\infty}) := m_G(\kappa_{\infty}) \odot G(\theta^*, \eta^*)$,
we have $\widehat{G}_T(\widehat{\theta}_T) \ConvProb G_{*}(\kappa_{\infty})$.

\paragraph{Convergence of the weight matrix $\widehat{\Omega}_T(\widehat{\theta}_T^{(\text{os})})$.}
Similarly, for $\Omega(\theta, \eta) := \EB [ \Tilde{g}_t(\theta, \eta) \Tilde{g}_t^\top(\theta, \eta) ]$,
\begin{align*}
    \left\|\widehat{\Omega}_T(\widehat{\theta}_T^{(\text{os})}) - m_{\Omega}(\kappa_T) \odot \Omega(\widehat{\theta}_T^{(\text{os})}, \eta^*) \right\| \ConvProb 0.
\end{align*}
Since $\kappa_T \ConvProb \kappa_{\infty}$, by continuous mapping,
$m_{\Omega}(\kappa_T) \ConvProb m_{\Omega}(\kappa_{\infty})$.
Since $\widehat{\theta}_T^{(\text{os})} \ConvProb \theta^*$,
for $\Omega_{*}(\kappa_{\infty}) := m_\Omega(\kappa_{\infty}) \odot \Omega(\theta^*, \eta^*)$,
we have $\widehat{\Omega}_T(\widehat{\theta}_T) \ConvProb \Omega_{*}(\kappa_{\infty})$.
By continuity of the matrix inverse, we have
$\widehat{\Omega}_T(\widehat{\theta}_T^{(\text{os})})^{-1} \ConvProb \Omega_{*}(\kappa_{\infty})^{-1}$.

\paragraph{Asymptotic normality of $\sum_{t=1}^{T} g_t(\theta^*, \hEta) / \sqrt{T}$.}
Recall that 
\begin{align*}
    g_t(\theta^*, \eta) = m(s_t) \odot [\psi^{(1)}_t(\theta, \eta^{(1)}), \hdots, \psi^{(M)}_t(\theta, \eta^{(M)})]^\top.
\end{align*}
For any $i \in [M]$, let $D_t =  \psi^{(i)}_t(\theta^*, \hEta) - \psi^{(i)}_t(\theta^*, \eta^*)$. We have
\begin{align*}
    \frac{1}{\sqrt{T}} & \sum_{t=1}^{T} m_i(s_t) \psi^{(i)}_t(\theta^*, \hEta) \\
        &= \frac{1}{\sqrt{T}} \sum_{t=1}^{T} m_i(s_t) \psi^{(i)}_t(\theta^*, \eta^*) +
            \underbrace{\frac{1}{\sqrt{T}} \sum_{t=1}^{T} m_i(s_t) \left[ D_t - \EB_{t-1}[D_t] \right]}_{\text{EP (Empirical process)}} +
            \underbrace{\frac{1}{\sqrt{T}} \sum_{t=1}^{T} m_i(s_t) \EB_{t-1}[D_t]}_{\text{Bias}}.
\end{align*}
We first show that term \emph{EP} is $\litOp(1)$. 
Let $Q_t := m_i(s_t)(D_t - \EB_{t-1}[D_t])$. For any $\epsilon > 0$,
\begin{align}
    \PB \left( \left| \frac{1}{\sqrt{T}} \sum_{t=1}^{T} m_i(s_t) \left[ D_t - \EB_{t-1}[D_t] \right] \right| > \epsilon \right)
        &= \PB \left( \left| \frac{1}{\sqrt{T}} \sum_{t=1}^{T} Q_t \right|^2 > \epsilon^2 \right) \nonumber \\
        &\leq \frac{\sum_{t=1}^{T} \EB\left[ Q_t^2 \right] + \sum_{i \neq j} \cancelto{0}{\EB\left[ Q_i Q_j \right]}}{T \epsilon^2} \label{eq:apdx-g-hat-normal-proof-a0} \\
        &= \frac{\sum_{t=1}^{T} \EB\left[ Q_t^2 \right]}{T \epsilon^2} \nonumber \\
        &= \frac{\sum_{t} \EB\left[ m_i(s_t) \left( D_t - \EB_{t-1}[D_t] \right)^2 \right]}{T \epsilon^2} \nonumber \\
        &\leq \frac{ \sum_{t} \EB[D^2_t]}{T \epsilon^2} \label{eq:apdx-g-hat-normal-proof-a1} \\
        &\leq \frac{L \sum_{t} \EB [ \| \hEta - \eta^* \|^2]}{T \epsilon^2} \label{eq:apdx-g-hat-normal-proof-a} \\
        &= o\left( 1 \right) \label{eq:apdx-g-hat-normal-proof-b},
\end{align}
where (\ref{eq:apdx-g-hat-normal-proof-a0}) follows by Markov's inequality,
(\ref{eq:apdx-g-hat-normal-proof-a1}) because
$m_i(s_t) \in \{0, 1\}$,
$\EB\left[ Q_i Q_j \right] = 0$ because the sequence
$Q_t$ is a MDS,
(\ref{eq:apdx-g-hat-normal-proof-a}) follows by 
Lipschitzness (Property~\ref{property:ulln}(ii))
and (\ref{eq:apdx-g-hat-normal-proof-b})
because $\EB [ \| \hEta - \eta^* \|^2] = o(1)$
(by Property~\ref{property:ulln}(iv) and Prop.~\ref{prop:apdx-conv-of-moments}).

Next, we show that the term \emph{Bias} is $\litOp(1)$. By the Taylor expansion, we have
\begin{align}
    \EB_{t-1} [\psi^{(i)}_t(\theta^*, \hEta)] &= \EB [\psi^{(i)}_t(\theta^*, \eta^*)] + \underbrace{ \partial_{r} \EB_{t-1}[\psi_t(\theta^*, \eta^* + r ( \eta^* - \hEta))]|_{r=0}}_{= 0} + R_t \nonumber \\
        &= \EB [\psi^{(i)}_t(\theta^*, \eta^*)] + R_t, \label{eq:apdx-asymp-norm-bias-term-a}
\end{align}
where $R_t$ is the second-order remainder defined in Assumption~\ref{assum:nuisance-rates}b, and
(\ref{eq:apdx-asymp-norm-bias-term-a}) follows by Assumption~\ref{assum:nuisance-rates}a (Neyman orthogonality).
Therefore,
\begin{align*}
    \frac{1}{\sqrt{T}} \sum_{t=1}^{T} m_i(s_t) \EB_{t-1}[D_t] &\leq \frac{1}{\sqrt{T}} \sum_{t=1}^{T} m_i(s_t) R_t \\ 
        &\leq \frac{1}{\sqrt{T}} \sum_{t=1}^{T} R_t \\ 
        &\overset{(a)}{\leq} \frac{1}{\sqrt{T}} \sum_{t=1}^{T} o_p(t^{-\nicefrac{1}{2}})
        \overset{(b)}{=} o_p(1),
\end{align*}
where (a) follows by Assumption~\ref{assum:nuisance-rates}b and 
(b) follows by Lemma~\ref{lemma:apdx-op-half-converge}.
Therefore, we have
\begin{align}
    \frac{1}{\sqrt{T}} \sum_{t=1}^{T} g_t(\theta^*, \hEta) &=
        \frac{1}{\sqrt{T}} \sum_{t=1}^{T} g_t(\theta^*, \eta^*) + o_p(1). \label{eq:apdx-mix-normal-proof-bias-op}
\end{align}
Now, it remains to show that $\sum_{t=1}^{T} g_t(\theta^*, \eta^*) / \sqrt{T}$ is
asymptotically normal.
To do so, we will use the Cramer-Wold device (Prop.~\ref{prop:cramer-wold}) and 
the martingle CLT (Prop.~\ref{prop:apdx-martingale-clt}).
For any $v \in \RB^{M}$, $v^\top g_t(\theta^*, \eta^*)$ is a MDS
because $\EB_{t-1}[v^\top g_t(\theta^*, \eta^*)] = v^\top \EB_{t-1}[g_t(\theta^*, \eta^*)] = 0$.
We show that the two conditions of the martingale CLT hold:

\emph{(i) Conditional Lindeberg:}
The Lyapunov condition implies the Lindeberg condition \citep[Remark~6.25]{hausler2015stable}
and we show that it holds.
For any $\delta > 0$,
\begin{align}
    \frac{1}{T^{1 + \delta/2}} \sum_{t=1}^{T} \EB_{t-1} \left[ \left| v^\top g_t(\theta^*, \eta^*) \right|^{2 + \delta} \right] &\leq \frac{1}{T^{1 + \delta/2}} \sum_{t=1}^{T} \| v \|^{2 + \delta} \EB_{t-1} \left[ \left\| g_t(\theta^*, \eta^*) \right\|^{2 + \delta} \right] \label{eq:apdx-cond-lyapunov-a} \\
        &= \frac{\| v \|^{2 + \delta}}{T^{1 + \delta/2}} \sum_{t=1}^{T} \EB_{t-1} \left[ \left\| m(s_t) \odot \Tilde{g}_t(\theta^*, \eta^*) \right\|^{2 + \delta} \right] \nonumber \\
        &\leq \frac{\| v \|^{2 + \delta}}{T^{1 + \delta/2}} \sum_{t=1}^{T} \EB \left[ \left\| \Tilde{g}_t(\theta^*, \eta^*) \right\|^{2 + \delta} \right] \label{eq:apdx-cond-lyapunov-c} \\
        &\to 0, \label{eq:apdx-cond-lyapunov-d}
\end{align}
where (\ref{eq:apdx-cond-lyapunov-a}) holds by the Cauchy-Schwarz inequality,
(\ref{eq:apdx-cond-lyapunov-c}) because $m(s_t)$ is a binary vector,
and (\ref{eq:apdx-cond-lyapunov-d}) because $\EB \left[ \left\| \Tilde{g}_t(\theta^*, \eta^*) \right\|^{2 + \delta} \right] < \infty$
(by Property~\ref{property:ulln}(i)).

\emph{(ii) Convergence of conditional variance:}
The conditional variance is
\begin{align}
    \frac{1}{T} \sum_{t=1}^{T} \EB_{t-1}\left[ v^\top g_t(\theta^*, \eta^*) g_t(\theta^*, \eta^*)^\top v \right] &= \frac{1}{T} \sum_{t=1}^{T} v^\top \EB_{t-1}\left[ g_t(\theta^*, \eta^*) g_t(\theta^*, \eta^*)^\top \right] v \nonumber \\
        &= v^\top \left[ m_{\Omega}(\kappa_T) \odot \Omega(\theta^*, \eta^*) \right] v \label{eq:apdx-conv-variance-a0} \\
        &\ConvProb v^\top \left[ m_{\Omega}(\kappa_{\infty}) \odot \Omega(\theta^*, \eta^*) \right] v, \label{eq:apdx-conv-variance-a} \\
        &= v^\top \Omega_{*}(\kappa_{\infty}) v \nonumber,
\end{align}
where (\ref{eq:apdx-conv-variance-a}) follows because 
$\kappa_T \ConvProb \kappa_{\infty}$ (Assumption~\ref{assump:kappa-T-converges}).
Applying the martingale CLT (Prop.~\ref{prop:apdx-martingale-clt}) and the Cramer-Wold device (Prop.~\ref{prop:cramer-wold}), we get
\begin{align}
    \frac{1}{\sqrt{T}} \sum_{t=1}^{T} g_t(\theta^*, \eta^*) &\ConvDist \NM_{\kappa_{\infty}}(0, \Omega_{*}(\kappa_{\infty})), \nonumber \\
    \therefore \frac{1}{\sqrt{T}} \sum_{t=1}^{T} g_t(\theta^*, \hEta) &\ConvDist \NM_{\kappa_{\infty}}(0, \Omega_{*}(\kappa_{\infty})), \label{apdx:eq-conv-mix-normal-proof-a}
\end{align}
where (\ref{apdx:eq-conv-mix-normal-proof-a}) follows by Eq.~\ref{eq:apdx-mix-normal-proof-bias-op},
and $\NM_{\kappa_{\infty}}$ denotes a mixture of normals where
the mixture is taken over $\kappa_{\infty}$, i.e.,
it has the characteristic function
$\varphi(\mathbf{t}) = \EB_{\kappa_{\infty}}[-\frac{1}{2} \mathbf{t}^\top \Omega_{*}(\kappa_{\infty}) \mathbf{t} ]$.

\paragraph{Asymptotic normality of $\widehat{\theta}_T$.}
We can apply the Cramer-Slutsky theorem (Cor.~\ref{cor:apdx-cramer-slutsky}) to Eq.~\ref{eq:apdx-first-order-gmm-a} to get:
\begin{align*}
    \sqrt{T} (\widehat{\theta}_T - \theta^*) &\ConvDist \NM_{\kappa_{\infty}}(0, \Sigma_{*}(\kappa_{\infty})),
\end{align*}
where 
\begin{align}
    \Sigma_{*}(\kappa_{\infty}) = \left[ G_{*}(\kappa_{\infty})^{\top} \Omega_{*}(\kappa_{\infty})^{-1} G_{*}(\kappa_{\infty}) \right]^{-1}. \label{eq:apdx-Sigma-star}
\end{align}

\paragraph{Asymptotic normality of $\widehat{\beta}_T$.}
Recall that $\widehat{\beta}_T = \fTar(\widehat{\theta}_T)$. By the Taylor expansion,
\begin{align}
    \widehat{\beta}_T = \fTar(\widehat{\theta}_T) &= \fTar(\theta^*) + \nabla_{\theta} \fTar(\Tilde{\theta})^{\top} (\widehat{\theta}_T - \theta^*), \nonumber \\
    \therefore \sqrt{T}(\widehat{\beta}_T - \beta^*)   &= \nabla_{\theta} \fTar(\Tilde{\theta})^{\top} \sqrt{T} (\widehat{\theta}_T - \theta^*), \label{eq:apdx-beta-delta-method-a0} \\
        &= \nabla_{\theta} \fTar(\theta^*)^{\top} \sqrt{T} (\widehat{\theta}_T - \theta^*) + \litOp(1). \label{eq:apdx-beta-delta-method-a} \\
    \therefore \sqrt{T}(\widehat{\beta}_T - \beta^*) &\ConvDist \NM_{\kappa_{\infty}}(0, V_{*}(\kappa_{\infty})), \nonumber
\end{align}
where $\Tilde{\theta}$ is point on the line segment
joining $\widehat{\theta}_T$ and $\theta^*$,
(\ref{eq:apdx-beta-delta-method-a}) follows because $\Tilde{\theta} \ConvProb \theta^*$, and
\begin{align}
    V_{*}(\kappa_{\infty}) = \nabla_{\theta} \fTar(\theta^*)^{\top} \Sigma_{*}(\kappa_{\infty}) \nabla_{\theta} \fTar(\theta^*). \label{eq:apdx-var-star-expression}
\end{align}
\end{proof}

\newtheorem*{prop:asymp-inference}{Proposition~\ref{prop:asymp-inference}}
\begin{prop:asymp-inference}[Asymptotic inference]
Suppose that the Conditions of Prop.~\ref{prop:asymp-normality} hold.
For
\begin{align}
    \widehat{G}_T(\widehat{\theta}_T) &:= \left( \frac{1}{T} \sum_{t}^{T} \nabla_{\theta} g_t(\widehat{\theta}_T, \hEta) \right), \nonumber \\
    \widehat{\Omega}_T(\widehat{\theta}_T) &:= \left( \frac{1}{T} \sum_{t}^{T}  g_t(\widehat{\theta}_T, \hEta) g^{\top}_t(\widehat{\theta}_T, \hEta) \right), \nonumber \\
    \widehat{\Sigma}_T &:= \left[ \widehat{G}^{\top}_T(\widehat{\theta}_T) \widehat{\Omega}^{-1}_T(\widehat{\theta}_T) \widehat{G}_T(\widehat{\theta}_T) \right]^{-1}, \nonumber \\
    \hVar_{T} &:= \nabla_\theta \fTar(\widehat{\theta}_T)^\top \widehat{\Sigma}_T \nabla_\theta \fTar(\widehat{\theta}_T),
\end{align}
we have 
\begin{align*}
    \hVar^{-1/2}_{T} \sqrt{T} (\widehat{\beta}_T - \beta^*) \ConvDist \NM\left(0, 1 \right).
\end{align*}
\end{prop:asymp-inference}
\begin{proof}
In the proof of Prop.~\ref{prop:asymp-normality},
we have shown that $\hVar_{T} \ConvProb V_{*}(\kappa_{\infty})$.
We can apply the Cramer-Slutsky theorem (Cor.~\ref{cor:apdx-cramer-slutsky})
to get the desired result:
\begin{align*}
    \hVar^{-1/2}_{T} \sqrt{T} (\widehat{\beta}_T - \beta^*) &\ConvDist V_{*}(\kappa_{\infty})^{-1/2} \NM_{\kappa_{\infty}}(0, V_{*}(\kappa_{\infty})) \\
        &= \NM\left(0, 1 \right).
\end{align*}
\end{proof}

\newtheorem*{prop:regret-is-non-negative}{Proposition~\ref{prop:regret-is-non-negative}}
\begin{prop:regret-is-non-negative}
Suppose that (i) $\widehat{\theta}_T$ is estimated using Eq.~\ref{eq:gmm-estimator-argmin}, 
(ii) Assumption~\ref{assump:kappa-star-identify} holds,
and
(iii) the conditions of Prop.~\ref{prop:asymp-normality} hold.
Then, for any data collection policy $\pi$, $R_{\infty}(\pi) \geq 0$.
\end{prop:regret-is-non-negative}
\begin{proof}
By Prop.~\ref{prop:asymp-normality}, we have
\begin{align}
    \sqrt{T}(\widehat{\beta}_T - \beta^*) &\ConvDist \NM_{\kappa_{\infty}}(0, V_{*}(\kappa_{\infty})). \nonumber \\
    \therefore \text{AMSE}(\sqrt{T} (\widehat{\beta}_T - \beta^*) ) &= \EB_{\kappa_\infty} [V_{*}(\kappa_\infty)] \nonumber \\
        &\geq \EB_{\kappa_\infty} [V_{*}(\kappa^{*})] \label{eq:apdx-regret-nonneg-proof-a} \\
        &= V_{*}(\kappa^{*}), \nonumber
\end{align}
where (\ref{eq:apdx-regret-nonneg-proof-a}) follows because
$\kappa^{*}$ minimizes $V_{*}(\kappa)$ (Assumption~\ref{assump:kappa-star-identify}).
\end{proof}

\section{Proofs for OMS-ETC (Section~\ref{sec:oms-etc})}
\label{sec:apdx-oms-etc}
\newtheorem*{defn:offpolicy-var-estimate}{Definition~\ref{defn:offpolicy-var-estimate}}
\begin{defn:offpolicy-var-estimate}[Variance estimator]
For any $\kappa \in \ChoiceSimplex$, we define
\begin{align}
    & \widehat{G}_T(\theta, \kappa) := m_G(\kappa) \odot (1 / m_G(\kappa_T)) \odot \widehat{G}_T(\theta), \nonumber \\
    & \widehat{\Omega}_T(\theta, \kappa) := m_{\Omega}(\kappa) \odot (1 / m_{\Omega}(\kappa_T)) \odot \widehat{\Omega}_T(\theta), \nonumber \\
    & \widehat{\Sigma}_T(\theta, \kappa) := \left[ \widehat{G}^{\top}_T(\theta, \kappa) \widehat{\Omega}^{-1}_T(\theta, \kappa) \widehat{G}_T(\theta, \kappa) \right]^{-1}, \nonumber \\
    & \hVar_{T}(\theta, \kappa) := \nabla_\theta \fTar(\theta)^\top \widehat{\Sigma}_T(\theta, \kappa) \nabla_\theta \fTar(\theta), \label{eq:apdx-var-estimator-defn}
\end{align}
where $m_G(\kappa)$ and $m_\Omega(\kappa)$ are defined in Eqs.~\ref{eq:apdx-m-g-grad} and \ref{eq:apdx-m-omega};
and $1 / m_G(\kappa_T)$ represents the element-wise reciprocal
of the non-zero elements of the matrix (and likewise for $1/ m_{\Omega}(\kappa_T)$).
\end{defn:offpolicy-var-estimate}

\newtheorem*{lemma:oracle-simplex-estimation}{Lemma~\ref{lemma:oracle-simplex-estimation}}
\begin{lemma:oracle-simplex-estimation}
Let $\widehat{k}_T := \arg\min_{\kappa \in \ChoiceSimplex} \hVar_{T}(\widehat{\theta}_T, \kappa)$
be the estimated oracle simplex.
Suppose that 
(i) the conditions of Prop.~\ref{prop:consistency-gmm} hold,
(ii) Assumption~\ref{assump:kappa-star-identify} holds,
and 
(iii) $\forall i \in [M], j \in [D], \partial_{\theta_j} 
\left[\psi^{(i)}_t(\theta, \eta) \right]$
satisfies Property~\ref{property:ulln}.
Then $\widehat{k}_T \ConvProb \kappa^*$.
\end{lemma:oracle-simplex-estimation}
\begin{proof}
We prove this Lemma by showing uniform convergence of 
the variance estimator over $\kappa \in \ChoiceSimplex$.
We use a covering number argument over the compact set $\ChoiceSimplex$.

In the proof of Prop.~\ref{prop:asymp-normality} (by Condition~(iii)), 
we showed that (see Eq.~\ref{eq:apdx-g-hat-conv})
\begin{align}
    \left\| \widehat{G}_T(\widehat{\theta}_T) - m_G(\kappa_T) \odot G(\widehat{\theta}_T, \eta^*) \right\| &\ConvProb 0 \label{eq:apdx-var-estimator-G-conv-a}.
\end{align}
For any $\kappa \in \ChoiceSimplex$, we have the following \emph{pointwise} convergence:
\begin{align}
    \widehat{G}_T(\widehat{\theta}_T, \kappa) &= m_G(\kappa) \odot (1 / m_G(\kappa_T)) \odot \widehat{G}_T(\widehat{\theta}_T) \nonumber \\
        &= m_G(\kappa) \odot (1 / m_G(\kappa_T)) \odot m_G(\kappa_T) \odot G(\widehat{\theta}_T, \eta^*) + \litOp(1) \label{eq:apdx-var-estimator-G-conv-b} \\
        &= m_G(\kappa) \odot G(\widehat{\theta}_T, \eta^*) + \litOp(1) \nonumber \\
        &= m_G(\kappa) \odot G(\theta^*, \eta^*) + \litOp(1) \label{eq:apdx-var-estimator-G-conv-b1} \\
        &= G_{*}(\kappa) + \litOp(1), \label{eq:apdx-var-estimator-G-conv-c}
\end{align}
where (\ref{eq:apdx-var-estimator-G-conv-b}) follows by Eq.~\ref{eq:apdx-var-estimator-G-conv-a} and
(\ref{eq:apdx-var-estimator-G-conv-b1}) because $\widehat{\theta}_T \ConvProb \theta^*$.

We now show uniform convergence over $\ChoiceSimplex$
using a covering number argument.  
Let $\left\{ \kappa_j \right\}_{j=1}^{J}$ be a minimal 
$\delta$-cover of $\ChoiceSimplex$ and
$\BM_{\delta}(\kappa_j)$ denotes the $\delta$-ball around $\kappa_j$. 
By compactness of $\ChoiceSimplex$, $J$ is finite.
For the rest of the proof, for any $\kappa$, 
let $\kappa_j$ denote the point in the $\delta$-cover
such that $\kappa \in \BM_{\delta}(\kappa_j)$. 
By the triangle inequality, 
\begin{align*}
    & \sup_{\kappa \in \ChoiceSimplex} \left\| \widehat{G}_T(\widehat{\theta}_T, \kappa) -  G_{*}(\kappa) \right\| \\
        &\leq \underbrace{\sup_{\kappa} \left\| \widehat{G}_T(\widehat{\theta}_T, \kappa) -  \widehat{G}_T(\theta, \kappa_j) \right\|}_{\text{T1}} + \underbrace{\max_{j \in [J]} \left\| \widehat{G}_T(\widehat{\theta}_T, \kappa_j) -  G_{*}(\kappa_j) \right\|}_{\text{T2}} + \underbrace{\max_{j \in [J]} \left\| G_{*}(\kappa_j) - G_{*}(\kappa) \right\|}_{\text{T3}}.
\end{align*}
We now show that the three terms are $\litOp(1)$.
\paragraph{Term (T1).} 
We have
\begin{align}
   \sup_{\kappa} \left\| \widehat{G}_T(\widehat{\theta}_T, \kappa) -  \widehat{G}_T(\theta, \kappa_j) \right\| &= \sup_{\kappa} \left\| [m_G(\kappa) - m_G(\kappa_j)] \odot (1 / m_G(\kappa_T)) \odot \widehat{G}_T(\widehat{\theta}_T)  \right\| \nonumber \\
    &\leq O(\delta) \left\|(1 / m_G(\kappa_T)) \odot \widehat{G}_T(\widehat{\theta}_T)  \right\| \label{eq:apdx-var-uniform-a} \\
    &= O(\delta) \left\|(1 / m_G(\kappa_T)) \odot m_G(\kappa_T) \odot G(\widehat{\theta}_T, \eta^*)  \right\| + \litOp(1)  \label{eq:apdx-var-uniform-b}  \\
    &= O(\delta) \left\|G(\widehat{\theta}_T, \eta^*)  \right\| + \litOp(1) \nonumber \\
    &= O(\delta) \left\|G(\theta^*, \eta^*)  \right\| + \litOp(1) \label{eq:apdx-var-uniform-c} \\
    &= O(\delta) + \litOp(1) \label{eq:apdx-var-uniform-d},
\end{align}
where (\ref{eq:apdx-var-uniform-a}) follows 
because $\| m_G(\kappa) - m_G(\kappa_j) \| \leq DM \delta = O(\delta)$
(by Lipschitzness of $m_G(\kappa)$),
(\ref{eq:apdx-var-uniform-b}) by Eq.~\ref{eq:apdx-var-estimator-G-conv-a},
(\ref{eq:apdx-var-uniform-c}) because $\widehat{\theta}_T \ConvProb \theta^*$,
and (\ref{eq:apdx-var-uniform-d}) because $\left\|G(\theta^*, \eta^*)  \right\| < \infty$ (by Condition~(iii)).

\paragraph{Term (T2).}
We have
\begin{align*}
    \max_{j \in [J]} \left\| \widehat{G}_T(\widehat{\theta}_T, \kappa_j) -  G_{*}(\kappa_j) \right\| &\leq \sum_{j \in [J]} \left\| \widehat{G}_T(\widehat{\theta}_T, \kappa_j) -  G_{*}(\kappa_j) \right\| \\
        &= \litOp(1),
\end{align*}
where the last line follows by Eq.~\ref{eq:apdx-var-estimator-G-conv-c}
and finiteness of $J$.

\paragraph{Term (T3).}
We have
\begin{align}
    \max_{j \in [J]} \left\| G_{*}(\kappa_j) - G_{*}(\kappa) \right\| &= \max_{j \in [J]} \left\| [m_G(\kappa_j) - m_G(\kappa) ] \odot G(\theta^*, \eta^*) \right\| \nonumber \\
        &\leq O(\delta) \left\| G(\theta^*, \eta^*) \right\| \label{eq:apdx-var-uniform-t3-a}, \\
        &= O(\delta), \label{eq:apdx-var-uniform-t3-b}
\end{align}
where (\ref{eq:apdx-var-uniform-t3-a}) follows because $\| m_G(\kappa) - m_G(\kappa_j) \| = O(\delta)$ and
(\ref{eq:apdx-var-uniform-t3-b}) because $\left\|G(\theta^*, \eta^*)  \right\| < \infty$.

Combining the three terms, we have
\begin{align*}
    \sup_{\kappa \in \ChoiceSimplex} \| \widehat{G}_T(\widehat{\theta}_T, \kappa) - G_{*}(\kappa) \| &= \litOp(1), \nonumber \\
    \therefore \| \widehat{G}_T(\widehat{\theta}_T, \widehat{k}_T) -  G_{*}(\widehat{k}_T) \| &= \litOp(1).
\end{align*}

In the same way, we can show that
\begin{align*}
    \sup_{\kappa \in \ChoiceSimplex} \| \widehat{\Omega}_T(\widehat{\theta}_T, \kappa) - \Omega_{*}(\kappa) \| &= \litOp(1), \\
    \therefore \| \widehat{\Omega}_T(\widehat{\theta}_T, \widehat{k}_T) - \Omega_{*}(\widehat{k}_T) \| &=  \litOp(1).
\end{align*}
By the continuous mapping theorem,
\begin{align*}
    | \hVar_{T}(\widehat{\theta}_T, \widehat{k}_T) - V_{*}(\widehat{k}_T) | &= \litOp(1).
\end{align*}

To summarize, we have the following:
\begin{align}
    & \kappa^* = \arg\min_{\kappa \in \ChoiceSimplex} V_{*}(\kappa), \label{eq:apdx-var-argmin-a} \\
    & \widehat{k}_{T} = \arg\min_{\kappa \in \ChoiceSimplex} \hVar_{T}(\widehat{\theta}_T, \kappa), \label{eq:apdx-var-argmin-b} \\
    & | \hVar_{T}(\widehat{\theta}_T, \widehat{k}_T) - V_{*}(\widehat{k}_T)| =  \litOp(1), \label{eq:apdx-var-argmin-c} \\
    & | \hVar_{T}(\widehat{\theta}_T, \kappa^*) - V_{*}(\kappa^*) |  = \litOp(1). \label{eq:apdx-var-argmin-d}
\end{align}
Next, we follow the proof of 
\citet[Thm.~2.1]{newey1994large}.
For any $\epsilon > 0$, with probability approaching (w.p.a.) $1$,
\begin{align}
    V_{*}(\widehat{k}_T) &\leq \hVar_{T}(\widehat{\theta}_T, \widehat{k}_T) + \epsilon \label{eq:apdx-var-argmin-e} \\
        &\leq \hVar_{T}(\widehat{\theta}_T, \kappa^{*}) + \epsilon \label{eq:apdx-var-argmin-f} \\
        &\leq V_{*}(\kappa^{*}) + 2 \epsilon \label{eq:apdx-var-argmin-g},
\end{align}
where (\ref{eq:apdx-var-argmin-e}) follows by Eq.~\ref{eq:apdx-var-argmin-c},
(\ref{eq:apdx-var-argmin-f}) by Eq.~\ref{eq:apdx-var-argmin-b},
and (\ref{eq:apdx-var-argmin-g}) by Eq.~\ref{eq:apdx-var-argmin-d}.
Let $\BM$ be any open subset of $\ChoiceSimplex$ containing $\kappa^*$
and let $\BM^{c}$ denote its complement.
Let $\Tilde{\kappa} = \arg\min_{\kappa \in (\ChoiceSimplex \cap \BM^{c})} V_{*}(\kappa)$.
The minimum exists because $(\ChoiceSimplex \cap \BM^{c})$ is compact.
By Eq.~\ref{eq:apdx-var-argmin-a},
we have $V_{*}(\kappa^*) \leq V_{*}(\Tilde{\kappa})$.
Thus, for $\epsilon = \frac{1}{2} ( V_{*}(\Tilde{\kappa}) -  V_{*}(\kappa^*)) > 0$ 
(by Assumption~\ref{assump:kappa-star-identify}),
we have $V_{*}(\widehat{k}_T) \leq V_{*}(\Tilde{\kappa})$ w.p.a. $1$.
Therefore $\widehat{k}_T \in \BM$, completing the proof.
\end{proof}

\section{Proofs for OMS-ETG (Section~\ref{sec:oms-etg})}
\label{sec:apdx-oms-etg}
\newtheorem*{thm:etg-finite-rounds}{Theorem~\ref{thm:etg-finite-rounds}}
\begin{thm:etg-finite-rounds}
Suppose that
(i) the conditions of Theorem~\ref{thm:etc-regret} hold
and
(ii) (Finite rounds) $\lim_{T \to \infty} S / T = r$
for some constant $r \in (0, 1)$.
Then, $R_\infty(\pi_{\text{ETG}}) = 0$.
\end{thm:etg-finite-rounds}
\begin{proof}
For any $\epsilon > 0$ and $t_{j} = Te + jS$ for $j \in [0, \hdots, J]$, we have
\begin{align}
    \PB\left( \forall j \in [0, \hdots, J], \widehat{k}_{t_j} \in \BM_{\epsilon}(\kappa^*) \right) &\geq 1 - \sum_{j = 0}^ {J} \PB\left( \widehat{k}_{t_j} \notin \BM_{\epsilon}(\kappa^*) \right) \label{eq:etg-finite-proof-a} \\
        &= 1 - o_p(1), \label{eq:etg-finite-proof-b}
\end{align}
where (\ref{eq:etg-finite-proof-a}) follows by the union bound
and (\ref{eq:etg-finite-proof-b}) because $\forall j, \widehat{k}_{t_j} \ConvProb \kappa^*$ (by Lemma~\ref{lemma:oracle-simplex-estimation}).
Note that $\kappa_{t_{j+1}}$ moves as close as possible to $\widehat{k}_{t_j}$ after every round.
By Eq.~\ref{eq:etg-finite-proof-b}, 
this means that
$\kappa_T$ moves towards $\BM_{\epsilon}(\kappa^*)$
after every round and thus
we have
$\| \kappa_T - \text{proj}(\kappa^*, \Tilde{\Delta}_e )\| \ConvProb 0$.
Due to negligible exploration ($e \in o(1)$),
the set $\Tilde{\Delta}_e$ asymptotically covers the entire simplex $\ChoiceSimplex$, i.e.,
$\text{proj}(\kappa^*, \Tilde{\Delta}_e ) \ConvProb \kappa^*$.
Thus $\kappa_T \ConvProb \kappa^*$ and
by Lemma~\ref{lemma:zero-regret-policy}, $\pi_{\text{ETG}}$ has zero regret.
\end{proof}

\begin{lemma}[{\citep[Prop.~B.4]{waudby2021time}}]\label{lemma:conv-as-lim-sup-conv-prob}
Let $(X_t)_{t=1}^{\infty}$ be a sequence of random variables.
Then $X_t \ConvAS 0 \iff \sup_{T \geq t} X_T \ConvProb 0$.
Equivalently, $X_t \ConvAS 0$ if and only if
\begin{align*}
    \forall \epsilon > 0, \, \lim_{t \to \infty} \PB(\exists T > t: |X_T| > \epsilon) = 0.
\end{align*}
\end{lemma}

\begin{lemma}\label{lemma:avg-eta-convas}
Let $X_t$ be a sequence of
non-negative random variables
such that
(i) $X_t \ConvAS 0$ and 
(ii) $\sup_t \EB[X_{t}^{1+\delta}] < \infty$ for some $\delta > 0$. 
Then, $\sum_{t=1}^{T} X_t / T = \litOas(1)$.
\end{lemma}
\begin{proof}
Let $S_T := \sum_{t=1}^{T} X_t$. For any $T > t$,
\begin{align*}
    \frac{S_T}{T} &\leq \frac{S_{t}}{t} + \frac{1}{T} \sum_{i=t+1}^{T} X_i.
\end{align*}
By the union bound, for any $\epsilon > 0$,
\begin{align}
    \PB\left( \forall T > t : \frac{S_{T}}{T} \leq 2 \epsilon \right) &\geq 1 - \underbrace{\PB\left( \frac{S_{t}}{t} > \epsilon \right)}_{:= a} - \underbrace{\PB\left( \exists T > t : \frac{1}{T} \sum_{i=t+1}^{T} X_i > \epsilon \right)}_{:= b}. \label{eq:apdx-avg-eta-convas-proof-a}
\end{align}
We show that Term \emph{a} in Eq.~\ref{eq:apdx-avg-eta-convas-proof-a} is $o(1)$. By Markov's inequality, we have
\begin{align}
    \PB\left( \frac{S_{t}}{t} > \epsilon \right) &\leq \frac{ \sum_{i=1}^{t} \EB[X_i]}{t \epsilon} \nonumber \\
        &= \frac{o(t)}{t \epsilon} \label{eq:apdx-avg-eta-convas-proof-b} \\
        &= o(1), \nonumber
\end{align}
where (\ref{eq:apdx-avg-eta-convas-proof-b}) follows by Condition~(ii) and Prop.~\ref{prop:apdx-conv-of-moments}.
Next, we show that Term \emph{b} in Eq.~\ref{eq:apdx-avg-eta-convas-proof-a} is $o(1)$. 
For any $\epsilon > 0$,
\begin{align}
    \PB\left( \forall T > t : \frac{1}{T} \sum_{i=t+1}^{T} X_{i} < \epsilon \right) &\geq \PB\left( \forall T > t : X_{T} < \epsilon \right) \nonumber \\
    &= 1 - o(1), \label{eq:apdx-avg-eta-convas-proof-c} \\
    \therefore \PB\left( \exists T > t : \frac{1}{T} \sum_{i=t}^{T} X_i > \epsilon \right) &= o(1), \nonumber
\end{align}
where (\ref{eq:apdx-avg-eta-convas-proof-c}) follows
by Lemma~\ref{lemma:conv-as-lim-sup-conv-prob}
because $X_t \ConvAS 0$.
Using these results in Eq.~\ref{eq:apdx-avg-eta-convas-proof-a}, we get
\begin{align*}
    \PB\left( \forall T > t : \frac{S_{T}}{T} \leq 2 \epsilon \right) &\geq 1 - o(1).
\end{align*}
Therefore, by Lemma~\ref{lemma:conv-as-lim-sup-conv-prob}, $S_T / T = \litOas(1)$.
\end{proof}

\begin{lemma}[Strong uniform convergence]\label{lemma:apdx-uniform-conv-almost-sure-semiparametric}
Suppose that (i) $a_t(\theta, \eta) := a(X_t; \theta, \eta)$ satisfies Property~\ref{property:ulln};
(ii) (Nuisance strong consistency) $\| \hEta - \eta^* \| \ConvAS 0$; and
(iii) $\forall \eta \in \TM, \, \EB[\| \eta - \eta^* \|^2] < \infty$.
Let $a_{*}(\theta, \eta) = \EB[a(X_t; \theta, \eta)]$ and
$s_t \in \{ 0, 1 \}$ be $H_{t-1}$-measurable. Then,
\begin{align*}
    \sup_{\theta \in \Theta} \left| \frac{1}{T} \sum_{t=1}^{T} s_t
    \left\{ a_t(\theta, \hEta)
     - a_*(\theta, \eta^*)
    \right\} \right| \ConvAS 0.
\end{align*}
\end{lemma}
\begin{proof}
We prove this in the same way as Lemma~\ref{lemma:apdx-uniform-conv-prob-semiparametric},
strengthening convergence in probability to almost sure convergence
using the stronger assumptions on the nuisance estimators.
We begin with the same decomposition as in Lemma~\ref{lemma:apdx-uniform-conv-prob-semiparametric}
and show that each term converges almost surely to zero.
Let $\left\{ \theta_i \right\}_{i=1}^{K}$ be a minimal $\delta$-cover of $\Theta$ and
$\BM_{\delta}(\theta_k)$ denote the $\delta$-ball around $\theta_k$. 
By compactness of $\Theta$, $K$ is finite.
Going forward, for any $\theta \in \Theta$, 
let $\theta_k$ denote the element of the $\delta$-cover
such that $\theta \in \BM_{\delta}(\theta_k)$.
By the triangle inequality,
\begin{align*}
    \sup_{\theta \in \Theta} & \left| \frac{1}{T} \sum_{t=1}^{T}
    s_t \left\{ a_t(\theta, \hEta)
     - a_*(\theta, \eta^*) \right\} \right| \\
        &\leq \begin{aligned}[t]
            & \underbrace{\frac{1}{T} \sum_{t=1}^{T} \sup_{\theta} \left| a_t(\theta, \hEta) - a_t(\theta_k, \hEta) \right|}_{\text{T1}} +
            \underbrace{\max_{k \in [K]} \left| \frac{1}{T} \sum_{t=1}^{T} s_t \left\{ a_t(\theta_k, \hEta) - a_{*}(\theta_k, \eta^*) \right\}  \right|}_{\text{T2}} + \\
            & \underbrace{\max_{k \in [K]} \left| a_{*}(\theta_k, \eta^*) - a_{*}(\theta, \eta^*) \right|}_{\text{T3}}
        \end{aligned}.
\end{align*}

\paragraph{Term (T1).}
We showed in the proof of Lemma~\ref{lemma:apdx-uniform-conv-prob-semiparametric}
that this term is $\litOas(1)$.

\paragraph{Term (T2).}
We have
\begin{align*}
    & \left| \frac{1}{T} \sum_t s_t \left[ a_t(\theta_k, \hEta) - a_{*}(\theta_k, \eta^*) \right] \right| \\
    &= \begin{aligned}[t]
        & \frac{1}{T} \sum_t s_t \left[ a_t(\theta_k, \eta^*) - a_{*}(\theta_k, \eta^*) \right] + \\
        & \frac{1}{T} \sum_t s_t \left[ a_t(\theta_k, \hEta) - a_t(\theta_k, \eta^*) - \EB_{t-1}[a_t(\theta_k, \hEta) - a_t(\theta_k, \eta^*) ]  \right] + \\
        & \frac{1}{T} \sum_t s_t \EB_{t-1}[a_t(\theta_k, \hEta) - a_t(\theta_k, \eta^*)]
    \end{aligned}
\end{align*}
In the proof of Lemma~\ref{lemma:apdx-uniform-conv-prob-semiparametric},
we showed that the first two terms are $\litOas(1)$.
For the last term, we have
\begin{align}
    \frac{1}{T} \sum_t s_t \EB_{t-1}[a_t(\theta_k, \hEta) - a_t(\theta_k, \eta^*)] 
        &\leq \frac{\sqrt{L}}{T} \sum_t \| \hEta - \eta^* \| \label{eq:apdx-alm-sure-term-t2-3-a} \\
        &= \litOas(1) \label{eq:apdx-alm-sure-term-t2-3-b},
\end{align}
where (\ref{eq:apdx-alm-sure-term-t2-3-a}) follows by Property~\ref{property:ulln}(ii, iii)
and (\ref{eq:apdx-alm-sure-term-t2-3-b}) follows by Lemma~\ref{lemma:avg-eta-convas}.

\paragraph{Term (T3).}
We showed in the proof of Lemma~\ref{lemma:apdx-uniform-conv-prob-semiparametric}
that this term is $\litOas(1)$.
\end{proof}

\newtheorem*{lemma:oracle-simplex-estimation-almost-sure}{Lemma~\ref{lemma:oracle-simplex-estimation-almost-sure}}
\begin{lemma:oracle-simplex-estimation-almost-sure}
Suppose that (i) the conditions of Lemma~\ref{lemma:oracle-simplex-estimation} hold and
(ii) (Nuisance strong consistency) $\| \hEta - \eta^* \| \ConvAS 0$.
Then $\widehat{\theta}_T \ConvAS \theta^*$
and $\widehat{k}_T \ConvAS \kappa^*$.
\end{lemma:oracle-simplex-estimation-almost-sure}
\begin{proof}
This can be proved in the same way as Lemma~\ref{lemma:oracle-simplex-estimation},
using Lemma~\ref{lemma:apdx-uniform-conv-almost-sure-semiparametric}
to strengthen convergence in probability to
almost sure convergence.
\end{proof}

\newtheorem*{thm:etg-infinite-rounds}{Theorem~\ref{thm:etg-infinite-rounds}}
\begin{thm:etg-infinite-rounds}
Suppose that (i) the conditions of Theorem~\ref{thm:etc-regret} hold and
(ii) the conditions of Lemma~\ref{lemma:oracle-simplex-estimation-almost-sure} hold.
Then, for any batch size $S$, $R_\infty(\pi_{\text{ETG}}) = 0$.
\end{thm:etg-infinite-rounds}
\begin{proof}
Let $\widehat{A}_{\epsilon}(t)$ denote the event that 
the estimated oracle simplex $\widehat{k}$ is inside 
$\BM_{\epsilon}(\kappa^*)$ for all time steps after $t$: 
$\widehat{A}_{\epsilon}(t) := \{ \forall t' > t : \widehat{k}_{t'} \in \BM_{\epsilon}(\kappa^*) \}$.
We define an analogous event for the selection simplex $\kappa_t$ (Defn.~\ref{defn:selection-simplex}):
$A_{\epsilon}(t) := \{ \forall t' > t : \kappa_{t'} \in \BM_{\epsilon}(\kappa^*) \}$.
Since $Te \to \infty$ as $T \to \infty$,
by Lemma~\ref{lemma:oracle-simplex-estimation-almost-sure}, 
$\widehat{k}_{Te} \ConvAS \kappa^*$. 
By Lemma~\ref{lemma:conv-as-lim-sup-conv-prob}, this is equivalent to
\begin{align}
    \forall \epsilon > 0, \, \lim_{T \to \infty} \PB\left( \widehat{A}_{\epsilon}(Te) \right) = 1. \label{eq:etg-proof-infinite-a}
\end{align}
Observe that $\kappa_t$ is getting as close as possible to $\widehat{k}_t$ 
after each round.
Since the exploration is negligible ($e \in o(1)$),
if the event $\widehat{A}_{\epsilon}(Te)$ occurs,
for some $t_0 > Te$, the event $A_{\epsilon}(t_0)$ 
also occurs.
That is, when $\widehat{A}_{\epsilon}(Te)$ occurs,
$\kappa_{t_0}$ will eventually enter the $\epsilon$-ball around $\kappa^*$ 
at some time $t_0 > Te$ and remain inside it for subsequent time steps.
Therefore, $\forall \epsilon > 0$,
\begin{align}
     \lim_{T \to \infty} \PB\left( A_{\epsilon}(t_0) \right) &\geq \lim_{T \to \infty} \PB\left( \widehat{A}_{\epsilon}(Te) \right) \nonumber \\
      &= 1, \label{eq:etg-proof-infinite-b} \\
    \therefore \kappa_T &\ConvAS \kappa^*, \label{eq:etg-proof-infinite-c}
\end{align}
where (\ref{eq:etg-proof-infinite-b}) follows by Eq.~\ref{eq:etg-proof-infinite-a}
and (\ref{eq:etg-proof-infinite-c}) by Lemma~\ref{lemma:conv-as-lim-sup-conv-prob}.
By Lemma~\ref{lemma:zero-regret-policy}, the regret is zero.
\end{proof}

\begin{figure}[t]
\centering
{
    \setlength{\interspacetitleruled}{0pt}%
    \setlength{\algotitleheightrule}{0pt}%
    \begin{algorithm}[H]
    \SetAlgoLined
    \KwInput{Horizon $T \in \mathbb{N}$, Exploration policy $(\epsilon_t)_{t=1}^{T}$.}
    $\widehat{k} \gets \ctrSim$\;
    $n \gets 0$\;
    \For{$t \in [1, 2, \hdots, T]$}{
        $n \gets n + 1$\;
        Sample $u \sim \text{Uniform}([0, 1])$\;
        \lIf{ $u \leq \epsilon_t$ }{
            Query one of the data sources with equal probability
        }
        \uElse{
            Collect the next samples s.t. $\kappa_{n} = \widehat{k}$\;
            $\widehat{\theta}_{n} \gets \arg\min_{\theta \in \Theta} \widehat{Q}_{n}(\theta, (\hEta)_{t=1}^{n})$\;
            $\widehat{k}_{n} \gets \arg\min_{\kappa \in \ChoiceSimplex} \hVar_n(\widehat{\theta}_{n}, \kappa)$\;
            $\widehat{k} \gets \text{proj}(\widehat{k}_{n}, \Tilde{\Delta}_{n+1})$\;
        }
    }
    $\widehat{\theta}_T = \arg\min_{\theta \in \Theta} \widehat{Q}_T(\theta, (\hEta)_{t=1}^{T})$\;
    \KwOutput{$f_{\text{tar}}(\widehat{\theta}_T)$}
    \end{algorithm}
}
\label{fig:policy-algorithm-eps-greedy}
\caption{The $\epsilon$-greedy data collection policy.}
\end{figure}

\newtheorem*{thm:eps-greedy}{Theorem~\ref{thm:eps-greedy}}
\begin{thm:eps-greedy}
Let $(\epsilon_t)_{t=1}^{\infty}$ be a non-increasing sequence and
$E_T = \sum_{t=1}^{T} \epsilon_t$.
Suppose that (i) the conditions of Prop.~\ref{prop:asymp-normality} hold; 
(ii) the conditions of Lemma~\ref{lemma:oracle-simplex-estimation-almost-sure} hold;
and
(iii) $E_T = o(T)$ and 
$E_T \to \infty$ (e.g., $\epsilon_t \asymp 1 / t$).
Then, the $\epsilon$-greedy policy suffers zero regret: 
$R_\infty(\pi_{\epsilon\text{-greedy}}) = 0$.
\end{thm:eps-greedy}
\begin{proof}
This can be proved similarly to Theorem~\ref{thm:etg-infinite-rounds}.
Since $E_T \to \infty$, by Lemma~\ref{lemma:oracle-simplex-estimation-almost-sure}, 
$\widehat{k}_{E_T} \ConvAS \kappa^*$. 
Applying Lemma~\ref{lemma:conv-as-lim-sup-conv-prob}, we have
\begin{align}
    \forall \epsilon > 0, \, \lim_{T \to \infty} \PB\left( \widehat{A}_{\epsilon}(E_T) \right) = 1. \label{eq:epg-proof-infinite-a}
\end{align}
Since $\kappa_t$ is getting as close as possible to $\widehat{k}_t$,
by similar reasoning as Thm.~\ref{thm:etg-infinite-rounds},
for some $t_0 > E_T$, we have
\begin{align}
    \forall \epsilon > 0, \, \lim_{T \to \infty} \PB\left( A_{\epsilon}(t_0) \right) &\geq \lim_{T \to \infty} \PB\left( \widehat{A}_{\epsilon}(E_T) \right) 
    \nonumber \\
    &= 1 \label{eq:epg-proof-infinite-b} \\
    \therefore \kappa_T &\ConvAS \kappa^*, \label{eq:epg-proof-infinite-c}
\end{align}
where (\ref{eq:epg-proof-infinite-b}) follows by Eq.~\ref{eq:epg-proof-infinite-a},
and (\ref{eq:epg-proof-infinite-c}) by Lemma~\ref{lemma:conv-as-lim-sup-conv-prob}.
By Lemma~\ref{lemma:zero-regret-policy}, the regret is zero.
\end{proof}

\section{OMS with cost structure}
\label{sec:apdx-cost-structure}
\begin{figure}[t]
\centering
\begin{subfigure}[b]{0.49\textwidth}
{
    \setlength{\interspacetitleruled}{0pt}%
    \setlength{\algotitleheightrule}{0pt}%
    \begin{algorithm}[H]
    \SetAlgoLined
    \KwInput{Budget $B$, Exploration $e$, Cost $c$}
    $n \gets \floor{Be / (\ctrSim^\top c)}$\;
    Collect $n$ samples s.t. $\kappa_{n} = \ctrSim$\;
    $\widehat{\theta}_{n} \gets \arg\min_{\theta \in \Theta} \widehat{Q}_n(\theta, (\hEta)_{t=1}^{n})$\;
    $\widehat{k} \gets \arg\min_{\kappa \in \ChoiceSimplex} \hVar_n(\widehat{\theta}_{n}, \kappa) (\kappa^\top c)$\;
    Use remaining budget such that $\kappa_T = \text{proj}(\widehat{k}, \Tilde{\Delta})$ (see Eq.~\ref{eq:apdx-etc-cs-feasible-region})\;
    $\widehat{\theta}_T \gets \arg\min_{\theta \in \Theta} \widehat{Q}_T(\theta, (\hEta)_{t=1}^{T})$\;
    \KwOutput{$f_{\text{tar}}(\widehat{\theta}_T)$}
    \end{algorithm}
}
\caption{The ETC-CS policy.}
\label{fig:policy-algorithm-etc-cs}
\end{subfigure}
\hfill
\begin{subfigure}[b]{0.49\textwidth}
{
    \setlength{\interspacetitleruled}{0pt}%
    \setlength{\algotitleheightrule}{0pt}%
    \begin{algorithm}[H]
    \SetAlgoLined
    \KwInput{Budget $B$, Batch size $S$, Exploration $e$, Cost $c$}
    $\widehat{k} \gets \ctrSim$, Rounds $J \gets \floor*{B(1-e) / S}$\;
    $n \gets 0$\;
    \For{$H \in [Be, \underbrace{S, \hdots, S}_{J \, \text{times}}]$}{
        $s \gets \floor{H / (\widehat{k}^\top c)}$\;
        $n \gets n + s$\;
        Collect $s$ samples s.t. $\kappa_{n} = \widehat{k}$\;
        $\widehat{\theta}_{n} \gets \arg\min_{\theta \in \Theta} \widehat{Q}_{n}(\theta, (\hEta)_{t=1}^{n})$\;
        $\widehat{k}_{n} \gets \arg\min_{\kappa \in \ChoiceSimplex} \hVar_n(\widehat{\theta}_{n}, \kappa) (\kappa^\top c)$\;
        $\widehat{k} \gets \text{proj}(\widehat{k}_{n}, \Tilde{\Delta}_{j+1})$ (see Eq.~\ref{eq:apdx-etg-cs-feasible-region})\;
    }
    $\widehat{\theta}_T = \arg\min_{\theta \in \Theta} \widehat{Q}_T(\theta, (\hEta)_{t=1}^{T})$\;
    \KwOutput{$f_{\text{tar}}(\widehat{\theta}_T)$}
    \end{algorithm}
}
\caption{The ETG-CS policy.}
\label{fig:policy-algorithm-etg-cs}
\end{subfigure}
\caption{Algorithms for OMS-ETC-CS and OMS-ETG-CS.}
\label{fig:policy-algorithm-oms-cs}
\end{figure}

\subsection{Feasible values of the selection simplex}\label{sec:apdx-cost-feasible}

\paragraph{Feasible values of $\kappa_T$ for OMS-ETC-CS.}
The agent uses $Be$ budget for exploration. The number of samples collected after
exploration is 
\begin{align*}
    T_e = \floor*{\frac{Be}{\ctrSim^\top c}}.
\end{align*}
With the remaining $B (1 - e)$ budget, the agent can collect samples with
any $\kappa \in \ChoiceSimplex$. The total number of samples collected is
\begin{align*}
    T = T_e + \floor*{\frac{B(1 - e)}{\kappa^\top c}}.
\end{align*}
Therefore, the values of $\kappa_T$ that can be achieved are
\begin{align}
    \Tilde{\Delta} &= \left\{ \frac{T_e \ctrSim + (T - T_e) \kappa}{T} : \kappa \in \ChoiceSimplex \right\}. \label{eq:apdx-etc-cs-feasible-region}
\end{align}

\paragraph{Feasible values of $\kappa_T$ for OMS-ETG-CS.}
Let the number of samples collected after round $j$ be $T_j$.
For any $\kappa \in \ChoiceSimplex$, the number of samples collected after
round $j + 1$ is
\begin{align*}
    T_{j+1} = T_j + \floor*{\frac{S}{\kappa^\top c}}.
\end{align*}
Therefore, the values that $\kappa_{T_{j+1}}$ can achieve are
\begin{align}
    \Tilde{\Delta}_{j+1} = \left\{ \frac{ T_j \kappa_{T_j} + (T_{j+1} - T_j) \kappa }{T_{j+1}} : \kappa \in \ChoiceSimplex \right\}. \label{eq:apdx-etg-cs-feasible-region}
\end{align}

\subsection{Proofs}\label{sec:apdx-cost-proofs}

When the data sources have an associated cost structure,
the horizon $T$ is a random variable that depends on the policy $\pi$: 
\begin{align}
    T = \max \{ t \in \NB :  t \cdot (\kappa^\top_t c) \leq B \}. \label{eq:apdx-random-T-cost}
\end{align}
The proofs of consistency and asymptotic inference
in the cost-structure setting
are similar to the uniform cost setting
and so we only highlight the differences.
Let $c_{\min} = \min_{i} c_i$ and $c_{\max} = \max_{i} c_i$
be the minimum and maximum costs of the data sources.
Observe that
\begin{align*}
    \underbrace{\frac{B}{c_{\max}}}_{T_{\min}} \leq T \leq \underbrace{\frac{B}{c_{\min}}}_{T_{\max}}.
\end{align*}
The horizon $T \to \infty$ as $B \to \infty$
and for a given budget $B$, the horizon $T$ is a stopping time.
Therefore, if a sequence $X_t$ is a MDS, then the sequence
$X_t \mathbbm{1}(t \leq T)$ is also a MDS.

To extend the consistency results to this setting,
we show that the MDS SLLN (Cor.~\ref{cor:apdx-mds-slln}) also
holds for this random $T$.

\begin{lemma}
Let $\{ X_t, t \in \mathbb{N} \}$ be a martingale difference sequence such that $\sup_k \EB[X^2_k] < \infty$. 
Then $\sum_{t=1}^{T} X_t / T \ConvAS 0$ as $B \to \infty$, 
where $T$ is defined in Eq.~\ref{eq:apdx-random-T-cost}.
\end{lemma}
\begin{proof}
We have
\begin{align*}
    \left| \frac{1}{T} \sum_{t=1}^{T} X_t \right| &= \frac{T_{\max}}{T} \left| \frac{1}{T_{\max}} \sum_{t=1}^{T_{\max}} X_t \mathbbm{1}(t \leq T) \right| \\
    &\leq \frac{c_{\max}}{c_{\min}} \left| \frac{1}{T_{\max}} \sum_{t=1}^{T_{\max}} X_t \mathbbm{1}(t \leq T) \right| \\
    &= \litOas(1),
\end{align*}
where the last line follows by Corollary~\ref{cor:apdx-mds-slln} (MDS SLLN).
\end{proof}

\newtheorem*{prop:asymp-norm-cost}{Proposition~\ref{prop:asymp-norm-cost}}
\begin{prop:asymp-norm-cost}
Suppose that the conditions of Prop.~\ref{prop:asymp-normality} hold. Then,
\begin{align*}
    \sqrt{B} \left(\widehat{\beta}_T - \beta^* \right) \ConvDist \NM_{\kappa_{\infty}}\left(0, V_{*}(\kappa_{\infty}) \cdot \left( \kappa_{\infty}^\top c \right) \right).
\end{align*}
\end{prop:asymp-norm-cost}
\begin{proof}
We can demonstrate asymptotic normality 
under the conditions of Prop.~\ref{prop:asymp-normality} for
\begin{align*}
    \frac{1}{\sqrt{T}} \sum_{t=1}^{T} X_t = \underbrace{\sqrt{\frac{T_{\max}}{T}}}_{\text{(A)}} \underbrace{\frac{1}{\sqrt{T_{\max}}} \sum_{t=1}^{T_{\max}} X_t \mathbbm{1}(t \leq T)}_{\text{(B)}}.
\end{align*}
Since the horizon $T \ConvProb \floor*{B / \kappa^\top_{\infty} c}$,
for term (A), we have $(T_{\max} / T) \ConvProb \kappa^\top_{\infty} c / c_{\min}$.
We can apply the regular martingale CLT
for term (B)
and combine the two terms using the Cramer-Slutsky theorem (Corollary~\ref{cor:apdx-cramer-slutsky}).
\end{proof}

\newtheorem*{prop:regret-etc-cost-structure}{Proposition~\ref{prop:regret-etc-cost-structure}}
\begin{prop:regret-etc-cost-structure}
Suppose that the conditions of Prop.~\ref{prop:asymp-normality} and
Lemma~\ref{lemma:oracle-simplex-estimation} hold.
If $e = o(1)$ and $Be \rightarrow \infty$ as $B \to \infty$, then
$R_{\infty}(\pi_{\text{ETC-CS}}) = 0$.
\end{prop:regret-etc-cost-structure}
\begin{proof}
We prove this similarly to Theorem~\ref{thm:etc-regret}.
Since we use $Be$ budget for exploration, the number of samples collected
after exploration is
\begin{align*}
    T_e = \floor*{\frac{Be}{\ctrSim^{\top} c}}.
\end{align*}
The oracle simplex is estimated as
\begin{align*}
    \widehat{k}_{T_e} = \arg\min_{\kappa \in \ChoiceSimplex} \widehat{V}_{T_e}(\widehat{\theta}_{T_e}, \kappa) \cdot \left( \kappa^\top c \right).
\end{align*}
Since $T_e \to \infty$ as $B \to \infty$,
by Lemma~\ref{lemma:oracle-simplex-estimation},
$\widehat{k}_{T_e} \ConvProb \kappa^*$.
By Lemma~\ref{lemma:zero-regret-policy}, the regret is zero.
\end{proof}

\newtheorem*{prop:regret-etg-cs}{Proposition~\ref{prop:regret-etg-cs}}
\begin{prop:regret-etg-cs}
Suppose that 
the conditions of Prop.~\ref{prop:asymp-normality} and
Lemma~\ref{lemma:oracle-simplex-estimation-almost-sure} hold.
If $e = o(1)$ and $Be \to \infty$ as $B \to \infty$,
then $R_\infty(\pi_{\text{ETG-CS}}) = 0$.
\end{prop:regret-etg-cs}
\begin{proof}
We prove this in the same way as Theorem~\ref{thm:etg-infinite-rounds}
with minor changes to account for the cost structure.
Let $B_j := (Be + jS)$ denote the budget used after round $j \in [0, \hdots, J]$
for $J = B(1-e) / S$.
Since we use a fixed budget $S$ in each round, the number of samples
collected is random and depends on the selection simplex $\kappa$.
The number of samples collected after round $j$ is
\begin{align*}
    T_j = \floor*{ \frac{B_j}{\kappa^\top_{T_j} c} }.
\end{align*}
Since the conditions of Lemma~\ref{lemma:oracle-simplex-estimation-almost-sure} hold,
we have $\widehat{k}_t \ConvAS \kappa^*$. 
Thus, for any $\epsilon > 0$,
\begin{align*}
    \lim_{\Tilde{t} \to \infty} \PB\left( \forall t > \Tilde{t} : \,\, \widehat{k}_t \in \BM_{\epsilon}(\kappa^*) \right) = 1.
\end{align*}
By the same argument as Theorem~\ref{thm:etg-infinite-rounds},
since $\kappa_t$ moves as close to $\widehat{k}_t$
as possible
after each round,
we also have $\kappa_T \ConvAS \kappa^*$.
Applying Lemma~\ref{lemma:zero-regret-policy} completes the proof.
\end{proof}

\section{Proofs for asymptotic confidence sequences}
\label{sec:apdx-time-uniform}
\begin{lemma}\label{lemma:apdx-log-t-alm-sure-conv}
Let $X_t$ be a sequence of 
non-negative random variables such that
(i) $X_t = \litOas\left( \sqrt{\log t / t} \right)$, and 
(ii) $\sup_t \EB[X^{1+\delta}_t] < \infty$ for some $\delta > 0$.
Then, 
\begin{align*}
    \frac{1}{T} \sum_{t=1}^{T} X_t = \litOas\left( \sqrt{\log T / T} \right).
\end{align*}
\end{lemma}
\begin{proof}
Let $S_T := \sum_{t=1}^{T} X_t$. For any $T > t$,
\begin{align*}
    \frac{S_T}{\sqrt{T \log T}} &\leq \frac{S_{t}}{\sqrt{t \log t}} + \frac{1}{\sqrt{T \log T}} \sum_{i=t+1}^{T} X_i
\end{align*}
By the union bound, for any $\epsilon > 0$,
\begin{align}
    \PB\left( \forall T > t : \frac{S_{T}}{\sqrt{T \log T}} \leq 3 \epsilon \right) &\geq 1 - \underbrace{\PB\left( \frac{S_{t}}{\sqrt{t \log t}} > \epsilon \right)}_{:= a} - \underbrace{\PB\left( \exists T > t : \frac{\sum_{i=t+1}^{T} X_i}{\sqrt{T \log T}} > 2 \epsilon \right)}_{:= b}. \label{eq:apdx-log-t-alm-sure-conv-proof-a}
\end{align}
We show that Term \emph{a} in Eq.~\ref{eq:apdx-log-t-alm-sure-conv-proof-a} is $o(1)$. By Markov's inequality, we have
\begin{align}
    \PB\left( \frac{S_{t}}{\sqrt{t \log t}} > \epsilon \right) &\leq \frac{ \sum_{i=1}^{t} \EB[X_i]}{\sqrt{t \log t} \epsilon} \nonumber \\
        &= \frac{o(\sqrt{t \log t})}{\sqrt{t \log t} \epsilon} \label{eq:apdx-log-t-alm-sure-conv-proof-b} \\
        &= o(1), \nonumber
\end{align}
where (\ref{eq:apdx-log-t-alm-sure-conv-proof-b}) follows because $\EB[X_i] = o(\log i / i)$ (
by Conditions~(i, ii) and Prop.~\ref{prop:apdx-conv-of-moments}).
Next, we show that Term \emph{b} in Eq.~\ref{eq:apdx-log-t-alm-sure-conv-proof-a} is $o(1)$. 
Since $X_t = \litOas\left( \sqrt{\log t / t} \right)$,
by Lemma~\ref{lemma:conv-as-lim-sup-conv-prob},
for any $\epsilon > 0$,
\begin{align}
    & \lim_{t \to \infty} \PB\left( \forall T > t : X_T < \epsilon \sqrt{\log T / T} \right) = 1, \nonumber \\
    & \lim_{t \to \infty} \PB\left( \forall T > t : \sum_{i=t+1}^{T} X_i < \epsilon \sum_{i=t+1}^{T} \sqrt{\log i / i} \right) = 1, \nonumber \\
    & \lim_{t \to \infty} \PB\left( \forall T > t : \sum_{i=t+1}^{T} X_i < 2 \epsilon \sqrt{T \log T} \right) = 1 \label{eq:apdx-log-t-alm-sure-conv-proof-c}, \\
    & \lim_{t \to \infty} \PB\left( \forall T > t : \frac{\sum_{i=t+1}^{T} X_i}{\sqrt{T \log T}} < 2 \epsilon \right) = 1, \nonumber
\end{align}
where (\ref{eq:apdx-log-t-alm-sure-conv-proof-c}) follows
because $\sum_{i=t+1}^{T} \sqrt{\log i / i} \leq 2 \sqrt{T \log T}$.
Plugging these results into Eq.~\ref{eq:apdx-log-t-alm-sure-conv-proof-a}
and applying Lemma~\ref{lemma:conv-as-lim-sup-conv-prob} completes the proof.
\end{proof}

\begin{proposition}[Strong approximation {\citep[Lemma~A.2]{waudby2021time}}]\label{prop:apdx-strong-approx}
Let $\{ X_i, \FM_i, 1 \leq i \leq n \}$ be a martingale difference sequence
with $\sigma^2_i = \EB_{i-1}[X^2_i]$ and $\Gamma_t = \sum_{i=1}^t \sigma^2_i$.
Suppose that (i) $\Gamma_t \to \infty$ a.s. and 
(ii) (Lindeberg condition) for some $\delta > 0$,
$\sum_{t=1}^{\infty} \EB_{t-1}[|X_t|^2 \mathbf{1}(|X_t|^2 > \Gamma^{\delta}_t) ] / \Gamma^{\delta}_t < \infty$.
Then, on a potentially enriched probability space, there exist i.i.d.
standard Gaussians $(G_i)_{i=1}^{\infty}$ such that
\begin{align*}
    \frac{1}{n} \sum_{i=1}^n X_i - \frac{1}{n} \sum_{i=1}^n \sigma_i G_i = \litOas\left( \frac{\Gamma^{\nicefrac{3}{8}}_t \log(\Gamma_t)}{t} \right).
\end{align*}
\end{proposition}

\begin{lemma}[Time-uniform Gaussian tail bound]\label{lemma:time-uniform-gaussian-bound}
Let $(\sigma_i)_{i=1}^{\infty}$ be a sequence of positive random variables
that are predictable w.r.t. the filtration $(\FM_i)_{i=1}^{\infty}$, that is,
$\sigma_i$ is $\FM_{i-1}$-measurable. Define $\Gamma_t = \sum_{i=1}^t \sigma_i$.
Let $(G_i)_{i=1}^{\infty}$ be a sequence of i.i.d. standard Gaussian
random variables. Then, for any constant $\rho > 0$ and $\alpha \in (0, 1)$, 
\begin{align*}
    \PB\left(\forall t \geq 1, \left| \frac{1}{t} \sum_{i=1}^t \sigma_i G_i \right| \leq
        \sqrt{ \frac{\Gamma_t \rho^2 + 1}{t^2 \rho^2} \log\left( \frac{\Gamma_t \rho^2 + 1}{\alpha^2} \right) } \right) \geq 1 - \alpha.
\end{align*}
\end{lemma}
\begin{proof}
See Step~2 of the proof of Proposition~2.5 in \citet{waudby2021time}.
\end{proof}

\newtheorem*{thm:asympotic-conf-sequence}{Theorem~\ref{thm:asympotic-conf-sequence}}
\begin{thm:asympotic-conf-sequence}
Suppose that (i) the Conditions of Lemma~\ref{lemma:oracle-simplex-estimation-almost-sure} hold;
(ii) Assumption~\ref{assum:nuisance-rates-asymp-cs} holds;
(iii) $\exists \delta > 0, \forall i \in [M]$ such that $\EB[| \psi^{(i)}_t(\theta^*, \eta^*) |^{2 + \delta}] < \infty$;
and
(iv) $\kappa_T \ConvAS \Tilde{\kappa}$ for some constant $\Tilde{\kappa} \in \ChoiceSimplex$.
Then, for any $\rho \in \RB_{> 0}$ and $\alpha \in (0, 1)$, 
the following time-uniform bound holds:
\begin{align*}
    \PB\left(\forall t \in \NB : \left| \widehat{\beta}_t - \beta^* \right| \leq
        \underbrace{\sqrt{ \frac{t \widehat{V}_t \rho^2 + 1}{t^2 \rho^2} \log\left( \frac{t \widehat{V}_t \rho^2 + 1}{\alpha^2} \right) }}_{:= \Bar{\CM}_t} + \, \litOas\left( \sqrt{ \frac{\log{t}}{t} } \right) \right) \geq 1 - \alpha,
\end{align*}
where $\widehat{V}_t$ is the estimated variance defined in Eq.~\ref{eq:variance-estimator}.
\end{thm:asympotic-conf-sequence}
\begin{proof}
For $i \in [M]$ and 
$D_t = \psi^{(i)}_t(\theta^*, \hEta) - \psi^{(i)}_t(\theta^*, \eta^*)$,
\begin{align*}
    \frac{1}{T} & \sum_{t=1}^{T} m_i(s_t) \psi^{(i)}_t(\theta^*, \hEta) \\
    &= \frac{1}{T} \sum_{t=1}^{T} m_i(s_t) \psi^{(i)}_t(\theta^*, \eta^*) +
            \underbrace{\frac{1}{T} \sum_{t=1}^{T} m_i(s_t) \left[ D_t - \EB_{t-1}[D_t] \right]}_{\text{EP (Empirical process)}} +
            \underbrace{\frac{1}{T} \sum_{t=1}^{T} m_i(s_t) \EB_{t-1}[D_t]}_{\text{Bias}}.
\end{align*}

\paragraph{Term EP.}
The sequence $D_t - \EB_{t-1}[D_t]$ is a MDS.
To show that \emph{Term EP}
is $\litOas(1/\sqrt{T})$, 
we apply Prop.~\ref{prop:apdx-mds-slln-general} (MDS SLLN) with $b_t = \sqrt{t}$:
\begin{align}
    \sum_{t=1}^{\infty} \frac{\EB\left[ m_i(s_t) \left( D_t - \EB_{t-1}[D_t] \right)^2 \right]}{b^2_t} &\leq
        \sum_{t=1}^{\infty} \frac{\EB\left[ \left( D_t - \EB_{t-1}[D_t] \right)^2 \right]}{t} \label{eq:apdx-asympcs-ep-a} \\
        &\leq \sum_{t=1}^{\infty} \frac{\EB\left[ D_t^2 \right]}{t} \nonumber \\
        &\leq \sum_{t=1}^{\infty} \frac{L \EB \left[ \| \hEta - \eta^* \|^2 \right]}{t} \label{eq:apdx-asympcs-ep-b} \\
        &\leq \sum_{t=1}^{\infty}  O\left(\frac{L}{t^{1 + 2 \gamma}}\right) \label{eq:apdx-asympcs-ep-c} \\
        &< \infty, \nonumber
\end{align}
where (\ref{eq:apdx-asympcs-ep-a}) follows because $m_i(s_t) \in \{0, 1\}$,
(\ref{eq:apdx-asympcs-ep-b}) by Property~\ref{property:ulln}(ii),
and (\ref{eq:apdx-asympcs-ep-c}) by Assumption~\ref{assum:nuisance-rates-asymp-cs}(d).
Therefore, $\sum_{t=1}^{T} m_i(s_t) \left[ D_t - \EB_{t-1}[D_t] \right] / \sqrt{T} \ConvAS 0$.

\paragraph{Term Bias.}
Next, we show that \emph{Term Bias} is $\litOas\left( \sqrt{\log T/T} \right)$.
As shown in the proof of Prop.~\ref{prop:asymp-normality},
by Neyman orthogonality (Assumption~\ref{assum:nuisance-rates-asymp-cs}(a)),
\begin{align*}
    \frac{1}{T} \sum_{t=1}^{T} m_i(s_t) \EB_{t-1}[D_t] &\leq \frac{1}{T} \sum_{t=1}^{T} R_t
        \overset{(a)}{=} \litOas\left( \sqrt{\log T/T} \right),
\end{align*}
where (a) follows by Lemma~\ref{lemma:apdx-log-t-alm-sure-conv}.
Using these results, it follows that
\begin{align}
    \frac{1}{T} \sum_{t=1}^{T} g_t(\theta^*, \hEta) &=
        \frac{1}{T} \sum_{t=1}^{T} g_t(\theta^*, \eta^*) + \litOas\left( \sqrt{\log T/T} \right). \label{eq:apdx-asymp-decomp-rate}
\end{align}

Recall from Eqs.~\ref{eq:apdx-first-order-gmm-a} and \ref{eq:apdx-beta-delta-method-a0} that
\begin{align*}
    \widehat{\theta}_T - \theta^* &= - \left[ \widehat{G}^\top(\widehat{\theta}_T) \widehat{\Omega}(\widehat{\theta}^{(\text{os})}_T)^{-1} \widehat{G}(\Tilde{\theta})  \right]^{-1} \widehat{G}^\top(\widehat{\theta}_T) \widehat{\Omega}(\widehat{\theta}^{(\text{os})}_T)^{-1} \frac{1}{T} \sum_{t=1}^{T} g_t(\theta^*, \hEta), \\
    \widehat{\beta}_T - \beta^* &= \nabla_{\theta} \fTar(\Tilde{\theta})^{\top} (\widehat{\theta}_T - \theta^*).
\end{align*}
For $\Tilde{\kappa}$ defined in Condition~(iii),
by Lemma~\ref{lemma:oracle-simplex-estimation-almost-sure},
we have
$\widehat{G}^\top(\widehat{\theta}_T) \ConvAS G_{*}(\Tilde{\kappa})$ and
$\widehat{\Omega}(\widehat{\theta}^{(\text{os})}_T) \ConvAS \Omega_{*}(\Tilde{\kappa})$.
For $M_{*}(\Tilde{\kappa}) := - \nabla_{\theta} \fTar(\theta^*)^{\top} \left[ G_{*}(\Tilde{\kappa})^\top \Omega_{*}(\Tilde{\kappa})^{-1} G_{*}(\Tilde{\kappa})  \right]^{-1} G_{*}(\Tilde{\kappa})^\top \Omega_{*}(\Tilde{\kappa})^{-1} \in \RB^{1 \times M}$,
\begin{align}
    \widehat{\beta}_T - \beta^* &= \left[ M_{*}(\Tilde{\kappa}) + \litOas(1) \right] \frac{1}{T} \sum_{t=1}^{T} g_t(\theta^*, \hEta) \nonumber \\
        &= \left[ M_{*}(\Tilde{\kappa}) + \litOas(1) \right] \left\{ 
        \frac{1}{T} \sum_{t=1}^{T} g_t(\theta^*, \eta^*) + \litOas\left( \sqrt{ \frac{\log{T}}{T} } \right) \right\} \label{apdx:eq-decom-remainder-a} \\
        &= \frac{1}{T} \sum_{t=1}^{T} M_{*}(\Tilde{\kappa}) g_t(\theta^*, \eta^*) + \litOas\left( \sqrt{ \frac{\log{T}}{T} } \right), \nonumber 
\end{align}
where (\ref{apdx:eq-decom-remainder-a}) follows by Eq.~\ref{eq:apdx-asymp-decomp-rate}.

The sequence $M_{*}(\Tilde{\kappa}) g_t(\theta^*, \eta^*)$ is a MDS.
Let $\sigma^2_t := \EB_{t-1}[(M_{*}(\Tilde{\kappa}) g_t(\theta^*, \eta^*))^2]$ and
$\Gamma_T = \sum_{t=1}^T \sigma^2_t$.
The Lindeberg condition in Prop.~\ref{prop:apdx-strong-approx}
is implied by the Lyapunov condition
\citep[Appendix~B.5]{waudby2021time}
shown to hold
for Prop.~\ref{prop:asymp-normality}.
Thus, by Prop.~\ref{prop:apdx-strong-approx},
\begin{align*}
    \frac{1}{T} \sum_{t=1}^T M_{*}(\Tilde{\kappa}) g_t(\theta^*, \eta^*) - \frac{1}{T} \sum_{t=1}^T \sigma_t G_t &= \litOas\left( \frac{\Gamma^{\nicefrac{3}{8}}_T \log(\Gamma_T)}{T} \right) = \litOas\left( \frac{\log(T)}{T^{\nicefrac{5}{8}}} \right) \\
    \therefore (\widehat{\beta}_T - \beta^*) - \frac{1}{T} \sum_{t=1}^T \sigma_t G_t &= \litOas\left( \sqrt{ \frac{\log{T}}{T} } \right).
\end{align*}
By Lemma~\ref{lemma:time-uniform-gaussian-bound},
\begin{align}
    & \PB\left(\forall t \geq 1, \left| \frac{1}{t} \sum_{i=1}^t \sigma_i G_i \right| \leq 
        \sqrt{ \frac{\Gamma_t \rho^2 + 1}{t^2 \rho^2} \log\left( \frac{\Gamma_t \rho^2 + 1}{\alpha^2} \right) } \right) \geq 1 - \alpha \nonumber \\
    \therefore \,\, & \PB\left(\forall t \geq 1, \left| \widehat{\beta}_t - \beta^* \right| \leq
        \sqrt{ \frac{\Gamma_t \rho^2 + 1}{t^2 \rho^2} \log\left( \frac{\Gamma_t \rho^2 + 1}{\alpha^2} \right) } + \litOas\left( \sqrt{ \frac{\log{t}}{t} } \right) \right) \geq 1 - \alpha. \label{apdx:eq-conf-seq-true-var}
\end{align}
We now derive an asymptotic confidence sequence in terms of the empirical variance:
\begin{align}
    \frac{1}{T} \Gamma_T &= \frac{1}{T} \sum_{t=1}^T \EB_{t-1}[(M_{*}(\Tilde{\kappa}) g_t(\theta^*, \eta^*))^2] \nonumber \\
        &= \frac{1}{T} \sum_{t=1}^T M_{*}(\Tilde{\kappa}) \EB_{t-1}[ g_t(\theta^*, \eta^*) g^\top_t(\theta^*, \eta^*)] M^\top_{*}(\kappa) \nonumber \\
        &= M_{*}(\Tilde{\kappa}) \left[ m_{\Omega}(\kappa_T) \odot \Omega(\theta^*, \eta^*) \right] M^\top_{*}(\kappa) \label{eq:apdx-emp-var-conv-a} \\
        &= M_{*}(\Tilde{\kappa}) \Omega_{*}(\kappa) M^\top_{*}(\kappa) + \litOas(1) \label{eq:apdx-emp-var-conv-b} \\
        &= \nabla_{\theta} \fTar(\theta^*)^{\top} \Sigma_{*}(\Tilde{\kappa}) \nabla_{\theta} \fTar(\theta^*) + \litOas(1) \label{eq:apdx-emp-var-conv-d} \\
        &= V_{*}(\Tilde{\kappa}) + \litOas(1), \label{eq:apdx-emp-var-conv-c}
\end{align}
where (\ref{eq:apdx-emp-var-conv-a}) follows by Eq.~\ref{eq:apdx-conv-variance-a0}, (\ref{eq:apdx-emp-var-conv-b}) because $\kappa_T \ConvAS \Tilde{\kappa}$,
and $\Sigma_{*}(\Tilde{\kappa})$ in (\ref{eq:apdx-emp-var-conv-d})
is defined in Eq.~\ref{eq:apdx-Sigma-star}.
Under the conditions of Lemma~\ref{lemma:oracle-simplex-estimation-almost-sure},
we have shown that $\widehat{V}_T \ConvAS V_{*}(\Tilde{\kappa})$,
where $\widehat{V}_T$ is the empirical variance estimator defined in Eq.~\ref{eq:variance-estimator}.
Combining this with Eq.~\ref{eq:apdx-emp-var-conv-c}, we have
\begin{align*}
    \left|\frac{1}{T} \Gamma_T - \widehat{V}_T \right| \ConvAS 0.
\end{align*}
Plugging this result into Eq.~\ref{apdx:eq-conf-seq-true-var}, we get
\begin{align*}
    \sqrt{ \frac{\Gamma_t \rho^2 + 1}{t^2 \rho^2} \log\left( \frac{\Gamma_t \rho^2 + 1}{\alpha^2} \right) } &= \sqrt{ \frac{(t \widehat{V}_t + \litOas(t)) \rho^2 + 1}{t^2 \rho^2} \log\left( \frac{(t \widehat{V}_t + \litOas(t)) \rho^2 + 1}{\alpha^2} \right) } \\
        &= \sqrt{ \left( \frac{t \widehat{V}_t \rho^2 + 1}{t^2 \rho^2} + \litOas\left(\frac{1}{t}\right) \right) \log\left[\left( \frac{t \widehat{V}_t \rho^2 + 1}{\alpha^2} \right) (1 + \litOas(1)) \right] } \\
        &= \sqrt{ \left( \frac{t \widehat{V}_t \rho^2 + 1}{t^2 \rho^2} + \litOas\left(\frac{1}{t}\right) \right) \left[\log\left( \frac{t \widehat{V}_t \rho^2 + 1}{\alpha^2} \right) + \log(1 + \litOas(1)) \right] } \\
        &\overset{(a)}{=} \sqrt{ \left( \frac{t \widehat{V}_t \rho^2 + 1}{t^2 \rho^2} + \litOas\left(\frac{1}{t}\right) \right) \left[\log\left( \frac{t \widehat{V}_t \rho^2 + 1}{\alpha^2} \right) + \litOas(1) \right] } \\
        &= \sqrt{ \frac{t \widehat{V}_t \rho^2 + 1}{t^2 \rho^2} \log\left( \frac{t \widehat{V}_t \rho^2 + 1}{\alpha^2} \right) + \litOas\left(\frac{1}{t}\right) + \litOas\left(\frac{\log{t}}{t}\right) + \litOas\left(\frac{1}{t}\right) } \\
        &= \sqrt{ \frac{t \widehat{V}_t \rho^2 + 1}{t^2 \rho^2} \log\left( \frac{t \widehat{V}_t \rho^2 + 1}{\alpha^2} \right) + \litOas\left(\frac{\log{t}}{t}\right) } \\
        &\overset{(b)}{\leq} \sqrt{ \frac{t \widehat{V}_t \rho^2 + 1}{t^2 \rho^2} \log\left( \frac{t \widehat{V}_t \rho^2 + 1}{\alpha^2} \right) } + \litOas\left(\sqrt{\frac{\log{t}}{t}}\right),
\end{align*}
where (a) follows because $\log(1+\litOas(1)) = \litOas(1)$, and (b) because
$\sqrt{a + b} \leq \sqrt{a} + \sqrt{b}$.
This allows us to construct a $(1-\alpha)$ AsympCS using the empirical variance:
\begin{align*}
    \PB\left(\forall t \geq 1, \left| \widehat{\beta}_t - \beta^* \right| \leq
        \sqrt{ \frac{t \widehat{V}_t \rho^2 + 1}{t^2 \rho^2} \log\left( \frac{t \widehat{V}_t \rho^2 + 1}{\alpha^2} \right) } + \litOas\left( \sqrt{ \frac{\log{t}}{t} } \right) \right) \geq 1 - \alpha.
\end{align*}
\end{proof}

\section{Convergence rates for nuisance estimation}
\label{sec:apdx-nuisance-convergence-rates}
Some of our results require almost sure convergence
 of the nuisance estimators.
In this section, we illustrate how non-asymptotic tail
bounds from
learning theory can be used to verify convergence rates.
Consider the standard nonparametric regression setup:
\begin{align*}
    y_i = f^*(x_i) + \nu_i,
\end{align*}
where $y_i \in \RB$ is the response variable, 
$x_i \in \mathcal{X}$ are the covariates,
$f^*(x) = \EB[Y|X=x]$, and $\nu_i$ is an independent exogenous noise term.
Given $n$ i.i.d. samples of $\{x_i, y_i\}_{i=1}^{n}$, suppose that the estimator
$\widehat{f}$ is obtained by solving the following minimization problem:
\begin{align*}
    \widehat{f} \in \arg\min_{f \in \FM} \left\{ \frac{1}{n} \sum_{i=1}^n (y_i - f(x_i))^2 \right\},
\end{align*}
where $\FM$ is a suitably chosen function class such that $f^* \in \FM$.
Under suitable restrictions on $\FM$, the following tail bound can be obtained:
\begin{align}
    \PB\left( \| \widehat{f} - f^* \| > c_0 \delta_n \right) < c_1 \exp\left( - c_2 n \delta^2_n \right), \label{eq:apdx-nonparam-tail-bound}
\end{align}
where $c_0$, $c_1$, and $c_2$ are constants, and $\delta_n$ depends
on the richness of the function class $\FM$ (e.g., the metric entropy).
For example, when $\FM$ is the class of convex $1$-Lipschitz functions,
Eq.~\ref{eq:apdx-nonparam-tail-bound} holds with $\delta_n = n^{-2/5}$ 
\citep[Example~14.4]{wainwright2019high}.
We refer the reader to \citet[Chapters~13, 14]{wainwright2019high} for more details.
The tail bound in Eq.~\ref{eq:apdx-nonparam-tail-bound} 
along with the the Borel–Cantelli lemma can be used 
to verify almost sure convergence rates:

\begin{proposition}\label{cor:apdx-almost-sure-nonparam}
Suppose that (i) Eq.~\ref{eq:apdx-nonparam-tail-bound} holds; and
(ii) $n \delta^2_n = n^{\gamma}$ for some $\gamma > 0$.
Then, for any $\alpha \geq 0$ such that $n^{\alpha} \delta_n = o(1)$,
we have $\| \widehat{f} - f^* \| = \litOas(n^{-\alpha})$.
\end{proposition}
\begin{proof}
By Eq.~\ref{eq:apdx-nonparam-tail-bound}, we have
\begin{align*}
    \PB\left( \| \widehat{f} - f^* \| > c_0 \delta_n \right) &= \PB\left( n^{\alpha} \| \widehat{f} - f^* \|  > c_0 n^{\alpha} \delta_n \right) \\
        &< c_1 \exp\left( - c_2 n \delta^2_n \right).
\end{align*}
Since $n^{\alpha} \delta_n = o(1)$, 
for any $\epsilon > 0$, there is a large enough $n_0$ such that
$\forall n > n_0, \epsilon > c_0 n^{\alpha} \delta_{n}$. Then,
\begin{align*}
    \sum_{n=1}^{\infty} \PB\left( n^{\alpha} \| \widehat{f} - f^* \| > \epsilon \right) &\leq n_0 + \sum_{n=n_0}^{\infty} \PB\left( n^{\alpha} \| \widehat{f} - f^* \| > c_0 n^{\alpha} \delta_n \right) \\
        &< n_0 + \sum_{n=n_0}^{\infty} c_1 \exp\left( - c_2 n \delta^2_n \right) \\
        &= n_0 + \sum_{n=n_0}^{\infty} c_1 \exp\left( - c_2 n^{\gamma} \right) \\
        &< \infty.
\end{align*}
Therefore, by the Borel–Cantelli lemma, we have $\| \widehat{f} - f^* \| = \litOas(n^{-\alpha})$.
\end{proof}

\section{Experiments}
\label{sec:apdx-experiments}
\subsection{Nonlinear two-sample IV}\label{sec:apdx-expt-nonlinear-iv}

In this section, we present additional details
for the synthetic experiments 
with nonlinear causal models (Section~\ref{sec:expr-synthetic}). 
The structural equations for generating the data are:
\begin{align*}
    W &\sim \text{Uniform}[-\theta_w, \theta_w], \\
    U &\sim \text{Uniform}[-\theta_u, \theta_u], \\
    Z &\sim \text{Bernoulli}( \omega( f_Z(W) )), \\
    X &\sim \text{Bernoulli}( \omega( f_U(U) + f_{WZ}(W, Z) + b_Z )), \\
    Y &:= f_{Y.U}(U) + f_{WX}(W, X) + b_Y + \text{Uniform}[-\theta_y, \theta_y],
\end{align*}
where $\omega(.)$ is the sigmoid function;
$\theta_w, \theta_u, b_Z, b_Y \in \RB$ are the model parameters;
$f_{*}(.)$ are nonlinear functions generated using
$100$ random Fourier features \citep{rahimi2007random}
that approximate a Gaussian Process \footnote{We used the code from \href{https://random-walks.org/book/papers/rff/rff.html}{https://random-walks.org/book/papers/rff/rff.html}.} ;
and $U$ is an unmeasured confounder.

\subsection{Synthetic Linear Causal Models}\label{sec:apdx-expt-linear-models}

We also evaluate our methods on three tasks where 
the data is generated from linear causal models (Fig.~\ref{fig:synthetic-data-linear-results}).
We use linear models for nuisance estimation.
Since a correctly specified parametric model is used, we obtain
$\sqrt{n}$-convergence rates for nuisance estimation.

We first test our methods on the
two-sample LATE estimation task 
(see Example~\ref{example:two-sample-late} and Fig.~\ref{fig:disjoint-iv-graph})
where the data sources are $\DM = \left( \PB(W, Z, Y), \PB(W, Z, X) \right)$.
We set the model parameters such that the oracle simplex is
$\kappa^* \approx [0.65, 0.35]^\top$.
We compare the relative regrets of our proposed policies 
to a fixed policy that queries both data sources equally 
with $\kappa_T = [0.5, 0.5]^\top$
(Fig.~\ref{fig:two-sample-iv-regret}).
The fixed policy suffers constant relative regret of $\approx 15\%$.
By contrast, the ETC and ETG policies have close to zero relative regret as the horizon $T$
increases, demonstrating the efficiency gained using
adaptive data collection.

Next, we test our methods on the ATE estimation task for
the overidentified confounder-mediator causal graph
where both the backdoor and frontdoor identification strategies hold
(see Example~\ref{example:confounder-mediator} and Fig.~\ref{fig:confounder-mediator-graph}).
The data sources are $\DM = \left( \PB(W, X, Y), \PB(X, M, Y) \right)$.
We set the model parameters such that the oracle simplex is
$\kappa^* = [0, 1]^\top$, i.e., the frontdoor estimator is more efficient.
We compare our proposed policies to a fixed policy that queries both
data sources uniformly (see Fig.~\ref{fig:frontdoor-backdoor-regret}).
We observe that all adaptive policies outperform the fixed
policy significantly 
(we omit \emph{etc} because it performs similarly
as \emph{etg}).
However, the relative regret increases as the exploration increases
(\emph{etg}\_$0.1$ performs the best)
since it reduces the available budget for getting close
to the oracle simplex.

\begin{figure}
\centering
\begin{subfigure}[b]{1\textwidth}
\centering
\includegraphics[scale=0.35]{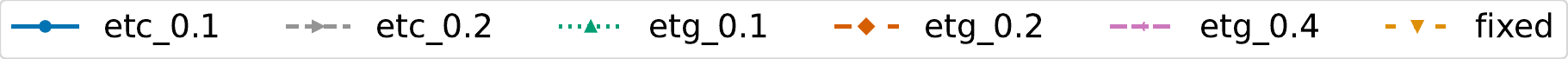}
\vspace{10px}
\end{subfigure}
\begin{subfigure}[b]{0.31\textwidth}
\centering
\includegraphics[scale=0.31]{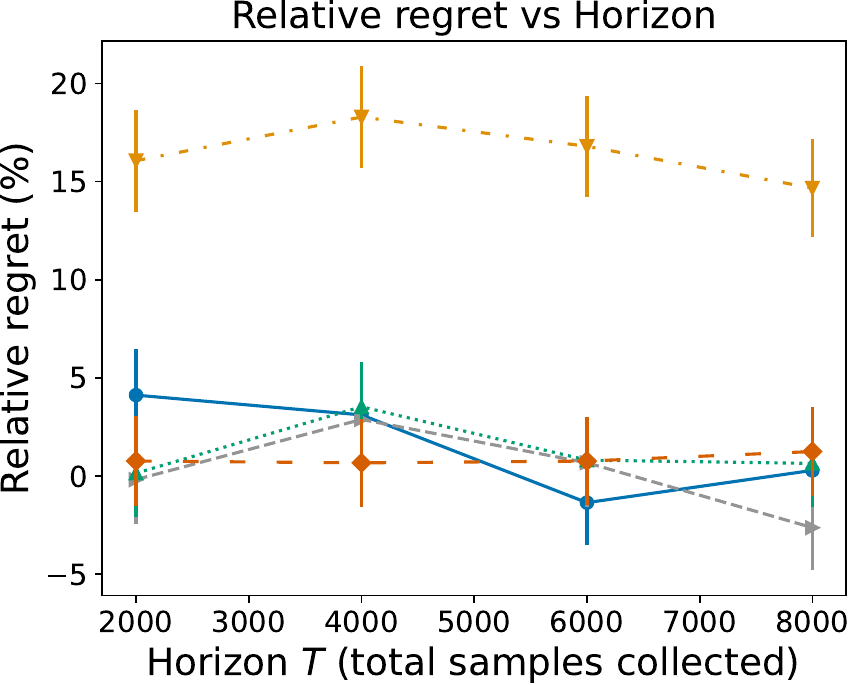}
\caption{Example~\ref{example:two-sample-late}}
\label{fig:two-sample-iv-regret}
\end{subfigure}
\begin{subfigure}[b]{0.31\textwidth}
\centering
\includegraphics[scale=0.31]{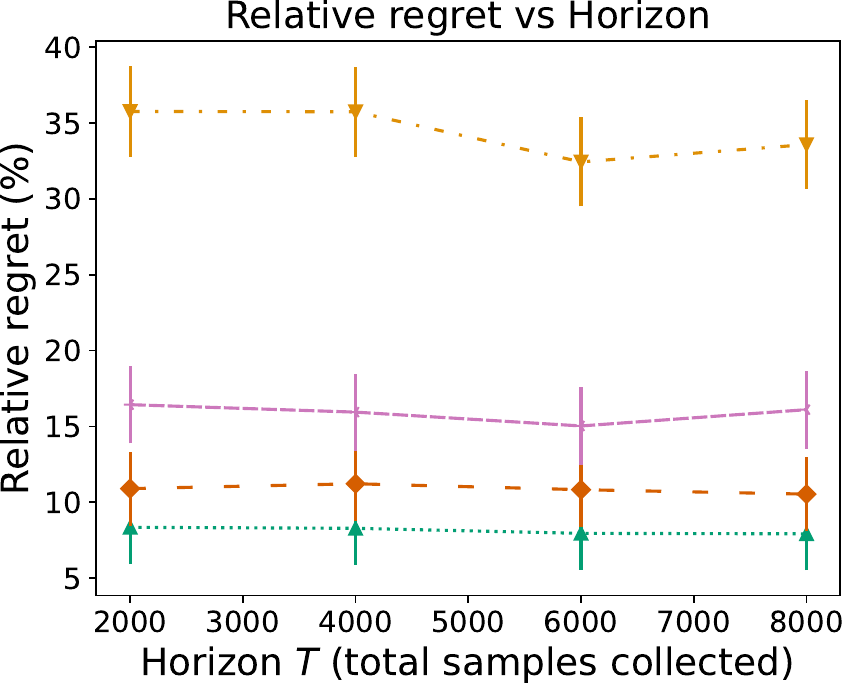}
\caption{Example~\ref{example:confounder-mediator}}
\label{fig:frontdoor-backdoor-regret}
\end{subfigure}
\begin{subfigure}[b]{0.31\textwidth}
\centering
\includegraphics[scale=0.31]{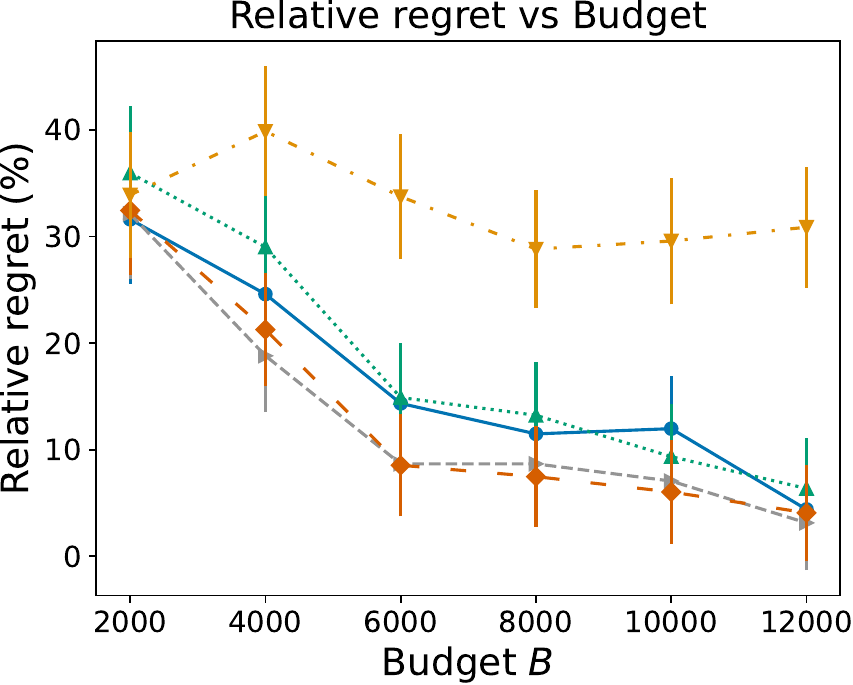}
\caption{Example~\ref{example:combine-observational-datasets}}
\label{fig:combine-obs-datasets-regret}
\end{subfigure}
\caption{The relative regret of different policies across various horizons/budgets for three ATE estimation tasks (error bars denote $95\%$ CIs).
We simulate data from synthetic linear causal models and use
linear nuisance estimators.
In all cases, the online data collection policies outperform the fixed policy as the horizon increases.
}
\label{fig:synthetic-data-linear-results}
\end{figure}

Next, we test our methods on the ATE estimation task 
in Example~\ref{example:combine-observational-datasets}
(also see Fig.~\ref{fig:two-covariates-graph})
with a non-uniform cost structure.
The data sources are $\DM = \left( \PB(U, W, X, Y), \PB(W, X, Y) \right)$.
We set the cost structure $c = [2, 1]$, i.e., querying the first data source is
twice as costly.
We set the model parameters such that the oracle simplex is
$\kappa^* \approx [0.4, 0.6]^\top$.
We compare our proposed policies to a fixed policy that queries only the 
first data source with $\kappa_T = [1, 0]^\top$ (see Fig.~\ref{fig:combine-obs-datasets-regret}).
The fixed policy suffers a constant relative regret of $\approx 35\%$.
We observe that the ETC and ETG policies suffer high relative regret at lower budgets.
As the budget increases, the ETC and ETG policies converge to zero relative regret.

\section{Examples of Online Moment Selection}
\label{sec:apdx-data-fusion-examples}
\newtheorem*{example:confounder-mediator}{Example~\ref{example:confounder-mediator} (more details)}
\begin{example:confounder-mediator}
Consider the confounder-mediator causal graph (Fig.~\ref{fig:confounder-mediator-graph})
with a binary treatment $X$, mediator $M$, outcome $Y$, confounder $W$.
The target parameter is the causal effect of $X$ on $Y$,
i.e., $\beta^* = \EB[Y|\text{do}(X=1)] - \EB[Y|\text{do}(X=0)]$ \citep{pearl2009causality}.
With $\DM = \{ \PB(W, X, Y), \PB(X, M, Y) \}$,
$\beta^*$ can be estimated with
the backdoor 
or the frontdoor criterion \citep[Sec.~3.3]{pearl2009causality}. 
The moment conditions can be written as
\begin{align*}
    g_t(\theta) = \begin{bmatrix}
        s_{t, 1} \\
        (1-s_{t, 1})
    \end{bmatrix} \odot \begin{bmatrix}
    \psi_{\text{AIPW}}((W_t, X_t, Y_t); \eta^{\text{(AIPW)}}) - \beta \\
    \psi_{\text{fd}}((X_t, M_t, Y_t); \eta^{\text{(fd)}}) - \beta
    \end{bmatrix},
\end{align*}
where $\psi_{\text{AIPW}}$ is defined in Eq.~\ref{eq:aipw-influence-func} and $\psi_{\text{fd}}$ is the efficient
influence function for the frontdoor estimand \citep{fulcher2020robust}:
\begin{align}
    \psi_{\text{fd}}((X_t, M_t, Y_t); & \eta^{\text{(fd)}}_{*}) = \frac{p_{1}(M_t) - p_{0}(M_t)}{p_{X_t}(M_t)}(Y_t - \EB[Y|X_t, M_t]) + \nonumber \\
        &\begin{aligned}[t]
            & \frac{X_t - (1-X_t)}{p(X_t)} \left( \sum_{x \in \{0, 1\}} \EB[Y|x, M_t] p(x) - \sum_{x, m} \EB[Y|x, m] p_{X_t}(m) p(x) \right)\\
            & + \sum_{m} \EB[Y|X_t, m] (p_{1}(m) - p_{0}(m)).
        \end{aligned} \label{eq:apdx-frontdoor-eif},
\end{align}
where $p_x(m) = \PB(M = m | X=x)$.
\end{example:confounder-mediator}

\begin{example}[Adaptive Neyman allocation]
Consider the setting of a randomized controlled trial
with a binary treatment.
Let $Y(1)$ and $Y(0)$ denote the potential outcomes \citep{imbens2015causal}
for an individual in the treatment and control group, respectively.
For each incoming subject, the experimenter must decide
whether to assign them to the treatment or the control
group \citep{zhao2023adaptive, neyman1934}.
For this setting, take 
$\DM = \left\{ \PB(Y(1)), \, \PB(Y(0)) \right\}$, and
\begin{align*}
    g_t(\theta) = \begin{bmatrix}
        s_{t, 1} \\
        1-s_{t, 1}
    \end{bmatrix} \odot \begin{bmatrix}
    Y_t(1) - \beta - \alpha \\
    Y_t(0) - \alpha
    \end{bmatrix},
\end{align*}
where $\theta = [\beta, \alpha]^\top$ and $\fTar(\theta) = \beta$
is the target parameter (the ATE).
\end{example}

\begin{example}[Long-term treatment effects]
In this setting, the aim is to combine experimental data on short-term outcomes and
observational data on long-term outcomes to estimate the long-term treatment effect \citep{athey2020combining}.
The efficient influence function in \citet[Theorem~3.1]{chen2023semiparametric} can 
be incorporated into our framework. In the notation of \citet{chen2023semiparametric}, we have
\begin{align*}
    g_t(\theta) = \begin{bmatrix}
        s_{t, 1} \\
        (1-s_{t, 1})
    \end{bmatrix} \odot \begin{bmatrix}
    \frac{w - (1 - w)}{\rho_w(s, x)}(y - \mu_w(s, x)) + \Bar{\mu}_1(x) - \Bar{\mu}_0(x) - \beta + \alpha \\
    \frac{w - (1 - w)}{w \varrho(x) + (1-w) (1-\varrho(x))} (\mu_w(s, x) - \Bar{\mu}_w(x)) - \alpha
    \end{bmatrix},
\end{align*}
where the target parameter $\beta$ is the long-term ATE. 
\end{example}

\end{document}